\documentclass{myarxiv}

\definecolor{assistantorange}{HTML}{C66700}

\title{Honeycomb-Surfels: Compact Hexagonal Primitives for Efficient Scene Rendering and Reconstruction}

\author[*, affil=1, email=pkaushi1@jh.edu]{Prakhar Kaushik}
\author[*, affil=1, email=spaul27@jh.edu]{Soumava Paul}
\author[affil=1, email=ayuille1@jh.edu]{Alan Yuille}
\affiliation{CCVL, Johns Hopkins University}
\institution{Johns Hopkins University}
\institutionlogo{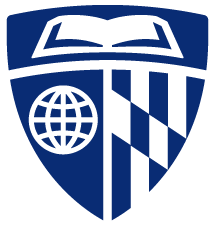}
\emails{\{pkaushi1, spaul27, ayuille1\}@jh.edu}
\projectlink{https://honeycomb-surfels.github.io/}
\codelink{https://github.com/honeycomb-surfels/hexels}
\titlelogo[height=86.43bp, xshift=-217.63bp, yshift=41.69bp]{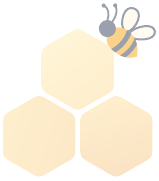}


\begin{document}

\begin{teaserfigure}[1]
\centering
\includegraphics[width=\textwidth]{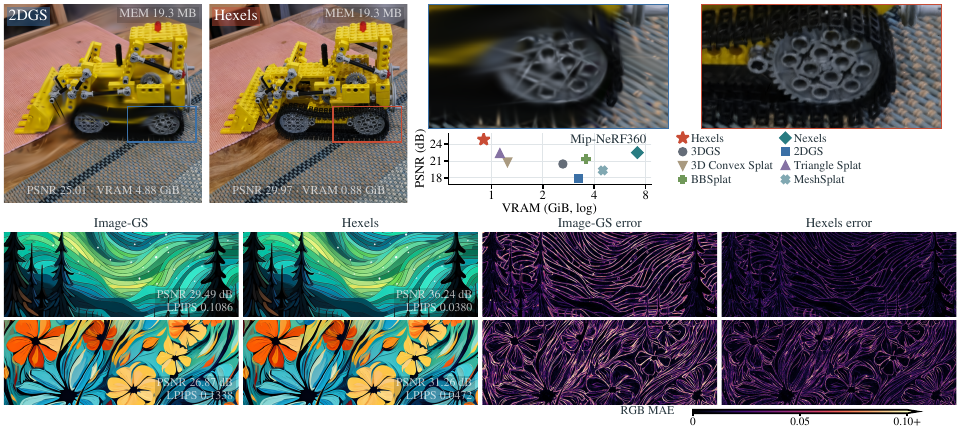}
\caption{\textbf{Hexels capture finer details using less memory.} Hexels are planar surface elements that combine hexagonal support with a first-order spatial color field and capitalize on structural redundancy. Top: At the same model memory (MEM), Hexels shows a gain of 4.96\,dB over 2DGS with $5.51\times$ less peak VRAM (0.88 versus 4.88\,GiB). Graph: Hexels outperform (mean PSNR) all other methods while requiring significantly less memory on 9 scenes in Mip-NeRF360 (Table~\ref{table:fixed_memory}). Bottom: Hexels can also efficiently represent images and outperform Image-GS with the same number of primitives (10k).}
\label{fig:teaser}
\end{teaserfigure}

\begin{abstract}
Novel view synthesis and 3D reconstruction have advanced rapidly with
radiance fields and splatting, yet their memory and computational demands
limit their use in resource-constrained settings. We introduce
\textbf{Hexels}: novel planar hexagonal surface elements with a first-order spatial color field for efficient 3D scene and 2D image representation. Hexels exploit locally planar regions with smoothly varying appearance using fewer, larger primitives. Their hexagonal support provides an area- and boundary-efficient domain, while the local color field captures appearance variation within each primitive. Across 34 scenes, Hexels improve PSNR over memory-matched 2DGS by 4.53--7.38\,dB across datasets while requiring lower training VRAM. We further ablate Hexels by showing that per primitive, Hexels outperform 2DGS by upto 4.74\,dB while using lower training VRAM and having $2.31\times$ faster FPS. Even in unconstrained settings, Hexels offers complementary quality-memory trade-offs. For efficient 3D surface reconstruction, Hexels reduce error by 28.9\% compared with memory-matched 2DGS across 15 DTU scans. On 6 Tanks and Temples surface scenes, Hexels achieves the highest mean F1 among compared methods (34.38\%, versus 31.72\% for 2DGS). We further also show that Hexels, as a general primitive, improves 2D image representation over Image-GS. Hexels push the Pareto frontier on the quality-efficiency trade-off of surface-aligned primitives.
\end{abstract}

\maketitle

\section{Introduction}
\label{sec:introduction}

Novel view synthesis and 3D reconstruction have progressed rapidly,
from neural radiance fields~\cite{mildenhall2020nerf} to explicit Gaussian
primitives that render in real time~\cite{kerbl3Dgaussians} to surface aligned planar primitives like 2DGS~\cite{huang2024_2dgs}, and Triangle Splatting~\cite{held2026triangle}.
These methods differ in geometry, but all face the same question: how much
of the scene can one primitive represent, and at what cost? This has led to the search for an overarching universal primitive which is efficient, expressive, and retains high geometric fidelity. This is even more important in resource-constrained settings, where efficiency includes not only rendering speed and model size but also the memory required during training. 
Previous methods have often traded one aspect for another. For example, conventional Gaussian and triangle splats typically represent one view-conditioned color per primitive, so spatial variation emerges through many overlapping elements, leading to large \textit{overdraw}. Others like BBSplat~\cite{svitov2025bbsplat} add
per-primitive textures, or use shared neural appearance fields like NeST~\cite{zhang2025nest} and
Nexels~\cite{rong2025nexels}.
These methods increase per-element capacity, but move complexity into
textures, grids, or decoders. In contrast, we target the entire quality--efficiency trade-off at the primitive itself. For resource-constrained settings, this efficiency includes not only rendering
speed and model size, but also the memory required during training. 

At constrained representation budgets, we seek a balance between geometric fidelity and appearance expressivity. Since agreement with training images alone does not guarantee accurate geometry~\cite{kaizhang2020}, we therefore target both surface reconstruction accuracy and novel-view rendering quality within a compact surface primitive.

We are motivated by the idea that natural images and scenes are highly structured: large regions are often locally planar with smoothly varying appearance. An efficient surface primitive should therefore be able to exploit this underlying structure, wherever possible, and do so in a compact and self-contained manner.

We introduce \emph{Hexels}, planar surface elements that combine a
hexagonal support with a first-order spatial color field. Natural scenes
contain large regions that are locally simple in geometry and appearance;
a Hexel can represent such a region with one large primitive, including
smooth color variation within it, while smaller Hexels capture boundaries,
curvature, and fine detail. This directly targets redundancy in surface-aligned splatting:
rather than approximating structured regions with many overlapping,
spatially constant primitives, Hexels increase what each primitive can
represent.

Why a hexagon? Regular hexagons tile a plane without gaps and cover more
area for the same boundary length than triangles or squares (other tileable primitives)~\cite{hales2001honeycomb} (discussed in Appendix~\ref{app:local-approximation}).
When combined with the local color field, each Hexel can model both the region itself and smooth appearance variation within it.
In 3D, Hexels are rendered as oriented surface elements using
perspective-correct ray-plane intersection and depth-ordered blending.
The same primitive can be placed directly in the image plane for 2D image
representation. Thus, Hexels provide a single compact primitive family for
novel-view synthesis, surface reconstruction, and image representation,
without requiring texture maps, neural appearance fields, or decoding networks.

We evaluate Hexels for 3D novel view synthesis (memory-matched, primitive-matched, and unconstrained), surface reconstruction, and 2D image representation. We primarily use 34 scenes across DTU~\cite{jensen2014large}, Tanks and Temples~\cite{Knapitsch2017}, and Mip-NeRF360~\cite{barron2022mip}. At matched model sizes, Hexels improves PSNR over 2DGS by 4.53--7.38\,dB with upto $3.6\times$ less VRAM requirements and over Triangle Splatting by 1.70--2.90\,dB, with up to $1.87\times$ higher rendering FPS (\autoref{table:fixed_memory}). At matched primitive budgets, Hexels improves PSNR over 2DGS by up to 4.74\,dB, using $1.76$--$1.95\times$ less model size and upto $7\times$ less VRAM, and renders up to $2.31\times$ faster (\autoref{table:fixed_primitives}). Under unrestricted model sizes, Hexels still improve PSNR over 2DGS by 0.90\,dB on both Tanks and Temples and DTU, while rendering $1.37\times$ and $2.26\times$ faster, respectively. For surface reconstruction, Hexels achieves 0.978\,mm mean Chamfer distance over 15 DTU scans, versus 1.375\,mm for 2DGS at the same model-memory budget (Table~\ref{tab:dtu_memory_main}). On 6 Tanks and Temples surface scenes, mean F1 for Hexels is 34.38\%, versus 31.72\% for 2DGS (Table~\ref{tab:tnt_surface_main}). Hexels also converges faster (\autoref{fig:2d_convergence}) and improves 2D image representation over Image-GS~\cite{imagegs2025} (Table~\ref{tab:hexels-2d}).

\paragraph{Contributions.}
\begin{itemize}
    \setlength{\itemsep}{0.35em}
    \setlength{\topsep}{0.35em}

    \item \textbf{Compact Hexel primitive.}
    We introduce \emph{Hexels}, compact, self-contained surface primitives
    for resource-constrained rendering and reconstruction, combining efficient
    hexagonal support with a first-order local color field without textures,
    feature grids, neural decoders, or baking.

    \item \textbf{Strong quality-efficiency trade-offs.}
    Under matched primitive budgets, Hexels outperform 2DGS by
    2.08--4.74\,dB while using $2\times$ less learned storage per primitive,
    upto $7\times$ lower peak training VRAM, and up to $2.31\times$ higher
    rendering throughput. Under matched model-memory budgets, they improve
    PSNR by 4.53--7.38\,dB while using upto $3.6\times$ lower peak training VRAM
    and rendering up to $1.87\times$ faster.

    \item \textbf{One primitive across tasks.}
    The same representation applies across novel-view synthesis, surface
    reconstruction, and 2D image representation, achieving strong performance
    and efficiency without changing the underlying primitive.
\end{itemize}

\section{Related Work}
\label{sec:related}

\begin{figure*}[htb]
 \centering
 \includegraphics[width=\textwidth]{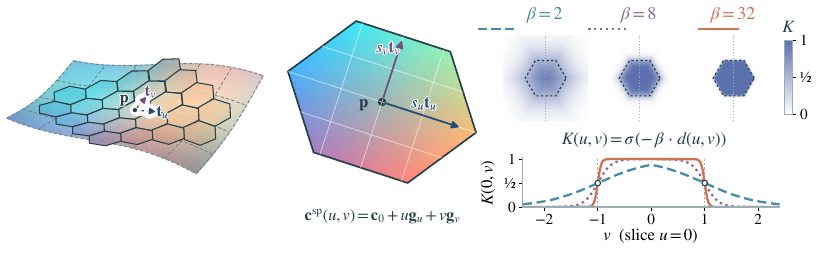}
\caption{A Hexel is a planar patch with a soft hexagonal boundary;
together, Hexels approximate a 3D surface (left). Its center $\bp$,
orthogonal unit tangents $\bt_u,\bt_v$, and scales $s_u,s_v$ determine
its position, orientation, and size (center). Color varies linearly from
$\bc_0$ at the center, with RGB gradients $\bg_u,\bg_v$ along local
coordinates $u,v$. $K$ (spatial weight) controls the patch's influence:
increasing sharpness $\beta$ sharpens the boundary (right). The tileable
hexagonal support and local linear color field jointly exploit structure
in scene geometry and appearance, allowing each Hexel to represent larger
coherent regions while retaining a controllably sharp boundary.}
 \label{fig:splat}
\end{figure*}

Novel view synthesis and 3D reconstruction have improved rapidly since
NeRF~\cite{mildenhall2020nerf}, which represents scenes as neural fields of
density and view-dependent radiance. Feature grids and explicit encodings
reduced their training and rendering cost~\cite{muller2022instant,yu2022plenoxels},
while 3DGS~\cite{kerbl3Dgaussians} replaced neural fields with explicit
anisotropic Gaussian primitives, enabling real-time rasterization.
2DGS~\cite{huang2024_2dgs} further replaced volumetric Gaussians with planar
surfels, substantially improving surface reconstruction. These advances made
rendering faster and geometry more accurate, but efficient representation
requires more: high quality together with low model storage, low training
VRAM, few primitives, and high FPS. Appendix~\ref{app:related} discusses these trade-offs.

Subsequent work improves different parts of this trade-off, but usually with
a catch. Triangle Splatting~\cite{held2026triangle}, Deformable Beta
Splatting~\cite{liu2025beta}, and 3D Convex Splatting~\cite{held2025convex}
redesign the primitive or kernel to improve expressiveness or rendering
speed, but at cost of higher memory and VRAM usage. Hybrid methods such as
BBSplat~\cite{svitov2025bbsplat}, NeST~\cite{zhang2025nest}, and
Nexels~\cite{rong2025nexels} reduce geometric redundancy by separating
geometry from appearance through textures or neural fields; this shifts
capacity into additional memory, training state, and computation. Other
methods retrofit existing representations through pruning, parameter
sharing, quantization, distillation, or baking
~\cite{kheradmand2024mcmc,lu2024scaffold,fan2024lightgaussian,chen2023mobilenerf} but still largely struggled to address the efficiency problem on all of its axes. 
Hexels provides a solution to the efficiency problem by redesigning the primitive itself to exploit local structure and geometry in appearance.
The resulting primitive is compact in the form in which it is learned and
rendered, simultaneously targeting storage, training VRAM, primitive count,
FPS, and reconstruction quality. Further, unlike previous works, Hexels are general surface
primitives: the same underlying design applies directly to oriented 3D
surfaces and to 2D images without major changes to the representation or
optimization.

\begin{figure*}[htb]
\centering
\includegraphics[width=\textwidth]
  {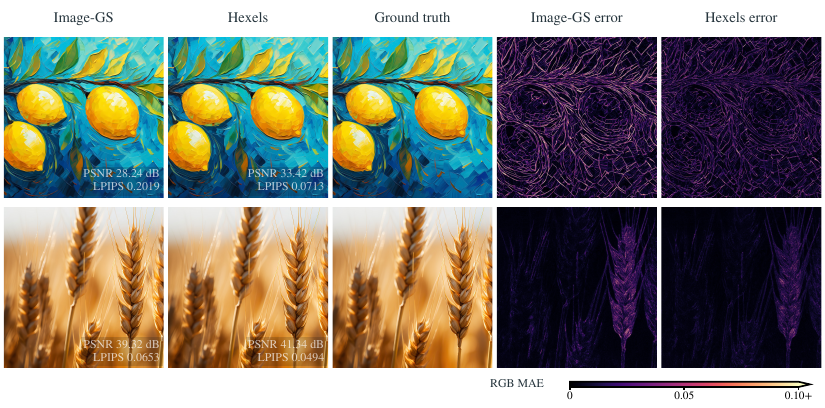}
\caption{\textbf{2D reconstruction on the Image-GS benchmark.}
Two $2048\!\times\!2048$ examples from the ImageGS dataset with 10K final primitives per method.
Hexels preserves finer boundaries and structure, with lower RGB error.}
\label{fig:2d_placeholder}
\label{fig:hexels-2d}
\label{fig:2d_reconstruction}
\end{figure*}
\section{Hexels: Hexagonal Surface Elements}
\label{sec:method}
We introduce \emph{Hexels}, planar hexagonal surface elements designed to reduce overlap and overdraw, avoid the long tails of Gaussian footprints, and represent more appearance within each primitive. Each Hexel combines a soft hexagonal support with a first-order spatial color field, allowing a single primitive to cover a larger coherent surface region while modeling smooth appearance variation within it. In 3D, a Hexel is an oriented planar patch with a local 2D coordinate
system. Following 2DGS~\cite{huang2024_2dgs}, we evaluate each primitive at the perspective-correct ray-plane intersection and blend the resulting contributions in depth order.

\paragraph{Modeling}
\label{sec:modeling}
A Hexel (Figure~\ref{fig:splat}) is centered at $\bp_i\in\mathbb{R}^3$ with two orthogonal unit
tangent directions $\bt_{u,i}$ and $\bt_{v,i}$ and positive tangent scales
$s_{u,i}$ and $s_{v,i}$. A local point $(u,v)$ maps to world space as
\begin{equation}
P_i(u,v)=
\bp_i+s_{u,i}\bt_{u,i}u+s_{v,i}\bt_{v,i}v.
\label{eq:plane-to-world}
\end{equation}
The tangent directions are obtained from a normalized quaternion, with
$\bn_i=\bt_{u,i}\times\bt_{v,i}$ defining the surface normal.
This maps the hexagon to an oriented, anisotropically scaled
planar patch in 3D.

\paragraph{Hexagonal support.}
In the local plane, we define
\begin{equation}
 d(u,v)=\max\left\{
 |v|,
 \left|\tfrac{\sqrt{3}}{2}u+\tfrac12v\right|,
 \left|\tfrac{\sqrt{3}}{2}u-\tfrac12v\right|
 \right\}-1 .
 \label{eq:hex-boundary}
\end{equation}
The contour $d(u,v)=0$ defines the canonical regular hexagon, with
$d<0$ inside and $d>0$ outside. We convert this into a soft support weight
\begin{equation}
 K_i(u,v)=\sigma\!\left(-\beta_i d(u,v)\right),\qquad
 \sigma(a)=\frac{1}{1+\exp(-a)},
 \label{eq:hex-footprint}
\end{equation}
where $\beta_i>0$ controls edge sharpness. Increasing $\beta_i$ produces a
sharper transition around the hexagon while preserving the $K_i=1/2$
contour. The scales in Eq.~\eqref{eq:plane-to-world} control the spatial
extent of the support, while the rasterizer truncates negligible sigmoid
tails below its contribution threshold (See \autoref{app:hex-math} for details). 

We use a hexagon because it provides an efficient, nearly isotropic way to
cover a plane. Regular hexagons tile a plane without gaps and, at equal
area, require less perimeter than regular triangles or squares, consistent
with the honeycomb theorem~\cite{hales2001honeycomb}. More concretely, under
the local first-order approximation model of
Appendix~\ref{app:local-approximation}, a suitably oriented and
anisotropically scaled hexagon can cover about $30\%$ more area than the
corresponding square and twice the area of the corresponding triangle for
the same bounded color error (See Proposition~\ref{prop:hex-area}). Our results reflect that these ideas do help structure a more efficient primitive even in non-ideal settings.

\paragraph{Spatial appearance.}
We model appearance within each Hexel using a first-order color field:
\begin{equation}
 \bc_i^{\mathrm{sp}}(u,v)
 =\bc_{0,i}+u\bg_{u,i}+v\bg_{v,i},
 \label{eq:spatial-color}
\end{equation}
where $\bc_{0,i}$ is the base RGB color and
$\bg_{u,i},\bg_{v,i}$ are RGB gradients along the two local coordinates.
This allows one Hexel to represent smooth appearance variation across its
surface rather than assigning a single spatial color to the primitive.
Since the field is defined in local coordinates, increasing a tangent scale
spreads the same color variation over a larger physical region. We refer to this view-independent model as \emph{V1}, with 20 learned
parameters per Hexel, compared with 58 for the evaluated 2DGS primitive.
Our \emph{V2} variant adds compact view dependence through a linear residual
$\br_i(\bd_i)=A_i\bd_i$, where $\bd_i$ is the normalized direction from the
camera center to the Hexel center, for 29 parameters in total. We also
consider an optional \emph{V3} variant that replaces the linear residual
with degree-three spherical harmonics.
See Appendix~\ref{app:appearance_variants} for the complete parameterization details.

\paragraph{Splatting.}
Following 2DGS, we recover local coordinates
$(u_i,v_i)$ at each pixel using perspective-correct ray--plane intersection
and composite primitives front-to-back in depth order. A Hexel contributes
$\alpha_i=\min\{o_iK_i(u_i,v_i),0.99\}$, with color evaluated from its local
appearance field at $(u_i,v_i)$. Thus, relative to 2DGS, the rasterization
and compositing remain unchanged; we replace the Gaussian weight with the
Hexel support and evaluate spatially varying color within each primitive.
Implementation details are given in Appendix~\ref{app:hex-renderer}.

\paragraph{Training.}
\label{sec:optimization}
We optimize all Hexel parameters directly from calibrated training images
with Adam~\cite{kingma2015adam}, using the standard photometric objective
$\mathcal{L}=(1-\lambda)\|\hat I-I\|_1+
\lambda(1-\operatorname{SSIM}(\hat I,I))$.
Hexels are initialized from either sparse COLMAP points\textcolor{assistantorange}{~\cite{schoenberger2016sfm}} or, where stated,
a denser MVS reconstruction\textcolor{assistantorange}{~\cite{schoenberger2016mvs}} obtained from the training views. Unless
otherwise specified, our reported models use a fixed primitive population.

\paragraph{Densification and Pruning.}
\label{sec:adaptive-population}
For unconstrained training, we additionally consider standard
densification and pruning, following 3DGS and
2DGS. Primitives with large position
gradients are cloned or split, with new Hexels initialized near the parent
in its tangent plane. We conservatively prune mature primitives that make
negligible contributions across the training views.
More details are given in Appendix~\ref{app:adaptive-population}.

At inference, we precompute camera-independent quantities and use an RGB-only rasterizer; depth ordering and compositing are unchanged from the training renderer. Full implementation details are provided in Appendix~\ref{app:hex-renderer}.

\subsection{Representing 2D Images as Hexels}
\label{sec:hexels-2d}
Unlike most 3D surface primitives, Hexels are general and can be used directly
for both 3D scenes and 2D images. For efficient image representation, a 2D
Hexel is defined by its image-plane center $\bm_i$, rotation $\theta_i$,
two scales $s_{x,i},s_{y,i}$, and support sharpness $\beta_i$. The same
hexagonal support from Eq.~\eqref{eq:hex-boundary} is evaluated directly at
each pixel, without projection, ray--plane intersection, opacity, or depth
ordering. Let $\bS_i=\operatorname{diag}(s_{x,i},s_{y,i})$ be the positive scale matrix and $\bR(\theta_i)$ the planar rotation. For a pixel $\bx$, we compute
\[
\bu_i=\bS_i^{-1}\bR(\theta_i)^\top(\bm_i-\bx),
\qquad
w_i(\bx)=\sigma\!\left(-\beta_i d(\bu_i)\right),
\]
Here $w_i(\bx)$ is a \emph{blending weight}: it is
larger inside the Hexel's footprint and decreases smoothly outside it.
Following Image-GS~\cite{imagegs2025}, each pixel retains up to $K=10$ Hexels
with the largest weights among those assigned to its tile. Let
$\mathcal{S}_K(\bx)$ be the indices of these selected Hexels. Their RGB
colors $\bc_i(\bx)$ form a weighted average:
\[
\bc(\bx)=
\frac{\sum_{i\in \mathcal{S}_K(\bx)} w_i(\bx)\bc_i(\bx)}
     {\sum_{i\in \mathcal{S}_K(\bx)} w_i(\bx)+\epsilon}.
\]
$\epsilon=10^{-4}$ prevents division by zero.

\begin{table}[ht]
\caption{\textbf{2D image representation on ImageGS dataset (45 images).}
Constant and linear RGB using 10K primitives. Metrics use unquantized FP32 parameters; Memory is losslessly compressed image files, including decoding metadata.}
\label{tab:hexels-2d}
\small
\fittabular{7}{%
\begin{tabular}{@{}lcccccc@{}}
\toprule
Method
& PSNR$\uparrow$
& SSIM$\uparrow$
& LPIPS$\downarrow$
& Iter.$\downarrow$
& \shortstack{Train\\(s)$\downarrow$}
& \shortstack{Memory\\(kB)$\downarrow$} \\
\midrule
Image-GS
& 33.16 & .9141 & .1372
& 10,000 & 238.40 & 292.175 \\

Hexels (constant)
& 34.16 & .9208 & .1099
& \textbf{5,100} & \textbf{112.11} & \textbf{289.799} \\

Hexels
& \textbf{35.82} & \textbf{.9359} & \textbf{.0924}
& \textbf{5,100} & 157.07 & 551.174 \\

\bottomrule
\end{tabular}}
\end{table}

We use the first-order spatial color field as the primary 2D appearance
model, expressed in image coordinates; constant-color Hexels, with a fixed shared $\beta$, match Image-GS's eight stored scalars per primitive. Both methods use image-gradient-based initialization from \cite{imagegs2025} and an L1 plus SSIM loss. Image-GS retains its original training, densification, and pruning methodology. Appendix~\ref{app:hex-2d} provides more training and inference details.
\begin{figure}[ht]
\centering
\includegraphics[width=\columnwidth]
  {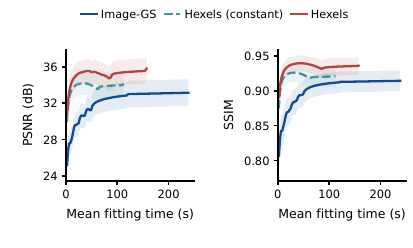}

\caption{\textbf{Hexels converge faster.} On the Image-GS benchmark, Hexels show faster convergence versus Image-GS while also outperforming it in terms of PSNR and SSIM.}
\label{fig:2d_convergence}
\end{figure} This 2D representation setting permits a straightforward way to evaluate \textit{overdraw} using reconstruction error and
primitive overlap. On the 45-image Image-GS benchmark, Hexels improve PSNR from 33.16 to 35.82\,dB and reduce LPIPS from .1372 to .0924 over ImageGS.
They use 5.1K rather than 10K iterations and 157.07 rather than 238.40 seconds
of image fitting/training time (Table~\ref{tab:hexels-2d}). Even using constant-color instead of linear color representation, Hexels still improve
all three quality metrics, fit in 112.11 seconds, and occupy 289.799\,kB memory,
versus 292.175\,kB for Image-GS. Figure~\ref{fig:2d_placeholder} shows qualitative comparisons.

\section{Experiments}
\label{sec:experiments}

We evaluate novel-view synthesis under matched model memory, matched
primitive budgets and native baseline settings. As our focus is on fundamental parameter-based efficiency, all Hexels results do not utilize any post-quantization or compression. For initial runtime compression analysis, refer to Appendix~\ref{app:nvs_compression}.

\paragraph{Implementation.}
Hexels start from sparse COLMAP points and train for 30K iterations.
Fixed configurations retain the initial population; densified configurations
use mixed clone/split growth and contribution-based pruning, with a 2M limit.

Appendix~\ref{app:hexels_variants} reports all configurations.
Training and evaluation details are in Appendix~\ref{sec:nvs_evaluation_protocol}.

\paragraph{Datasets, baselines, and metrics.}
We evaluate nine Mip-NeRF360 scenes~\cite{barron2022mip}, fifteen DTU
scans~\cite{jensen2014large}, and ten Tanks and Temples
scenes~\cite{Knapitsch2017}, using common held-out views and resolutions.
We report PSNR, SSIM\textcolor{assistantorange}{~\cite{wang2004ssim}}, LPIPS\textcolor{assistantorange}{~\cite{zhang2018perceptual}}, training time, peak training VRAM,
rendering FPS, and deployment memory size, averaged over scenes.
Our baselines are 3DGS~\cite{kerbl3Dgaussians}, 3D Convex
Splatting~\cite{held2025convex}, BBSplat~\cite{svitov2025bbsplat},
Nexels~\cite{rong2025nexels}, 2DGS~\cite{huang2024_2dgs},
Triangle Splatting~\cite{held2026triangle}, and
MeshSplatting~\cite{held2025mesh}.
Tables group methods into volumetric primitives, textured or neural surfels,
and native surface-aligned primitives with local appearance. We select these
baselines as competitive representatives of their respective classes, spanning
complementary quality-efficiency trade-offs. The final group - 2DGS,
Triangle Splatting, and MeshSplatting - provides the closest comparison to
Hexels, as these methods use explicit surface-aligned primitives without
separate texture maps or neural appearance fields; we therefore emphasize
this group in our analysis. To our knowledge, this evaluation over 34
real-world scenes is the largest quantitative NVS evaluation, by scene count,
among the original papers introducing the compared methods.

\begin{table*}[t]
\caption{\textbf{Matched model memory.} Hexels achieves the best PSNR and SSIM and the best or tied-best LPIPS across all three benchmarks, with low peak training VRAM and high rendering rates. Colors rank methods within each primitive group.}
\label{table:fixed_memory}
\label{tab:nvs_memory_v1_mip_outdoor}
\label{tab:nvs_memory_v1_dtu}
\small
\flushrules
\setlength{\tabcolsep}{4pt}
\renewcommand{\arraystretch}{1.14}
\fittabular{19}{%
\begin{tabular}{@{}lrrrrrr@{\hspace{9pt}}rrrrrr@{\hspace{9pt}}rrrrrr@{}}
\toprule
\multicolumn{1}{@{}l}{\multirow{3}{*}{Method}} & \multicolumn{6}{c}{Tanks and Temples (10)} & \multicolumn{6}{c}{Mip-NeRF 360 (9)} & \multicolumn{6}{c@{}}{DTU (15)} \\
\cmidrule(lr){2-7}\cmidrule(lr){8-13}\cmidrule(l){14-19}
 & Train & VRAM & FPS & PSNR & SSIM & LPIPS & Train & VRAM & FPS & PSNR & SSIM & LPIPS & Train & VRAM & FPS & PSNR & SSIM & LPIPS \\
 & min $\downarrow$ & GiB $\downarrow$ & $\uparrow$ & dB $\uparrow$ & $\uparrow$ & VGG $\downarrow$ & min $\downarrow$ & GiB $\downarrow$ & $\uparrow$ & dB $\uparrow$ & $\uparrow$ & VGG $\downarrow$ & min $\downarrow$ & GiB $\downarrow$ & $\uparrow$ & dB $\uparrow$ & $\uparrow$ & VGG $\downarrow$ \\
\midrule
\multicolumn{19}{@{}c@{}}{\textcolor{black!65}{\textit{Volumetric primitives}}} \\
3DGS & \cellcolor{tabtealfirst!55!white}5.6 & 1.60 & \cellcolor{tabtealfirst!55!white}807.5 & \cellcolor{tabtealfirst!55!white}20.39 & \cellcolor{tabtealfirst!55!white}0.683 & \cellcolor{tabtealfirst!55!white}0.477 & \cellcolor{tabtealfirst!55!white}8.2 & 2.62 & \cellcolor{tabtealfirst!55!white}852.6 & 20.46 & 0.528 & 0.614 & \cellcolor{tabtealfirst!55!white}4.5 & \cellcolor{tabtealfirst!55!white}0.47 & \cellcolor{tabtealfirst!55!white}1035.0 & \cellcolor{tabtealfirst!55!white}27.06 & 0.840 & 0.395 \\
3D Convex & 20.3 & \cellcolor{tabtealfirst!55!white}0.96 & 203.5 & 19.71 & 0.680 & 0.481 & 35.5 & \cellcolor{tabtealfirst!55!white}1.24 & 146.1 & \cellcolor{tabtealfirst!55!white}20.82 & \cellcolor{tabtealfirst!55!white}0.570 & \cellcolor{tabtealfirst!55!white}0.564 & 12.0 & 0.61 & 316.6 & 26.74 & \cellcolor{tabtealfirst!55!white}0.847 & \cellcolor{tabtealfirst!55!white}0.382 \\
\addlinespace[2pt]
\multicolumn{19}{@{}c@{}}{\textcolor{black!65}{\textit{Hybrid primitives}}} \\
BBSplat & \cellcolor{tabbluefirst!93!blue}21.1 & \cellcolor{tabbluefirst!93!blue}1.27 & \cellcolor{tabbluefirst!93!blue}132.0 & 19.00 & 0.624 & 0.548 & \cellcolor{tabbluefirst!93!blue}37.0 & \cellcolor{tabbluefirst!93!blue}3.58 & \cellcolor{tabbluefirst!93!blue}87.7 & 21.34 & 0.535 & 0.603 & \cellcolor{tabbluefirst!93!blue}14.4 & \cellcolor{tabbluefirst!93!blue}0.65 & \cellcolor{tabbluefirst!93!blue}179.9 & 26.90 & 0.833 & 0.380 \\
Nexels & 72.2 & 3.62 & 99.6 & \cellcolor{tabbluefirst!93!blue}21.32 & \cellcolor{tabbluefirst!93!blue}0.695 & \cellcolor{tabbluefirst!93!blue}0.424 & 82.7 & 7.15 & 49.9 & \cellcolor{tabbluefirst!93!blue}22.49 & \cellcolor{tabbluefirst!93!blue}0.590 & \cellcolor{tabbluefirst!93!blue}0.493 & 29.0 & 2.21 & 155.5 & \cellcolor{tabbluefirst!93!blue}28.23 & \cellcolor{tabbluefirst!93!blue}0.876 & \cellcolor{tabbluefirst!93!blue}0.303 \\
\addlinespace[2pt]
\multicolumn{19}{@{}c@{}}{\textcolor{black!65}{\textit{Surface primitives}}} \\
2DGS & \cellcolor{tabsecond}10.6 & \cellcolor{tabthird}1.37 & \cellcolor{tabthird}284.5 & \cellcolor{tabthird}18.52 & \cellcolor{tabthird}0.652 & \cellcolor{tabthird}0.508 & \cellcolor{tabsecond}20.0 & \cellcolor{tabthird}3.24 & \cellcolor{tabthird}201.2 & 17.82 & \cellcolor{tabthird}0.494 & 0.658 & \cellcolor{tabsecond}8.4 & \cellcolor{tabsecond}0.68 & \cellcolor{tabthird}478.0 & \cellcolor{tabthird}21.56 & \cellcolor{tabthird}0.757 & \cellcolor{tabthird}0.476 \\
Triangle Spl. & \cellcolor{tabfirst}9.4 & \cellcolor{tabsecond}0.74 & \cellcolor{tabfirst}495.2 & \cellcolor{tabsecond}21.35 & \cellcolor{tabsecond}0.731 & \cellcolor{tabsecond}0.409 & \cellcolor{tabfirst}15.8 & \cellcolor{tabsecond}1.12 & \cellcolor{tabsecond}312.7 & \cellcolor{tabsecond}22.36 & \cellcolor{tabsecond}0.588 & \cellcolor{tabsecond}0.538 & \cellcolor{tabfirst}7.4 & \cellcolor{tabthird}0.69 & \cellcolor{tabsecond}517.5 & \cellcolor{tabsecond}26.04 & \cellcolor{tabsecond}0.835 & \cellcolor{tabsecond}0.381 \\
MeshSplatting & 17.2 & 2.15 & 119.6 & 16.29 & 0.565 & 0.567 & 33.1 & 4.49 & 60.8 & \cellcolor{tabthird}19.31 & 0.475 & \cellcolor{tabthird}0.634 & \cellcolor{tabthird}13.3 & 1.92 & 130.7 & 21.14 & 0.666 & 0.505 \\
\textbf{Hexels V1} & \cellcolor{tabthird}13.3 & \cellcolor{tabfirst}0.56 & \cellcolor{tabsecond}438.4 & \cellcolor{tabfirst}23.05 & \cellcolor{tabfirst}0.811 & \cellcolor{tabfirst}0.295 & \cellcolor{tabthird}24.2 & \cellcolor{tabfirst}0.90 & \cellcolor{tabfirst}377.8 & \cellcolor{tabfirst}24.82 & \cellcolor{tabfirst}0.705 & \cellcolor{tabfirst}0.410 & \cellcolor{tabsecond}8.4 & \cellcolor{tabfirst}0.55 & \cellcolor{tabfirst}830.8 & \cellcolor{tabfirst}28.94 & \cellcolor{tabfirst}0.890 & \cellcolor{tabfirst}0.303 \\
\bottomrule
\end{tabular}}
\end{table*}

\begin{table*}[t]
\caption{\textbf{Matched primitive budgets.} At matched primitive-count budgets, Hexels combines high reconstruction quality with low peak training VRAM and compact models.}
\label{tab:nvs_count_mip_outdoor}
\label{table:fixed_primitives}
\label{tab:nvs_count_dtu}
\footnotesize
\flushrules
\setlength{\tabcolsep}{2pt}
\renewcommand{\arraystretch}{1.14}
\resizebox{\linewidth}{!}{%
\begin{tabular}{@{}lrrrrrrr@{\hspace{1pt}}rrrrrrr@{\hspace{1pt}}rrrrrrr@{}}
\toprule
\multicolumn{1}{@{}l}{\multirow{3}{*}{Method}} & \multicolumn{7}{c}{Tanks and Temples (10)} & \multicolumn{7}{c}{Mip-NeRF 360 (9)} & \multicolumn{7}{c@{}}{DTU (15)} \\
\cmidrule(lr){2-8}\cmidrule(lr){9-15}\cmidrule(l){16-22}
 & Train & VRAM & FPS & PSNR & SSIM & LPIPS & Deploy. & Train & VRAM & FPS & PSNR & SSIM & LPIPS & Deploy. & Train & VRAM & FPS & PSNR & SSIM & LPIPS & Deploy. \\
 & min $\downarrow$ & GiB $\downarrow$ & $\uparrow$ & dB $\uparrow$ & $\uparrow$ & VGG $\downarrow$ & MB $\downarrow$ & min $\downarrow$ & GiB $\downarrow$ & $\uparrow$ & dB $\uparrow$ & $\uparrow$ & VGG $\downarrow$ & MB $\downarrow$ & min $\downarrow$ & GiB $\downarrow$ & $\uparrow$ & dB $\uparrow$ & $\uparrow$ & VGG $\downarrow$ & MB $\downarrow$ \\
\midrule
\multicolumn{22}{@{}c@{}}{\textcolor{black!65}{\textit{Volumetric primitives}}} \\
3DGS & \cellcolor{tabtealfirst!55!white}7.5 & 2.11 & \cellcolor{tabtealfirst!55!white}581.8 & \cellcolor{tabtealfirst!55!white}22.95 & \cellcolor{tabtealfirst!55!white}0.773 & \cellcolor{tabtealfirst!55!white}0.361 & \cellcolor{tabtealfirst!55!white}45.15 & \cellcolor{tabtealfirst!55!white}9.8 & 2.79 & \cellcolor{tabtealfirst!55!white}607.2 & 23.38 & 0.599 & 0.527 & \cellcolor{tabtealfirst!55!white}28.45 & \cellcolor{tabtealfirst!55!white}5.1 & \cellcolor{tabtealfirst!55!white}0.64 & \cellcolor{tabtealfirst!55!white}819.6 & \cellcolor{tabtealfirst!55!white}29.29 & \cellcolor{tabtealfirst!55!white}0.893 & 0.324 & 8.43 \\
3D Convex & 29.2 & \cellcolor{tabtealfirst!55!white}1.60 & 114.3 & 22.20 & 0.758 & 0.376 & 46.90 & 45.7 & \cellcolor{tabtealfirst!55!white}1.90 & 83.0 & \cellcolor{tabtealfirst!55!white}23.88 & \cellcolor{tabtealfirst!55!white}0.652 & \cellcolor{tabtealfirst!55!white}0.466 & 30.35 & 15.0 & 0.82 & 213.3 & 28.24 & 0.887 & \cellcolor{tabtealfirst!55!white}0.305 & \cellcolor{tabtealfirst!55!white}8.12 \\
\addlinespace[2pt]
\multicolumn{22}{@{}c@{}}{\textcolor{black!65}{\textit{Hybrid primitives}}} \\
BBSplat & \cellcolor{tabbluefirst!93!blue}64.7 & 8.90 & 49.8 & \cellcolor{tabbluefirst!93!blue}24.41 & \cellcolor{tabbluefirst!93!blue}0.842 & \cellcolor{tabbluefirst!93!blue}0.224 & \cellcolor{tabbluefirst!93!blue}146.64 & \cellcolor{tabbluefirst!93!blue}77.7 & 14.45 & 34.1 & \cellcolor{tabbluefirst!93!blue}26.50 & \cellcolor{tabbluefirst!93!blue}0.767 & \cellcolor{tabbluefirst!93!blue}0.296 & \cellcolor{tabbluefirst!93!blue}96.28 & \cellcolor{tabbluefirst!93!blue}22.0 & \cellcolor{tabbluefirst!93!blue}1.88 & 125.4 & \cellcolor{tabbluefirst!93!blue}30.19 & \cellcolor{tabbluefirst!93!blue}0.913 & 0.228 & \cellcolor{tabbluefirst!93!blue}25.91 \\
Nexels & 76.9 & \cellcolor{tabbluefirst!93!blue}4.86 & \cellcolor{tabbluefirst!93!blue}77.8 & 23.48 & 0.803 & 0.272 & 183.02 & 91.2 & \cellcolor{tabbluefirst!93!blue}8.04 & \cellcolor{tabbluefirst!93!blue}44.2 & 25.49 & 0.736 & 0.304 & 164.99 & 31.5 & 3.14 & \cellcolor{tabbluefirst!93!blue}142.9 & 29.00 & 0.891 & \cellcolor{tabbluefirst!93!blue}0.209 & 143.37 \\
\addlinespace[2pt]
\multicolumn{22}{@{}c@{}}{\textcolor{black!65}{\textit{Surface primitives}}} \\
2DGS & \cellcolor{tabsecond}12.4 & 2.64 & 193.2 & 21.40 & 0.741 & 0.396 & 40.01 & \cellcolor{tabsecond}22.0 & 5.15 & 161.9 & 21.65 & 0.573 & 0.566 & 27.98 & \cellcolor{tabthird}9.1 & 1.01 & 364.8 & 24.83 & 0.832 & 0.399 & 8.29 \\
Triangle Spl. & \cellcolor{tabfirst}11.6 & 1.08 & \cellcolor{tabsecond}407.2 & 22.90 & 0.803 & 0.310 & 47.63 & \cellcolor{tabfirst}18.4 & 1.32 & 278.7 & 24.73 & 0.677 & 0.434 & 30.21 & \cellcolor{tabfirst}8.4 & 0.79 & 488.1 & 27.14 & 0.860 & 0.331 & 8.39 \\
MeshSplatting & 18.0 & 2.59 & 116.8 & 16.42 & 0.599 & 0.548 & \cellcolor{tabthird}28.33 & 31.5 & 4.54 & 60.9 & 18.93 & 0.498 & 0.620 & \cellcolor{tabthird}19.90 & 12.5 & 1.96 & 135.3 & 22.96 & 0.729 & 0.471 & \cellcolor{tabthird}4.58 \\
\textbf{Hexels V1} & \cellcolor{tabthird}13.3 & \cellcolor{tabfirst}0.56 & \cellcolor{tabfirst}438.4 & \cellcolor{tabthird}23.05 & \cellcolor{tabthird}0.811 & \cellcolor{tabthird}0.295 & \cellcolor{tabfirst}16.22 & \cellcolor{tabthird}24.2 & \cellcolor{tabfirst}0.70 & \cellcolor{tabfirst}377.8 & \cellcolor{tabthird}24.82 & \cellcolor{tabthird}0.705 & \cellcolor{tabthird}0.410 & \cellcolor{tabfirst}10.22 & \cellcolor{tabfirst}8.4 & \cellcolor{tabfirst}0.55 & \cellcolor{tabsecond}830.8 & \cellcolor{tabthird}28.94 & \cellcolor{tabthird}0.890 & \cellcolor{tabthird}0.303 & \cellcolor{tabfirst}3.04 \\
\textbf{Hexels V2} & 14.4 & \cellcolor{tabsecond}0.60 & \cellcolor{tabthird}386.9 & \cellcolor{tabsecond}23.49 & \cellcolor{tabsecond}0.819 & \cellcolor{tabsecond}0.287 & \cellcolor{tabsecond}22.77 & 24.5 & \cellcolor{tabsecond}0.74 & \cellcolor{tabsecond}375.0 & \cellcolor{tabsecond}25.06 & \cellcolor{tabsecond}0.711 & \cellcolor{tabsecond}0.405 & \cellcolor{tabsecond}14.35 & \cellcolor{tabsecond}8.5 & \cellcolor{tabsecond}0.56 & \cellcolor{tabfirst}845.6 & \cellcolor{tabsecond}29.57 & \cellcolor{tabsecond}0.895 & \cellcolor{tabsecond}0.298 & \cellcolor{tabsecond}4.26 \\
\textbf{Hexels V3} & 14.6 & \cellcolor{tabthird}0.71 & 370.6 & \cellcolor{tabfirst}23.78 & \cellcolor{tabfirst}0.824 & \cellcolor{tabfirst}0.281 & 48.99 & 24.3 & \cellcolor{tabthird}0.81 & \cellcolor{tabthird}366.9 & \cellcolor{tabfirst}25.41 & \cellcolor{tabfirst}0.716 & \cellcolor{tabfirst}0.401 & 30.86 & \cellcolor{tabsecond}8.5 & \cellcolor{tabthird}0.58 & \cellcolor{tabthird}792.3 & \cellcolor{tabfirst}29.95 & \cellcolor{tabfirst}0.899 & \cellcolor{tabfirst}0.293 & 9.15 \\
\bottomrule
\end{tabular}}
\end{table*}

\subsection{Novel View Synthesis}
 
\paragraph{Matched memory.}
\label{sec:fixed_memory}
Like BBSplat's fixed-storage comparisons~\cite{svitov2025bbsplat}, we constrain baseline primitive counts and other learned components to compare methods under matched uncompressed model-memory budgets.
Table~\ref{table:fixed_memory} compares all methods at the V1 model-memory
budget along with corresponding training resources and rendering rates; additional comparisons at the V2 and V3 budgets are provided in Appendix~\ref{app:nvs_budgets}. At matched model-memory budgets, Hexels V1 improves PSNR over 2DGS
by 4.53--7.38\,dB and over Triangle Splatting by 1.70--2.90\,dB across the
three benchmarks. Average rendering speed (FPS) range from 377.8 to 830.8\,FPS, or $1.54$--$1.87\times$ those of 2DGS. On Mip-NeRF360, the mean PSNR gain is 6.99\,dB. All these improvements come with upto $3.6\times$ less VRAM usage than 2DGS. 
Figure~\ref{fig:quality_memory} shows that Hexels retain fine scene structure
under these compact memory budgets. This matched-memory evaluation is our
primary efficiency setting, in which Hexels establish the strongest overall
quality--efficiency trade-off among the evaluated methods. Full experimental
details are provided in Appendix~\ref{sec:nvs_evaluation_protocol}.

\begin{figure*}[tp]
\centering
\includegraphics[width=\textwidth]{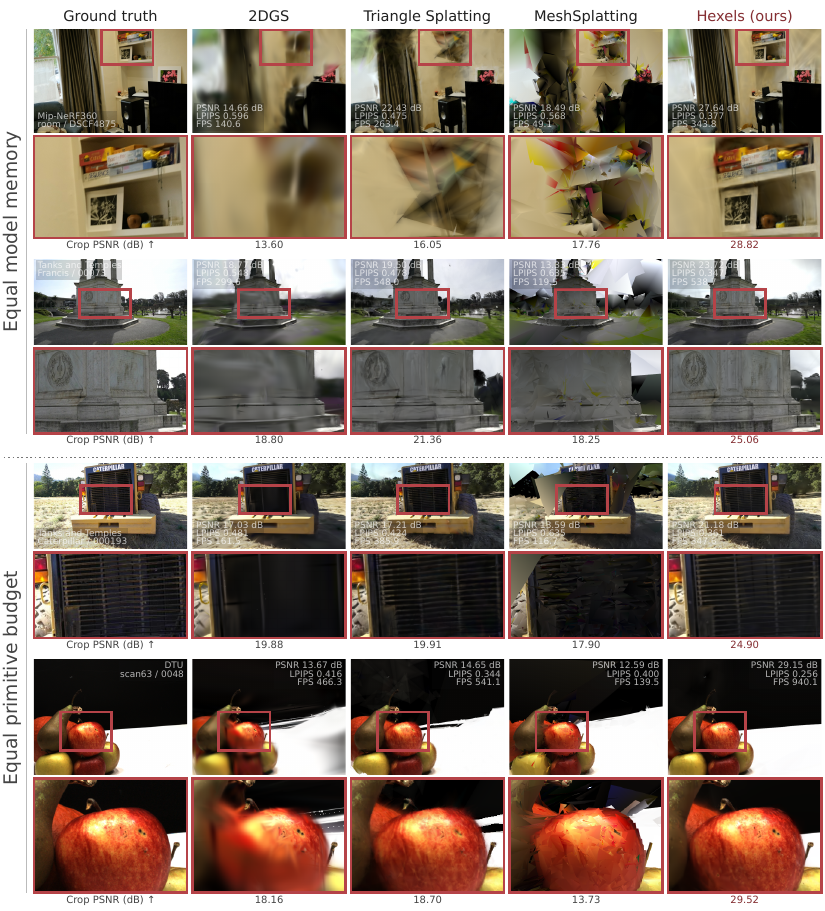}
\caption{\textbf{Qualitative comparisons under matched model-memory
and primitive budgets.}
Top: fixed Hexels V1 on room (Mip-NeRF360) and Francis (Tanks and Temples).
Bottom: fixed V3 on Caterpillar (Tanks and Temples) and DTU scan63.
Zoomed crops highlight sharper shelf detail, pedestal carvings,
radiator slats, and fruit texture achieved by Hexels over baseline methods.}
\label{fig:quality_memory}
\label{fig:quality_count}
\end{figure*}

\begin{table*}[t]
\caption{\textbf{Native novel-view synthesis.} Native baseline settings are compared with densified Hexels V3 and fixed-population V2. Deploy. gives deployment size in MB; Prims. gives primitive count in thousands. Mip-NeRF 360 uses scene-count-weighted means over five outdoor and four indoor scenes. Colors rank methods within each primitive group.}
\label{tab:nvs_native_mip_outdoor}
\label{table:unconstrained}
\label{tab:nvs_native_dtu}
\footnotesize
\flushrules
\setlength{\tabcolsep}{3pt}
\renewcommand{\arraystretch}{1.14}
\fittabular{25}{%
\begin{tabular}{@{}lcccccccc@{\hspace{0pt}}cccccccc@{\hspace{0pt}}cccccccc@{}}
\toprule
\multicolumn{1}{@{}l}{\multirow{3}{*}{Method}} & \multicolumn{8}{c}{Tanks and Temples (10)} & \multicolumn{8}{c}{Mip-NeRF 360 (9)} & \multicolumn{8}{c@{}}{DTU (15)} \\
\cmidrule(lr){2-9}\cmidrule(lr){10-17}\cmidrule(l){18-25}
 & Train & VRAM & FPS & PSNR & SSIM & LPIPS & Deploy. & Prims. & Train & VRAM & FPS & PSNR & SSIM & LPIPS & Deploy. & Prims. & Train & VRAM & FPS & PSNR & SSIM & LPIPS & Deploy. & Prims. \\
 & min $\downarrow$ & GiB $\downarrow$ & $\uparrow$ & dB $\uparrow$ & $\uparrow$ & VGG $\downarrow$ & MB $\downarrow$ & K & min $\downarrow$ & GiB $\downarrow$ & $\uparrow$ & dB $\uparrow$ & $\uparrow$ & VGG $\downarrow$ & MB $\downarrow$ & K & min $\downarrow$ & GiB $\downarrow$ & $\uparrow$ & dB $\uparrow$ & $\uparrow$ & VGG $\downarrow$ & MB $\downarrow$ & K \\
\midrule
\multicolumn{25}{@{}c@{}}{\textcolor{black!65}{\textit{Volumetric primitives}}} \\
3DGS & \cellcolor{tabtealfirst!55!white}18.1 & \cellcolor{tabtealfirst!55!white}4.65 & \cellcolor{tabtealfirst!55!white}157.5 & 24.35 & \cellcolor{tabtealfirst!55!white}0.846 & 0.237 & \cellcolor{tabtealfirst!55!white}294.75 & 1188.5 & \cellcolor{tabtealfirst!55!white}32.8 & \cellcolor{tabtealfirst!55!white}8.60 & \cellcolor{tabtealfirst!55!white}94.7 & \cellcolor{tabtealfirst!55!white}27.55 & \cellcolor{tabtealfirst!55!white}0.817 & 0.259 & \cellcolor{tabtealfirst!55!white}680.48 & 2743.8 & \cellcolor{tabtealfirst!55!white}10.5 & 1.72 & \cellcolor{tabtealfirst!55!white}301.8 & \cellcolor{tabtealfirst!55!white}28.91 & \cellcolor{tabtealfirst!55!white}0.910 & \cellcolor{tabtealfirst!55!white}0.249 & 97.27 & 392.2 \\
3D Convex & 70.2 & 13.95 & 44.4 & \cellcolor{tabtealfirst!55!white}24.80 & 0.841 & \cellcolor{tabtealfirst!55!white}0.224 & 488.65 & 1586.5 & 97.8 & 14.68 & 35.1 & 27.24 & 0.802 & \cellcolor{tabtealfirst!55!white}0.245 & 792.46 & 2572.9 & 21.8 & \cellcolor{tabtealfirst!55!white}1.32 & 136.3 & 28.49 & 0.898 & 0.254 & \cellcolor{tabtealfirst!55!white}22.21 & 72.1 \\
\addlinespace[2pt]
\multicolumn{25}{@{}c@{}}{\textcolor{black!65}{\textit{Hybrid primitives}}} \\
BBSplat & 100.5 & 30.67 & 27.4 & \cellcolor{tabbluefirst!93!blue}24.52 & \cellcolor{tabbluefirst!93!blue}0.844 & \cellcolor{tabbluefirst!93!blue}0.218 & \cellcolor{tabbluefirst!93!blue}229.48 & 300.0 & 156.7 & 21.77 & 14.0 & 26.75 & 0.783 & 0.274 & \cellcolor{tabbluefirst!93!blue}184.02 & 237.8 & \cellcolor{tabbluefirst!93!blue}34.3 & 13.34 & 74.5 & \cellcolor{tabbluefirst!93!blue}30.26 & \cellcolor{tabbluefirst!93!blue}0.914 & 0.221 & \cellcolor{tabbluefirst!93!blue}41.99 & 60.0 \\
Nexels & \cellcolor{tabbluefirst!93!blue}94.5 & \cellcolor{tabbluefirst!93!blue}7.54 & \cellcolor{tabbluefirst!93!blue}45.8 & 24.14 & 0.829 & 0.234 & 241.35 & 399.6 & \cellcolor{tabbluefirst!93!blue}116.3 & \cellcolor{tabbluefirst!93!blue}11.64 & \cellcolor{tabbluefirst!93!blue}25.3 & \cellcolor{tabbluefirst!93!blue}27.35 & \cellcolor{tabbluefirst!93!blue}0.802 & \cellcolor{tabbluefirst!93!blue}0.235 & 241.40 & 399.8 & 43.2 & \cellcolor{tabbluefirst!93!blue}3.82 & \cellcolor{tabbluefirst!93!blue}96.5 & 29.03 & 0.905 & \cellcolor{tabbluefirst!93!blue}0.186 & 241.40 & 399.8 \\
\addlinespace[2pt]
\multicolumn{25}{@{}c@{}}{\textcolor{black!65}{\textit{Surface primitives}}} \\
2DGS & \cellcolor{tabsecond}19.2 & \cellcolor{tabsecond}4.37 & 103.6 & \cellcolor{tabsecond}23.64 & \cellcolor{tabthird}0.829 & \cellcolor{tabthird}0.270 & \cellcolor{tabsecond}177.56 & 727.7 & \cellcolor{tabthird}41.2 & \cellcolor{tabsecond}11.06 & 48.9 & \cellcolor{tabsecond}26.82 & \cellcolor{tabsecond}0.798 & \cellcolor{tabsecond}0.301 & \cellcolor{tabthird}494.72 & 2027.5 & \cellcolor{tabthird}12.7 & \cellcolor{tabthird}2.04 & 213.2 & \cellcolor{tabthird}28.86 & \cellcolor{tabfirst}0.907 & \cellcolor{tabsecond}0.265 & 56.20 & 230.3 \\
Triangle Spl. & 38.6 & 14.44 & \cellcolor{tabthird}121.4 & 23.15 & \cellcolor{tabsecond}0.838 & \cellcolor{tabfirst}0.227 & 1039.86 & 2821.0 & 42.8 & 14.52 & \cellcolor{tabthird}85.8 & \cellcolor{tabfirst}27.02 & \cellcolor{tabfirst}0.808 & \cellcolor{tabfirst}0.233 & 1093.15 & 3595.3 & 16.5 & 2.38 & \cellcolor{tabthird}278.7 & 27.05 & 0.873 & \cellcolor{tabfirst}0.257 & 127.76 & 362.6 \\
MeshSplatting & 49.3 & \cellcolor{tabthird}10.38 & 43.3 & 20.73 & 0.754 & 0.358 & 498.08 & 3694.1 & 73.2 & \cellcolor{tabthird}13.48 & 23.8 & 24.74 & 0.731 & 0.366 & 707.77 & 5379.5 & 17.0 & 5.64 & 94.7 & 25.85 & 0.810 & 0.408 & \cellcolor{tabthird}34.71 & 285.1 \\
\textbf{Hexels V3} & \cellcolor{tabthird}24.3 & 12.06 & \cellcolor{tabsecond}142.7 & \cellcolor{tabfirst}24.53 & \cellcolor{tabfirst}0.839 & \cellcolor{tabsecond}0.248 & \cellcolor{tabthird}198.84 & 739.1 & \cellcolor{tabsecond}37.5 & 13.95 & \cellcolor{tabsecond}100.7 & \cellcolor{tabthird}26.52 & \cellcolor{tabthird}0.780 & \cellcolor{tabthird}0.307 & \cellcolor{tabsecond}246.76 & 917.3 & \cellcolor{tabsecond}10.6 & \cellcolor{tabsecond}1.21 & \cellcolor{tabsecond}483.1 & \cellcolor{tabfirst}29.76 & \cellcolor{tabsecond}0.902 & \cellcolor{tabthird}0.268 & \cellcolor{tabsecond}27.03 & 100.4 \\
\textbf{Hexels V2} & \cellcolor{tabfirst}14.4 & \cellcolor{tabfirst}0.60 & \cellcolor{tabfirst}386.9 & \cellcolor{tabthird}23.49 & 0.819 & 0.287 & \cellcolor{tabfirst}22.77 & 182.1 & \cellcolor{tabfirst}24.5 & \cellcolor{tabfirst}0.74 & \cellcolor{tabfirst}375.0 & 25.06 & 0.711 & 0.405 & \cellcolor{tabfirst}14.35 & 114.7 & \cellcolor{tabfirst}8.5 & \cellcolor{tabfirst}0.56 & \cellcolor{tabfirst}845.6 & \cellcolor{tabsecond}29.57 & \cellcolor{tabthird}0.895 & 0.298 & \cellcolor{tabfirst}4.26 & 34.0 \\
\bottomrule
\end{tabular}}
\end{table*}

\paragraph{Matched primitive budget.}
\label{sec:fixed_primitives}

We use this equal-primitive experiment to ablate the Hexel primitive design as a
whole under an equal primitive population ceiling. Table~\ref{table:fixed_primitives}
uses each scene's SfM point count as the common budget across all methods. Hexel V1/V2 provide compact alternatives, while V3 uses a richer appearance model with a larger per-primitive memory cost. V3 outperforms 2DGS, Triangle Splatting, and MeshSplatting in mean PSNR across all three datasets. V3 exceeds Triangle Splatting in mean PSNR by 0.68\,dB on Mip-NeRF360, 2.82\,dB on DTU and 0.88\,dB on Tanks and Temples. Hexels remain the most efficient primitives amongst all compared surface alignment methods. While we find that BBSplat sometimes outperforms Hexels, it comes at the expense of all relevant efficiency metrics, including training time, FPS, VRAM, and memory, as it requires a separate appearance model and therefore belongs to a separate primitive category. Figure~\ref{fig:quality_count} shows the detail retained by Hexels under this common primitive budget. At the same primitive budget, Hexels V2 outperforms 2DGS on all quality metrics across all three datasets, improving PSNR by up to 4.74\,dB while using up to $7\times$ lower peak training VRAM, 1.76--1.95$\times$ less deployment memory, and rendering up to 2.31$\times$ faster.

\paragraph{Native baselines.}
\label{sec:native}
Although unconstrained training is not the primary focus of our
efficiency-oriented evaluation, we also compare against all baselines in their
native settings in Table~\ref{table:unconstrained}. This tests whether Hexels
remain competitive when model memory and primitive count are unrestricted.
On Mip-NeRF360, densified V3 is within 0.30\,dB of 2DGS while rendering
$2.06\times$ faster and using $2.00\times$ less deployment memory.
On DTU and T\&T, it improves PSNR over 2DGS by 0.90 and 0.90\,dB,
respectively, while rendering $2.26\times$ and $1.37\times$ faster.
Hexel (FP32 deployment) memory is 27.03\,MB versus 56.20\,MB for 2DGS on DTU, and 198.84\,MB versus 177.56\,MB on T\&T. BBSplat remains
0.50\,dB better on DTU, while 3D Convex Splatting is 0.27\,dB better on T\&T; however, neither provides a direct comparison to Hexels as a native surface-aligned primitive. The fixed V2 variant provides the corresponding lower-cost yet high-performance alternative. In fact, for comparative visual quality, V2 uses upto $14\times$ less peak VRAM during training and upto $34\times$ less memory during inference compared to 2DGS. Figure~\ref{fig:quality_native} compares densified
V3 with the native/unconstrained baselines on Mip-NeRF360 and Tanks and Temples scenes. 

Across 34 scenes, Hexels consistently achieve a strong quality--efficiency trade-off. In our primary
matched-memory setting, Hexels outperform every evaluated baseline in PSNR,
including volumetric and textured or neural representations. We further
evaluate the primitive design as a whole under the same primitive-count
budget. In this setting, Hexels outperform all evaluated native
surface-aligned methods while using substantially less training VRAM and,
for V1/V2, less deployment/inference storage memory; Hexels also achieves the highest FPS on each
dataset. Finally, Hexels remain competitive in the native unconstrained
setting, where the fixed and densified variants together provide the strongest
overall quality--efficiency trade-off among the evaluated native
surface-aligned representations.

\begin{figure*}[ht]
\centering
\includegraphics[width=\textwidth]{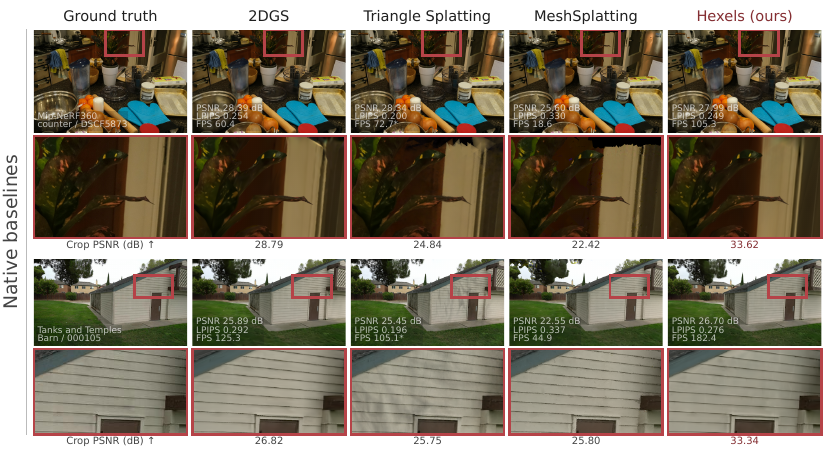}
\caption{\textbf{Unconstrained/Native NVS on Mip-NeRF360 and Tanks and Temples.}
Densified Hexels V3 preserves fine detail, with PSNR gains (in crops) of 4.84\,dB and 6.52\,dB over the best baseline.}
\label{fig:quality_native}
\end{figure*}

\subsection{Efficient Surface Reconstruction}
\label{sec:geometry_main}

We further evaluate surface reconstruction on all 15 DTU scans and 6 Tanks and Temples surface scenes. As can be seen from the results (Table~\ref{tab:dtu_memory_main}), our primary comparison is with 2DGS, and our primary setting is memory-matched and primitive-matched as before, while we also report unconstrained results. Table~\ref{tab:dtu_memory_main} reports mean DTU Chamfer distance (CD). Hexels reaches 0.978\,mm, compared with 1.375\,mm and 1.449\,mm for the memory- and primitive-controlled 2DGS variants. Triangle Splatting and MeshSplatting obtain 1.124\,mm and 1.662\,mm, respectively. Unconstrained 2DGS achieves 0.794\,mm, however trails Hexels in memory usage which use 0.61\,GiB mean peak training VRAM, versus 1.02\,GiB and 0.98\,GiB for the two memory/primitive matched 2DGS variants. Notably, Hexels in Table~\ref{tab:dtu_memory_main} are fixed, whereas other methods allow densification and pruning even beyond the final memory/primitive capacity. We report further improvements in surface reconstruction with densification for Hexels in Appendix~\ref{app:surface_ablation}. In summary, Hexels retain substantially better surface geometry for the same model memory and even the same number of primitives.

On Tanks and Temples, Table~\ref{tab:tnt_surface_main} reports the average of
the 6 scene-wise F1 scores at the official thresholds. Hexels achieves
34.38\%, versus 31.72\% for unconstrained 2DGS, and 22.05\% and 21.03\% for our main memory- and primitive-controlled settings, followed by 18.63\% for Triangle Splatting
and 27.13\% for MeshSplatting.

Figure~\ref{fig:dtu_surface_qualitative} shows qualitative results of surface reconstruction. Hexels demonstrate
reduced wall holes and surface fragmentation despite using significantly less VRAM and memory. 
\begin{table}[t]
\caption{\textbf{DTU surface reconstruction.}
We report mean Chamfer distance (CD) over 15 DTU scans. Hexels achieves 0.978\,mm,
vs 1.375\,mm and 1.449\,mm for memory-matched and primitive-matched 2DGS, respectively with lower peak VRAM usage than all 2DGS variants.}
\label{tab:dtu_memory_main}
\small
\flushrules
\renewcommand{\arraystretch}{1.15}
\fittabular[6pt]{3}{%
\begin{tabular}{@{}lrr@{}}
\toprule
Method & CD (mm)$\downarrow$ & VRAM (GiB)$\downarrow$ \\
\midrule
2DGS unconstrained & \cellcolor{tabfirst}0.794 & 1.93 \\
2DGS matched mem. & 1.375 & \cellcolor{tabthird}1.02 \\
2DGS matched prim. & 1.449 & \cellcolor{tabsecond}0.98 \\
Triangle Splatting & \cellcolor{tabthird}1.124 & 2.29 \\
MeshSplatting & 1.662 & 2.65 \\
\textbf{Hexels (V3)} & \cellcolor{tabsecond}0.978 & \cellcolor{tabfirst}0.61 \\
\bottomrule
\end{tabular}}
\end{table}

\begin{table}[t]
\caption{\textbf{Tanks and Temples surface reconstruction.}
We report mean F1 over six scenes, with higher values indicating better reconstruction. Hexels achieves the highest score among native and budget-controlled 2DGS, as well as reproduced triangle and mesh baselines.}
\label{tab:tnt_surface_main}
\begingroup
\small
\flushrules
\renewcommand{\arraystretch}{1.3}
\fittabular[6pt]{2}{%
\begin{tabular}{@{}lr@{}}
\toprule
\multicolumn{2}{c}{Tanks and Temples (6 scenes)}\\
\cmidrule(lr){1-2}
Method & F1 (\%)$\uparrow$\\
\midrule
2DGS unconstrained & \cellcolor{tabsecond}31.72 \\
2DGS matched mem. & 22.05 \\
2DGS matched prim. & 21.03 \\
Triangle Splatting & 18.63 \\
MeshSplatting & \cellcolor{tabthird}27.13 \\
\textbf{Hexels (V3)} & \cellcolor{tabfirst}34.38 \\
\bottomrule
\end{tabular}}
\endgroup
\end{table}

\begin{figure*}[htb]
\centering
\includegraphics[width=\textwidth]{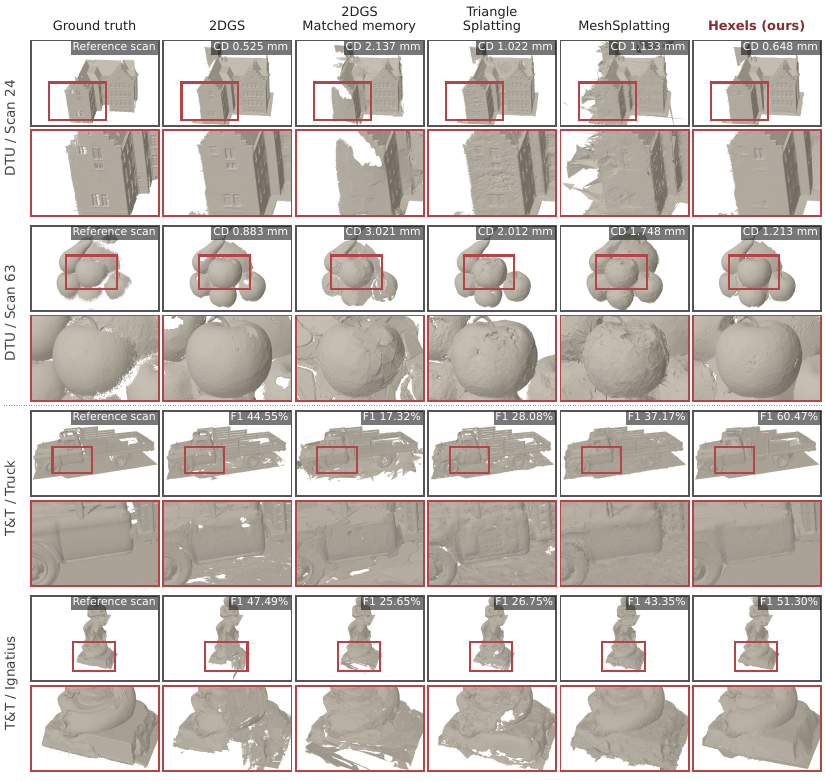}
\caption{\textbf{Efficient Surface reconstruction on DTU and Tanks and Temples scenes.} Hexels reduces wall holes, surface fragmentation (see Scan 63) and broken surfaces relative to memory-matched 2DGS and other surface splatting baselines.}
\label{fig:dtu_surface_qualitative}
\end{figure*}

\subsection{Applications}
\label{sec:applications}

\begin{figure}[htb]
 \centering
 \includegraphics[width=0.47\textwidth]{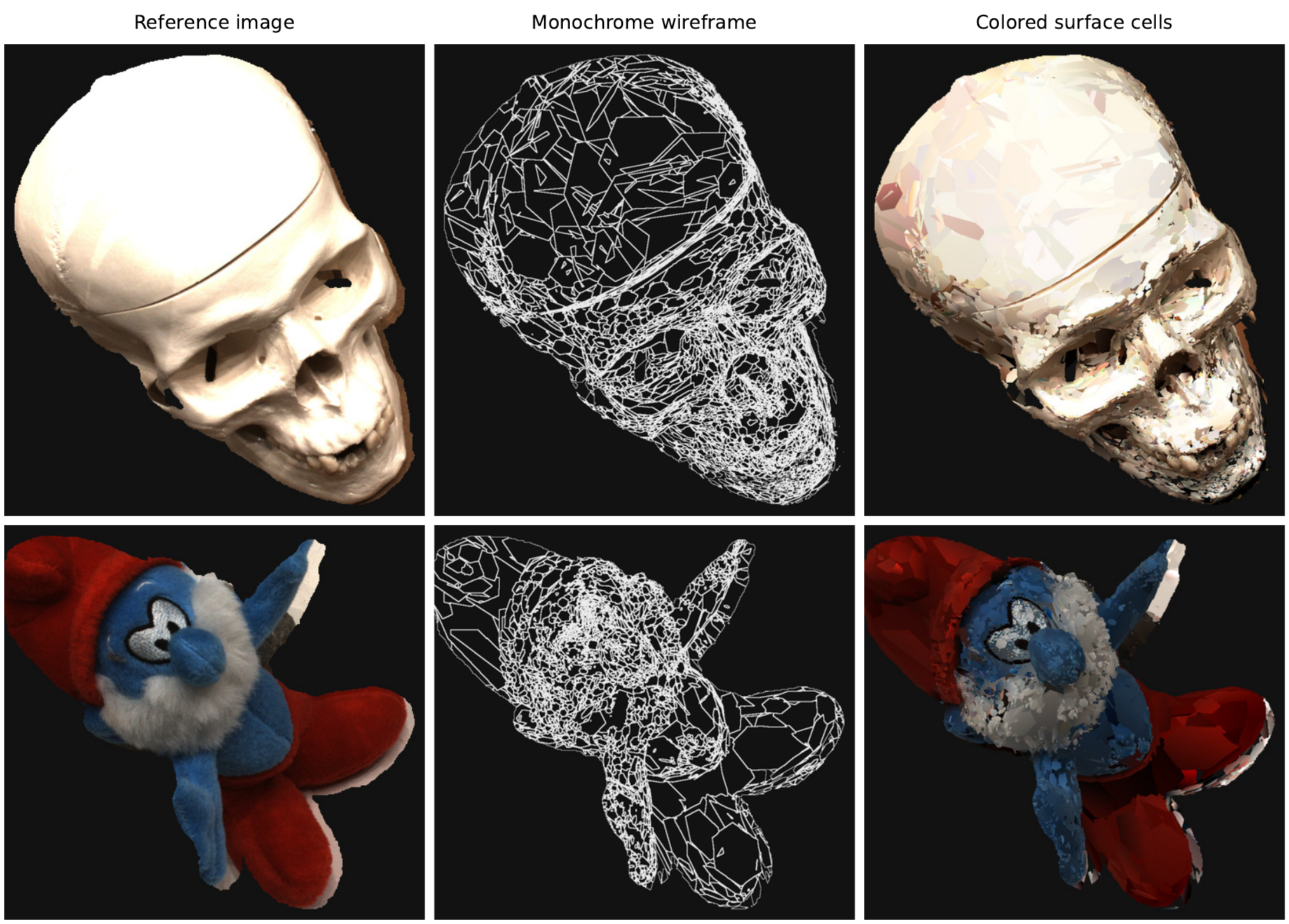}
 \caption{\textbf{Direct surface extraction from Hexels}.} 
 
 \label{fig:surface_export_references}
\end{figure}

\begin{figure*}[htb]
\centering
\includegraphics[width=.9\textwidth]{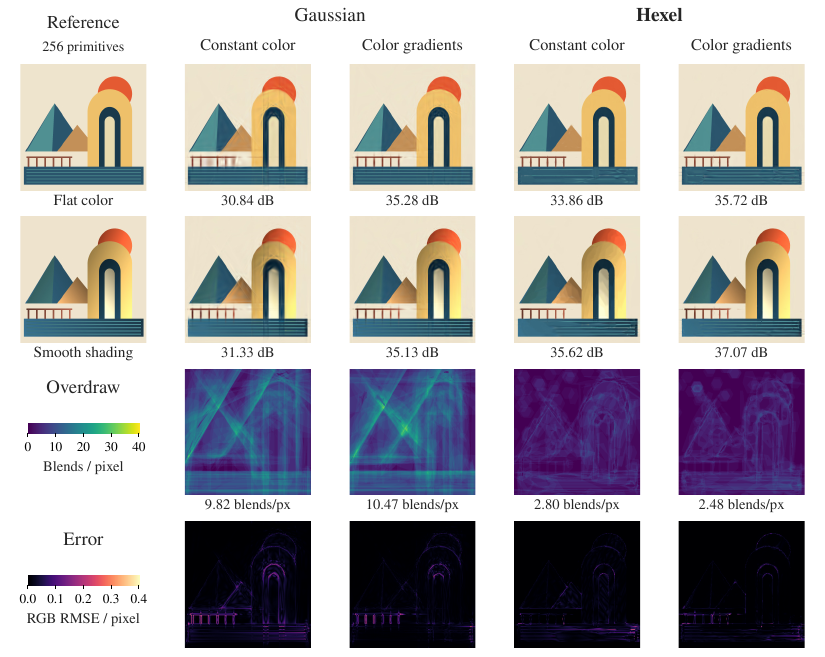}
\caption{\textbf{Shape, color, and overdraw ablation for Hexels.}
Gaussian and Hexel shapes are paired with constant color or spatial RGB color
gradients (V1) under same optimization conditions.
Top: flat and smoothly shaded images using 256 primitives; labels give PSNR.
Bottom: overdraw (how many primitives contribute to every pixel - blends/pixel) and per-pixel RGB error for the shaded target;
}
\label{fig:shape_color_overdraw}
\end{figure*}

\paragraph{Direct Surface-Cell Extraction}
\label{app:surface_export}

Hexels provide an explicit surface representation that can be exported
definably and in constant time after optimization from images. Unlike the depth-fusion pipeline
used for mesh extraction in 2DGS~\cite{huang2024_2dgs}, this conversion
requires no depth rendering, voxel grid, or Marching Cubes. The resulting
surface cells can be used directly in object-centric applications such
as 3D generation, segmentation, and representation learning, which we leave to future work.

For each retained Hexel, we transform the six vertices of its canonical
hexagon, defined by $d=0$, into world space using
Eq.~\eqref{eq:plane-to-world}. Connecting consecutive vertices produces
six edges and one polygonal face. Colors are obtained from the base
RGB or evaluated from the spatial color field. Each cell remains
associated with its parent Hexel, providing direct access to its geometry
and appearance, for any downstream application-based optimization, as the associated Position and parameters remain differentiable with respect to the task.

The conversion requires $O(1)$ work per Hexel. The exported cells can be optimized independently, opening avenues for many downstream applications in both 3D and 4D.
Figure~\ref{fig:surface_export_references} shows the resulting wireframes and surface cells. More details in Appendix~\ref{app:surface_cell_export}.

\paragraph{Resource-constrained reconstruction.}
Compact, fixed-population Hexels retain high reconstruction quality while
requiring little model memory and training VRAM and supporting fast rendering.
These gains arise directly from the primitive and do not rely on texture
atlases, neural decoders, pruning, or quantization. This makes Hexels well
suited to reconstruction and rendering on memory-limited and edge devices.
Pruning and quantization remain optional extensions that could further reduce
deployment storage.

We have also developed an efficient renderer that directly evaluates the same
Hexel representation used during training, without baking or conversion to a
separate deployment representation. We will release it together with the
training code. This provides a direct path toward on-device NVS and surface
reconstruction, where captured images and reconstructed scene data can remain
local to the device.

\subsection{Additional Ablation}
\label{sec:planar_shape_appearance}
To further isolate the effects of Hexel's shape and spatial appearance, we ablate
Gaussian and hexagonal shape supports with constant RGB color or first-order RGB gradients (Hexel V1).
We use planar image fitting for our ablation instead of a 3D scene to keep surface normals and depth order fixed, therefore
avoiding confounding changes in reconstructed geometry and multiview visibility.
All four settings of our ablation use the same primitive count,
initialization, loss (MSE),
and optimization (6K-step Adam). 

This differs from the Image-GS comparison in
Section~\ref{sec:hexels-2d}, which uses normalized top-$K$ blending and quantization-aware training. The objective here is to do one-to-one analysis of shape and appearance.

Figure~\ref{fig:shape_color_overdraw} compares the four ablative settings on flat and
smoothly shaded targets with 256 primitives. At either appearance
model, Hexels achieve higher PSNR than Gaussians on both targets, isolating
the effect of support shape. Adding spatial gradients improves both shapes,
showing the benefit of using our color gradients within each primitive, and providing support to our motivation for the appearance design.
The linear color gradient Hexel achieves the lowest error: on the shaded
target, it reaches 37.07\,dB versus 35.13\,dB for the same appearance Gaussian,
while reducing the overdraw - represented using mean blended contributions of primitives from 10.47 to 2.48 per pixel.
The error and contribution maps in the figure show these improvements.

Figure~\ref{fig:artwork_blending} shows more results regarding the overdraw benefits for complicated and detailed artworks. Hexels require significantly fewer primitives per pixel (reducing mean blends/pixel contributions by 71--80\% relative to their
Gaussian counterparts) to achieve similar perceptual quality. 

\section{Conclusion}
\label{sec:conclusion}

We introduced Hexels, a general surface primitive, for 2D and 3D, that exploits the local structure of the data it represents: hexagonal support efficiently covers
coherent surface regions, while a first-order spatial color field models
appearance variation within each primitive. Across our three 3D benchmarks,
Hexels establish a new quality--efficiency frontier among the evaluated surface-aligned representations, improving PSNR by up to 7.38\,dB at
matched model memory, using lower peak training VRAM and
up to 2$\times$ less learned state at matched primitive count, while rendering up
to 2.31$\times$ faster than 2DGS. The same primitive also performs strongly
for memory-constrained surface reconstruction and directly generalizes to
efficient 2D image representation. Hexels therefore move efficiency into the
primitive itself, opening a path toward high-quality novel-view synthesis
and explicit surface reconstruction on edge and resource-constrained
hardware.

\paragraph{Limitations and Future Work.}
A first-order color model cannot represent arbitrary texture within a
single primitive. High-frequency or uncorrelated appearance requires
smaller primitives or a richer color model, reducing the expected benefit
of exploiting local smoothness. Sparse initialization can also leave parts of a scene
poorly covered, and the independent surface cells do not form a watertight
mesh. Sharper footprints do not by themselves prevent sampling artifacts,
especially at small projected sizes or grazing views.
Hexels does not outperform every baseline on every metric;
the appendix provides the full comparisons. Our current primitives are planar. A natural extension is developable Hexels: curved patches with zero Gaussian curvature that remain locally isometric to the planar canonical hexagon. Such primitives could bend along one principal direction, for example, as cylindrical or conical patches, while preserving the simple 2D footprint and affine appearance parameterization. This may allow individual primitives to cover larger regions of locally near-developable geometry without introducing additional overlap.

\newpage
\bibliographystyle{ieeenat_fullname}
\bibliography{main}

\clearpage
\appendix
\begin{figure*}[!ht]
\centering
\includegraphics[width=\textwidth]{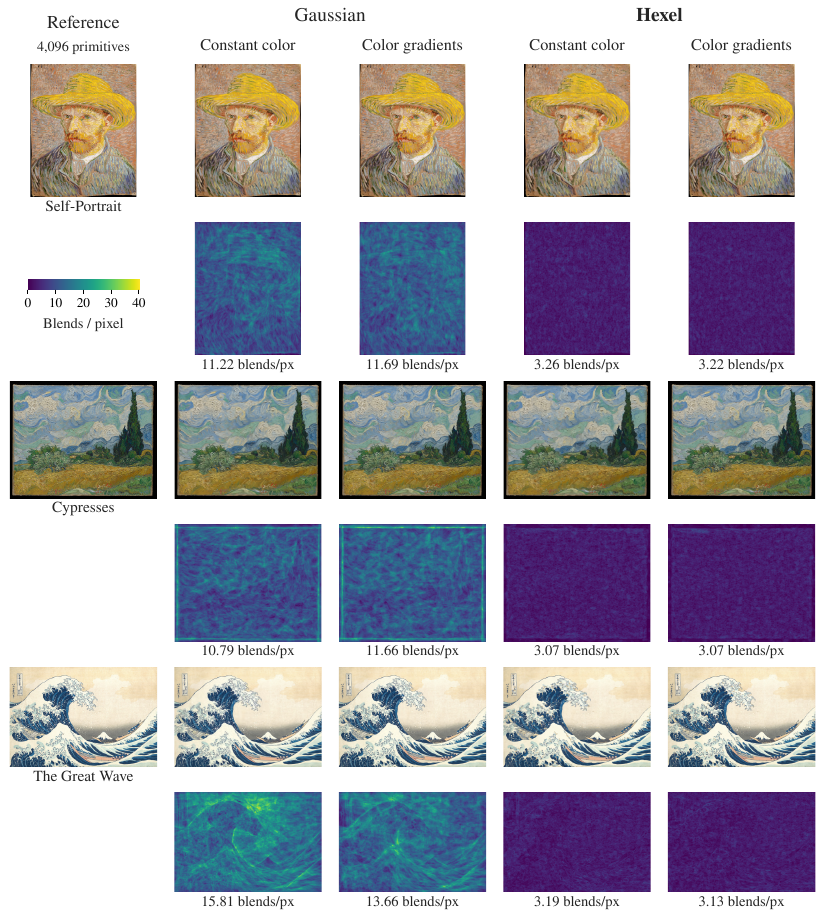}
\caption{\textbf{Overdraw Ablation.} Reference images and Gaussian/Hexel reconstructions are followed by their overdraw maps, using 4,096 primitives and same optimization setup.
Labels below the maps give mean blends/pixel \ie overdraw; The four ablative settings match those in
Figure~\ref{fig:shape_color_overdraw}. This visualization shows that the Hexel primitive requires lower overdraw for similar or better visual quality as compared to Gaussians.}
\label{fig:artwork_blending}
\end{figure*}

\section{Hexel Parameterization}
\label{app:hex-math}

\paragraph{Stored parameters.}
A 3D Hexel stores its center $\bp_i\in\mathbb{R}^3$, two positive tangent
scales $s_{u,i}$ and $s_{v,i}$, a quaternion defining its orientation,
opacity $o_i$, and support sharpness $\beta_i$. The quaternion is normalized
before constructing the tangent frame, and its first two axes define the
orthogonal tangent directions $\bt_{u,i}$ and $\bt_{v,i}$. Thus, the
geometric and support parameters require 11 learned scalars:
three for position, two for scale, four for orientation, one for opacity,
and one for support sharpness.

The V1 appearance model adds a base RGB value and two RGB spatial gradients,
for nine additional scalars. A V1 Hexel therefore contains 20 learned
scalars in total. V2 adds nine directional coefficients, while V3 adds
45 spherical-harmonic coefficients, giving 29 and 65 learned scalars,
respectively. 

\subsection{Appearance}
\label{app:appearance_variants}

The spatial appearance field used by all three variants is
\begin{equation}
    \bc_i^{\mathrm{sp}}(u,v)
    =
    \bc_{0,i}
    + u\,\bg_{u,i}
    + v\,\bg_{v,i},
    \label{eq:hex-spatial-appendix}
\end{equation}
where $\bc_{0,i}\in\mathbb{R}^3$ is the base color and
$\bg_{u,i},\bg_{v,i}\in\mathbb{R}^3$ are RGB gradients along the two local
Hexel coordinates. V1 uses this spatial field directly. V2 and V3 add a
view-dependent RGB residual:
\begin{equation}
    \bc_i(u,v,\bd_i)
    =
    \bc_i^{\mathrm{sp}}(u,v)
    + \br_i(\bd_i),
    \label{eq:hex-color-full}
\end{equation}
where $\bd_i$ is the normalized world-space direction from the camera center
to the Hexel center $\bp_i$, and
\begin{equation}
 \br_i(\bd_i)=
 \begin{cases}
  0, & \text{V1},\\[2pt]
  \bA_i\bd_i, & \text{V2},\\[4pt]
  \displaystyle
  \sum_{\ell=1}^{3}\sum_{m=-\ell}^{\ell}
       \ba_{i,\ell m}Y_{\ell m}(\bd_i),
       & \text{V3}.
 \end{cases}
 \label{eq:hex-appearance}
\end{equation}
Here $\bA_i\in\mathbb{R}^{3\times3}$ gives the nine coefficients of the
linear V2 residual. For V3, $Y_{\ell m}$ denotes a real spherical-harmonic
basis function and $\ba_{i,\ell m}\in\mathbb{R}^3$ its RGB coefficient.
Degrees one through three contain 15 basis functions, giving 45 learned
coefficients; the base RGB $\bc_{0,i}$ provides the view-independent
constant component.

The viewing direction is evaluated once at the Hexel center. Therefore, for
a fixed camera, the V2 and V3 residuals are constant over the primitive:
view dependence changes its RGB offset but not the spatial gradients
$\bg_{u,i}$ and $\bg_{v,i}$. 

\subsection{Derivatives}

\paragraph{Support derivatives.}
Let $\bu=(u,v)^\top$ and define the three unit normals
\begin{equation}
 \bn_0=(0,1)^\top,\quad
 \bn_1=(\sqrt{3}/2,1/2)^\top,\quad
 \bn_2=(\sqrt{3}/2,-1/2)^\top.
\end{equation}
Together, $\mathcal{N}=\{\pm\bn_0,\pm\bn_1,\pm\bn_2\}$ are the outward
normals of the six sides of the canonical hexagon. The hexagonal function
can therefore be written as
\begin{equation}
 d(\bu)
 =
 \max_{\bn\in\mathcal{N}}\bn^\top\bu-1
 =
 \max_{e\in\{0,1,2\}}|\bn_e^\top\bu|-1.
 \label{eq:hex-boundary-normal}
\end{equation}
Its zero level set defines the canonical hexagonal boundary, with $d<0$
inside and $d>0$ outside.

At a point where one projection is strictly larger than the other five,
let $\bn^\star$ denote the corresponding normal. For sufficiently small
changes in $\bu$, this projection remains the largest, so
$d(\bu)=\bn^{\star\top}\bu-1$ and
\begin{equation}
 \nabla_{\bu}d=\bn^\star.
 \label{eq:hex-boundary-grad}
\end{equation}
When two or more projections tie for the largest value, small movements
in different directions can select different side normals. The function
remains continuous, but has no ordinary gradient at that point. Any of
the tied maximizing normals is a valid subgradient: one of the slopes
that can be used for optimization at this corner of the function.
These ties form a zero-area set, so the ordinary gradient exists almost
everywhere. Although $d$ gives the exact hexagonal contour and
inside-outside sign, it is not the Euclidean signed distance outside
the hexagon near its vertices.

For the soft support $K(\bu)=\sigma(-\beta d(\bu))$, the chain rule gives
\begin{equation}
 \nabla_{\bu}K
 =
 -\beta K(1-K)\nabla_{\bu}d,
 \qquad
 \frac{\partial K}{\partial\beta}
 =
 -d(\bu)K(1-K).
 \label{eq:hex-footprint-grad}
\end{equation}
The spatial derivative holds wherever $d$ is differentiable. At a tie,
substituting a selected maximizing normal gives the corresponding gradient
choice for optimization. The derivative with respect to $\beta$ exists
at every point.

The rendered opacity contribution is $\alpha=\min\{oK,0.99\}$.
For $oK<0.99$,
\begin{equation}
 \frac{\partial\alpha}{\partial o}=K,
 \qquad
 \frac{\partial\alpha}{\partial K}=o.
\end{equation}
For $oK>0.99$, $\alpha$ is locally constant, and both derivatives are zero.
At $oK=0.99$, our backward implementation uses the unclamped-branch derivatives.

\paragraph{Color and compositing derivatives.}
For the spatial color field,
$\partial\bc_i^{\mathrm{sp}}/\partial u=\bg_{u,i}$ and
$\partial\bc_i^{\mathrm{sp}}/\partial v=\bg_{v,i}$.
Changes in local coordinates therefore affect both the support weight and
the color within the Hexel.

Let $\bc_{>i}$ be the color obtained by compositing only the retained
Hexels behind primitive $i$, including the background and starting with
transmittance one. With the retained contributions and their order fixed,
let $T_i=\prod_{j<i}(1-\alpha_j)$ be the transmittance reaching primitive
$i$. Then
\begin{equation}
 \frac{\partial\bc(\bx)}{\partial\bc_i}
 =T_i\alpha_i\bI_3,
 \qquad
 \frac{\partial\bc(\bx)}{\partial\alpha_i}
 =T_i\bigl(\bc_i-\bc_{>i}\bigr),
 \label{eq:hex-compositing-grad}
\end{equation}
where $\bI_3$ is the $3\times3$ identity matrix. The first expression
weights the color derivative by the primitive's blending weight. The
second describes how increasing its opacity replaces some of the color
behind it with its own color.

A back-to-front traversal accumulates these derivatives. Gradients
propagate to appearance and opacity, and through the local coordinates
and ray--plane intersection to geometry. For V2 and V3, moving the Hexel
center also changes its view-dependent color. We differentiate through
the numerically stabilized view direction
$\bd_i=(\bp_i-\bo)/\sqrt{\|\bp_i-\bo\|^2+10^{-8}}$,
where $\bo$ is the camera center.

Using $\boldsymbol\delta_i=\bp_i-\bo$ and
$\rho_i=\sqrt{\|\boldsymbol\delta_i\|^2+10^{-8}}$, the additional position
gradient is
\begin{equation}
 \frac{\partial\bd_i}{\partial\bp_i}
 =\frac{\bI_3}{\rho_i}
  -\frac{\boldsymbol\delta_i\boldsymbol\delta_i^\top}{\rho_i^3}.
 \label{eq:hex-view-jacobian}
\end{equation}
The V2 variant has Jacobian $\bA_i$ with
respect to $\bd_i$; Hexel V3 has a differentiable polynomial SH basis. Both
therefore contribute to the position gradient in addition to the
ray--plane geometry.

\paragraph{Optimization coordinates.}
For 3D applications, (positive) scales and sharpness parameters are
represented exponentially, and opacity is derived using a sigmoid:
$s_a=\exp(\ell_a)$, $\beta=\exp(b)$, and $o=\sigma(a)$.
For a latent or upstream gradient $\bar z=\partial\mathcal{L}/\partial z$,
\begin{equation}
 \bar\ell_a=s_a\bar s_a,\qquad
 \bar b=\beta\bar{\beta},\qquad
 \bar a=o(1-o)\bar o.
 \label{eq:hex-positive-gradients}
\end{equation}
For the forward pass, the kernel applies the floor function for sharpness parameter at
$10^{-4}$ for numerical stability. Below this value, its backward pass
uses the support derivative evaluated at the floor, using a
straight-through gradient value.

\section{Rasterization}
\label{app:hex-renderer}

\paragraph{Perspective-correct coordinates.}
Like 2DGS~\cite{huang2024_2dgs}, we recover local Hexel coordinates
through ray--plane intersection. Let $\bT_i\in\mathbb{R}^{3\times3}$
map local homogeneous coordinates $(u,v,1)^\top$ to homogeneous image
coordinates, combining Eq.~\eqref{eq:plane-to-world} with the calibrated
camera. It factors as $\bT_i=\bPi\bM_i$, where the
local-to-world matrix
\begin{equation}
 \begin{aligned}
 \bM_i&=
 \begin{pmatrix}
  s_{u,i}\bt_{u,i}&s_{v,i}\bt_{v,i}&\bp_i\\
  0&0&1
 \end{pmatrix}\in\mathbb{R}^{4\times3},\\
 \bT_i&=\bPi\bM_i.
 \end{aligned}
 \label{eq:hex-projective-map}
\end{equation}
is Eq.~\eqref{eq:plane-to-world} written as a
matrix: it maps $(u,v,1)^\top$ to the world point $P_i(u,v)$ in
homogeneous form, $(P_i(u,v),1)^\top$. It places the Hexel plane in
the world. $\bT_i$ is therefore the exact perspective map of the Hexel
plane, whereas 3DGS~\cite{kerbl3Dgaussians} projects each Gaussian with a
local affine approximation.
Denote its rows by $\bt_x^\top$, $\bt_y^\top$, and $\bt_z^\top$.
For a pixel $(x,y)$, define
\begin{equation}
 \bk=x\bt_z-\bt_x,\qquad
 \bl=y\bt_z-\bt_y,\qquad
 \mathbf{h}=\bk\times\bl.
 \label{eq:hex-ray-intersection}
\end{equation}
The intersection must satisfy
$\bk^\top(u,v,1)^\top=\bl^\top(u,v,1)^\top=0$.
Since $\mathbf{h}$ is orthogonal to both $\bk$ and $\bl$, the local
coordinates are $(u,v)=(h_x/h_z,h_y/h_z)$ when $h_z\neq0$.
This recovers the perspective-correct intersection without a per-pixel
matrix inverse. We skip intersections with $h_z=0$ or camera-space depth
below the near-plane threshold of $0.2$.
Before Hexels are assigned to image tiles
(see Visibility and compositing below), we also discard every Hexel whose
center has camera-space depth at most $0.2$.
The same local coordinates are used to evaluate the hexagonal support
and spatial color field.

\paragraph{Support bounds.}
The sigmoid is nonzero outside the canonical hexagon, so we bound the
region that can contribute above the rasterizer's alpha threshold
$\tau=1/255$. Since $\alpha_i\leq K_i$, a contribution with
$\alpha_i\geq\tau$ must satisfy $K_i\geq\tau$. Substituting
$K_i=\sigma(-\beta_i d)$ gives
\begin{equation}
 d(u,v)
 \leq
 \frac{1}{\beta_i}\log\frac{1-\tau}{\tau}
 =
 \frac{\log254}{\beta_i}.
 \label{eq:hex-support-cutoff}
\end{equation}
This defines a larger regular hexagon with center-to-edge distance
$1+\log254/\beta_i$. It is enclosed by a circle with radius
\begin{equation}
 r_i=\frac{2}{\sqrt{3}}
 \left(1+\frac{\log254}{\beta_i}\right).
 \label{eq:hex-support-radius}
\end{equation}
The radius is in local Hexel coordinates, before tangent scaling and
projection. We use this enclosing circle with the projective bounding
routine inherited from 2DGS to determine the image tiles assigned to
each primitive. The bound is conservative in the local plane because it
uses the maximum possible opacity, rather than the learned opacity of
each Hexel.
The kernel computes a second, tighter
box using the learned opacity. A contribution passes the threshold
only if $o_iK_i\geq\tau$, which holds only inside the smaller hexagon
$d(u,v)\leq\beta_i^{-1}\log(o_i/\tau-1)$;
a Hexel with $o_i<\tau$ with $o_i<\tau$ cannot contribute.

\paragraph{Visibility and compositing.}
We divide the image into $16\times16$-pixel tiles and assign each Hexel
to the tiles intersecting its projected bound. Within each tile,
primitives are sorted by their centers' camera-space depth. This order
is shared by the pixels in the tile; we do not sort intersections
separately at each pixel.

For each candidate intersection, we evaluate its opacity as
$\alpha_i=\min\{o_iK_i(u_i,v_i),0.99\}$ and skip contributions with
$\alpha_i<1/255$. Let $T$ be the remaining transmittance before processing
the candidate. If $T(1-\alpha_i)<10^{-4}$, we stop without accumulating
that candidate or updating $T$. Otherwise, its color contribution is
$T\alpha_i\bc_i(u_i,v_i,\bd_i)$, and we update
$T\leftarrow T(1-\alpha_i)$.

For the $m$ contributions actually accumulated at a pixel, indexed in
their compositing order, the final color is
\begin{equation}
 \begin{aligned}
 \bc(\bx)
 &=\sum_{i=1}^{m}T_i\alpha_i\bc_i(u_i,v_i,\bd_i)
   +T_{\mathrm{end}}\bc_{\mathrm{bg}},\\
 T_i&=\prod_{j<i}(1-\alpha_j),
 \qquad
 T_{\mathrm{end}}=\prod_{i=1}^{m}(1-\alpha_i).
 \end{aligned}
 \label{eq:hex-compositing}
\end{equation}
Here $\bc_{\mathrm{bg}}$ is the background color. Its weight is the
transmittance remaining after the contributions actually accumulated,
not the tentative value that triggered early termination.

\section{Densification and Pruning}
\label{app:adaptive-population}
Hexels (V1/V2/V3 - Fixed) do not primarily use densification and pruning to avoid exceeding memory and/or primitive constraints in our primary experiments. Given our focus on practical efficiency, we want to avoid any hacks to improve our performance by temporarily going over the memory or primitive limits - which is why the peak VRAM values are very important, as often, other efficiency-focused methods do not report them. We, however, also report the densified NVS and surface reconstruction counterparts in our results - primarily in unconstrained settings - which are not the focus of our work. Notably, our baseline methods are allowed to temporarily exceed memory and primitive constraints to achieve better results, as they often perform much worse when forced to remain below these requirements.

\paragraph{Mixed growth.}
Our densification strategy is similar to 3DGS and 2DGS. Between two growth triggers, for each Hexel, we track the norm of the loss gradient with respect to its projected 2D center, averaged over the iterations in which the Hexel is visible. A Hexel becomes a \textit{parent} Hexel when this average is large, it was visible often enough, and it is not too young. Small parents are cloned, and large parents are split. A parent Hexel counts as small when its larger tangent scale is at most 1\% of the scene extent. A clone is a copy of the parent Hexel whose center is moved within the parent's tangent plane by a random Gaussian offset, scaled by the parent's two scales. A split replaces the parent with two children, each offset randomly in the same plane, with both scales divided by $1.6$, as in 3DGS, and an opacity $o_{\mathrm{child}}=1-\sqrt{1-o_{\mathrm{parent}}}$. New Hexels copy the parent's appearance coefficients.

\paragraph{Contribution-based pruning.}
Pruning removes Hexels that contribute very little to the scene. Each pruning session renders all training views $\mathcal{V}_{\mathrm{train}}$ and scores Hexel $i$ by its mean blending weight

\begin{equation}
    q_i=\frac{1}{|\mathcal{V}_{\mathrm{train}}|}
 \sum_{v\in\mathcal{V}_{\mathrm{train}}}\frac{1}{|\Omega_v|}
 \sum_{\bx\in\Omega_v}T_{i,v}(\bx)\alpha_{i,v}(\bx),
 \label{eq:weak-pruning-score}
\end{equation}

where $\Omega_v$ is the set of pixels of view $v$ and $T_{i,v}(\bx)\alpha_{i,v}(\bx)$ is the weight of Hexel $i$ at pixel $\bx$ in Eq.~\eqref{eq:hex-compositing}. The weight is zero whenever the Hexel does not contribute -- outside its tiles, below the alpha threshold, or after early termination. A Hexel can be pruned when $q_i$ is below a threshold or has existed for longer than a grace period -- this gives opportunity for newer Hexels to contribute.

\paragraph{Schedule and thresholds.}
All densified (G) V1/V2/V3 models, for both NVS and surface reconstruction, use mixed growth every 500 iterations
from 5.5K through 14.5K, with a ceiling of 2M Hexels. A parent Hexel needs an average gradient norm of at least $2\times10^{-4}$, at least 20 observations since the previous growth trigger, and a minimum age of 500 iterations.
Clone and split offsets have standard deviations of $0.25$ and $0.5$ times the parent's scales, and a clone's opacity is at most $0.02$. Pruning runs every 1K iterations from 5K to 16K, with a 500-iteration grace period, a threshold of $10^{-7}$ on $q_i$, and at most 2\% of the Hexels are removed per trigger. At each growth trigger, we also remove every Hexel with opacity below $0.005$, regardless of its age or the 2\% limit. Opacities are never reset.

\section{2D Image Representation}
\label{app:hex-2d}
Hexels can represent an image directly by placing the primitives in the
image plane. We retain the hexagonal support and first-order spatial color
model, but replace 3D projection and depth-ordered compositing with
image-space evaluation and blending.

\paragraph{Geometry and function support.}
For a pixel $\bx=(x,y)^\top$, let $\bm_i=(m_{x,i},m_{y,i})^\top$ be the
Hexel center, $\bR(\theta_i)$ its planar rotation, and
$\bS_i=\operatorname{diag}(s_{x,i},s_{y,i})$ its positive scale matrix.
We compute local coordinates and the soft support weight
(for numerical stability) $\widetilde{\bS}_i=\bS_i+10^{-8}\bI_2$:
\begin{equation}
 \bu_i(\bx)=\widetilde{\bS}_i^{-1}\bR(\theta_i)^\top(\bm_i-\bx),
 \qquad
 w_i(\bx)=\sigma\!\left(-\beta_i d(\bu_i(\bx))\right),
 \label{eq:hex-2d-support}
\end{equation}
where $d$ is the hexagonal function from
Eq.~\eqref{eq:hex-boundary} and $\beta_i>0$ is the support sharpness.
This evaluates the support directly at each pixel, without projection
or ray--plane intersection. Since $d(\bu)=d(-\bu)$, reversing the
displacement does not change the support weight. Scales are optimized as log-scales. Hexels (V1) learn $\beta_i=\operatorname{softplus}(\tilde\beta_i)$, initialized near 10; the constant-color ablation uses a fixed $\beta=10$ shared by all primitives and not stored.

\paragraph{Spatial appearance.}
We use the first-order spatial color model of V1, expressed here in
image coordinates:
\begin{equation}
 \bc_i(\bx)
 =
 \bc_{0,i}
 +\eta\left[
 (m_{x,i}-x)\bg_{x,i}
 +(m_{y,i}-y)\bg_{y,i}
 \right],
 \label{eq:hex-2d-color}
\end{equation}
where $\bc_{0,i}$ is the base RGB color,
$\bg_{x,i},\bg_{y,i}\in\mathbb{R}^3$ are the spatial color coefficients,
and $\eta=0.1$ controls their scale. The coefficients describe color changes
along the image axes, using the center-minus-pixel displacement shown
above. Unlike Eq.~\eqref{eq:spatial-color}, these coordinates are not
rotated or normalized by the Hexel scales. Both formulations represent
an affine color field, but use different coordinate parameterizations.
The constant-color ablation omits the two spatial terms, giving
$\bc_i(\bx)=\bc_{0,i}$. Neither model uses view-dependent appearance.

\paragraph{Rendering.}
Following Image-GS~\cite{imagegs2025}, we retain the $K=10$ largest support
weights among the primitives assigned to the pixel's tile. Let
$\mathcal{S}_K(\bx)$ denote their indices. We blend their colors as
\begin{equation}
 \bc(\bx)=
 \frac{\displaystyle
       \sum_{i\in\mathcal{S}_K(\bx)}w_i(\bx)\bc_i(\bx)}
      {\displaystyle
       \sum_{i\in\mathcal{S}_K(\bx)}w_i(\bx)+\epsilon},
 \label{eq:hex-2d-blending}
\end{equation}
where $\epsilon=10^{-4}$ prevents division by zero. Unlike the 3D compositor,
this \textit{blend} has no transmittance, depth ordering, or learned opacity:
each retained primitive contributes according to its support weight
relative to the others. During optimization, gradients propagate through
the selected weights and colors to the Hexel parameters. The selected
indices are treated as fixed when computing these gradients. 
Backward pass reuses the saved top-$K$ indices and
reevaluates only their support weights and colors.

\paragraph{Finite support.}
Each Hexel is assigned to the $16\times16$ tiles overlapping a square of half-width
$r_i=\lceil\tfrac{2}{\sqrt3}\max(s_{x,i},s_{y,i})(1+4.5/\beta_i)\rceil$ pixels, which encloses the
contour $d=4.5/\beta_i$ where $w_i=\sigma(-4.5)\approx0.011$. Every candidate in these tiles is evaluated without a weight threshold, and pixels in tiles without candidates are rendered as zero. Unlike the 3D bound of Eq.~\eqref{eq:hex-support-cutoff}, this cutoff does not use $1/255$.

\paragraph{Experiment details.}
We fit all 45 Image-GS images at $2048\times2048$ resolution for a primary 2D representation experiment.
We use the Image-GS image-gradient initializer for both methods, with 30\% sampled
uniformly; base RGB values are sampled from the target image, scales are initialized at five pixels, while
rotations and spatial color gradients are zero-initialized. Each Hexel is initialized
with 20K primitives. From iterations 550 through 3000, we prune 200 primitives
every 50 iterations, converging to 10K primitives. Hexels use L1 plus $0.1(1-\mathrm{SSIM})$ loss and the Adam optimizer and
stop early at 5100 iterations since Hexels converge faster. In a smaller experiment, we found that further training to 10,000 iterations did see marginal improvements in quality. All parameters and optimizer states are in FP32. For all methods in Table~\ref{tab:hexels-2d} same precision is used for all parameters. Memory size includes every active parameter and any relevant decoding metadata. Figure~\ref{fig:2d_convergence} uses values with the mean of three different experimental runs, with the variance spread displaying the differences in convergence across the entire Image-GS dataset.

\paragraph{Primitive Overdraw.}
\label{app:overlap}
We distinguish the total number of primitives from the number covering
each pixel. Since both Gaussian and Hexel supports have nonzero tails,
we count a primitive as covering a pixel when its support weight reaches
a common fraction $\tau\in(0,1)$ of its own peak:
\begin{equation}
 C_\tau(\bx)=
 \sum_{i=1}^{N}\mathbf{1}\!\left[
 \frac{w_i(\bx)}{\max_{\mathbf y}w_i(\mathbf y)}\geq\tau
 \right].
 \label{eq:overlap-count}
\end{equation}
Here $N$ is the total primitive count and $\mathbf{1}$ is the indicator
function. A value $C_\tau(\bx)=0$ means that no primitive covers the
pixel at this threshold; $C_\tau(\bx)=1$ means that exactly one does,
and larger values indicate overlapping supports. We count supports
before top-$K$ selection, rather than summing the normalized blending
weights. This measures spatial overlap that we refer to as overdraw.

We summarize coverage by the mean of $C_\tau$ over the full image and
the fractions of pixels covered by zero or multiple supports. A mean tending to one can result from uncovered pixels
offsetting heavily overlapping regions. Primitive count and reconstruction
quality are also reported for completeness. Comparing methods at matched
reconstruction quality additionally tests whether Hexels require fewer
primitives, rather than simply reducing overlap at the expense of quality.

\section{Implementation Details}
\label{sec:experimental_details}

\subsection{Datasets and Metrics}
\label{sec:nvs_evaluation_protocol}
We evaluate 34 scenes: nine from Mip-NeRF360, fifteen DTU scans
(24, 37, 40, 55, 63, 65, 69, 83, 97, 105, 106, 110, 114, 118 and 122),
and ten from Tanks and Temples (Barn, Caterpillar, Courthouse, Francis,
Horse, Ignatius, Lighthouse, Meetingroom, Train and Truck).
Every eighth camera is held out, with identical cameras and images across
methods. Mip-NeRF360 uses quarter-resolution outdoor and half-resolution
indoor images; DTU and T\&T use half-resolution.
PSNR and SSIM use float RGB images, and LPIPS-VGG inputs are normalized. We average reported metrics over test views, then equally over scenes for each configuration.

Colors rank methods within each primitive family and dataset. Teal and blue mark the best volumetric and hybrid values, respectively; red, orange, and yellow mark the top three surface values. F/G denote fixed/grown (densified) Hexels.

\subsection{Initialization and Training}
\label{sec:training_details}

Hexels V1-V3 start from sparse COLMAP points and colors unless stated otherwise, with one
Hexel per point and its color used as the base RGB point. Rotations start at identity quaternion initialization, without estimated surface normals. Both tangent scales
are initialized to one quarter of the root-mean-square distance to the three
nearest neighbors, capped at 5\% of the scene camera extent. Opacity is initialized
at 0.1 and sharpness parameter at 10. Spatial color gradients and
view-dependent coefficients are initialized to zero. We optimize geometry, opacity,
sharpness, and appearance parameters jointly. We train using Adam optimizer for 30K iterations using the photometric loss and with SSIM weight 0.2. Experiments using densification, usually in the unconstrained case, use the method described in Appendix~\ref{app:adaptive-population}. Unlike some other baselines, for NVS experiments, Hexels do not use depth distortion, normal regularization, or MVS initialization.
Details regarding learning rates are given in Table~\ref{tab:training_parameters}.

\begin{table}[ht]
\centering
\caption{Initial learning rates for the Hexel V1-V3 models for NVS experiments. Only parameters enabled by a variant are optimized. All rates are halved after 15K, 20K, and 25K iterations.}
\label{tab:training_parameters}
\small
\setlength{\tabcolsep}{4pt}
\renewcommand{\arraystretch}{1.14}
\begin{tabular}{@{}lr@{}}
\toprule
Parameter & Learning rate \\
\midrule
Center & $1.6\times10^{-4}$ \\
Log tangent scales & $5\times10^{-3}$ \\
Rotation quaternion & $10^{-3}$ \\
Opacity logit & $10^{-1}$ \\
Base RGB & $10^{-2}$ \\
Spatial color gradients & $5\times10^{-3}$ \\
Linear view coefficients (V2) & $5\times10^{-3}$ \\
SH coefficients (V3) & $5\times10^{-4}$ \\
Log boundary sharpness & $10^{-2}$ \\
\bottomrule
\end{tabular}
\end{table}

\paragraph{2D image fitting.}
 For the image-representation experiments, Hexels use following learning rates:
$5\times10^{-4}$ for centers, $10^{-2}$ for (log-)scales, $2\times10^{-3}$
for rotation angles, $5\times10^{-3}$ for base RGB, and $10^{-3}$ for spatial
RGB gradients and support-sharpness parameters when present. All rates are
reduced tenfold at iteration 5001. Hexels stop at 5100 iterations after
pruning from 20K to 10K primitives, as detailed in Appendix~\ref{app:hex-2d}.
Image-GS retains its original parameterization and hyperparameters.

\subsection{Baseline Settings}
We use the authors' implementations of 2DGS, 3DGS, 3D Convex Splatting,
Triangle Splatting, BBSplat, MeshSplatting, and Nexels.
BBSplat trains for 32K iterations and includes compressed textures to reduce reported memory requirements post hoc.

Reported model memory, including matched model memory experiments, includes all primitive and appearance parameters,
including textures, feature grids and decoders.
Nexels use a reduced feature grid to fit within each experimental budget.
MeshSplatting is matched at initialization; native remeshing can change
its final memory.
Matched-primitive experiments limit final populations to the scene's
SfM point count while retaining native growth and pruning.
Native comparisons use each baseline's original population policy.

For DTU surface reconstruction experiment, native unconstrained and budget-controlled 2DGS use normal-consistency and depth-distortion loss weights of 0.05 and 1000, respectively. For the constrained experiments, the Hexels primitive or model-memory budget determines the final population ceiling; after initialization and each densification/pruning update, we retain the highest-opacity primitives up to that ceiling. This means that during densification, 2dGS can and does go beyond the set primitive limits - we note that if these limits are used strictly, 2DGS and other baselines perform poorly. This further shows that Hexels, which actually do strictly conform to the memory or primitive budgets, usually without doing densification, are state-of-the-art or surface line primitives in terms of efficiency-quality frontier. 

\subsection{Timing and Memory}
\label{sec:resource_protocol}

\paragraph{Hardware and rendering speed.}
Experiments use RTX A5000 GPUs; starred rows are experiments that were run on RTX 6000 Ada
as A5000 memory was insufficient for that specific experiment. For FPS calculations, we perform two warmup passes and run four trials of ten passes across all test camera viewpoints, with GPU synchronization. We report the median FPS per scene and average over scenes,
excluding loading, image transfer, and metric computation. 

\paragraph{2D image fitting/training timing.}
The 2D fitting times include optimization and pruning, excluding
initialization, evaluation and any post-hoc I/O. For Table~\ref{tab:hexels-2d}, all
three stored representations are decoded and timed using the same A5000
rendering protocol: 10 warmup renders followed by 5 trials of 20
renders, with GPU synchronization. We take the median latency per image,
convert it to FPS, and average equally over all 45 images. These repeated
renders of a single image measure rendering throughput, excluding file I/O.

\paragraph{Model memory and deployment size.}
Model memory is the memory of the uncompressed model used for inference: all FP32 parameters needed for rendering, without optimizer state.  Hexels V1/V2/V3 use 80/116/260 bytes per
primitive. Deployment size includes files needed for inference, including configuration
and appearance parameters, in decimal MB. NVS comparison tables report the original uncompressed FP32 inference files; compression for deployment is evaluated separately in Appendix~\ref{app:nvs_compression}.

\paragraph{Training resources.}
Peak VRAM is the sampled process GPU memory during training, in GiB;
dataset values average the per-scene peaks.
Training time includes initialization, optimization and checkpoint saving,
but excludes separate evaluation, compression, and video generation. For VRAM calculation, we run multiple runs and sample GPU memory at approximately 2 ms or 5 ms intervals using NVIDIA’s NVML API. All mean text VRAM calculations use this, while some in the appendix may use just 5ms sampling. 

\subsection{Hexel Surface-Cell Representation}
\label{app:surface_cell_export}
A Hexel can be used to specify a hexagonal cell through $d(u,v)=0$ in Eq.~\eqref{eq:hex-boundary}. The six vertices can be mapped analytically into the world space through Eq.~\eqref{eq:plane-to-world}. Importantly, this is differentiable and can be done in constant time per Hexel or any other polygonal primitive with low overdraw. However, this might not be possible for the more prevalent Gaussian-based surfels like 2DGS. A Gaussian surfel~\cite{huang2024_2dgs} has a smooth local representation $G(u,v)=\exp[-(u^2+v^2)/2]$, with no explicitly finite boundary; therefore, getting such a differentiable, constant-time surface cell representation is non-trivial.

\subsection{Spatial Appearance Ablations}

Tables~\ref{tab:nvs_hexels_appendix_mip_outdoor} and~\ref{tab:nvs_hexels_appendix_dtu} compare all six Hexels configurations (V1/V2/V3 with fixed or densified populations) on Mip-NeRF360, DTU, and Tanks and Temples.
V3 with densification achieves the highest PSNR across all three datasets. Fixed population variants require less memory and render faster. Densification improves V3's PSNR on Mip-NeRF360 and Tanks and Temples; on DTU, fixed V3 achieves slightly higher PSNR (29.95 versus 29.76\,dB), whereas densified V3 improves LPIPS (0.268 versus 0.293).

\section{Surface Reconstruction Ablations}
\label{app:surface_ablation}

In this section, we perform ablations on the Hexels' model underlying our DTU surface reconstruction results in Table~\ref{tab:dtu_memory_main}: a fixed-population V3 model initialized from multi-view stereo (MVS) points and trained with the normal-consistency and depth-distortion losses, similar to 2DGS. Starting with this configuration, we change one group of choices at a time.

\begin{itemize}
    \item \textbf{Surface losses} (Table~\ref{tab:surface_ablation}(a)): we drop normal consistency, depth distortion, or both.
    \item \textbf{Initialization} (Table~\ref{tab:surface_ablation}(a)): we start from sparse structure-from-motion (SfM) points instead of MVS points.
    \item \textbf{Appearance and population} (Table~\ref{tab:surface_ablation}(b)): keeping both losses, we train V1, V2, and V3 from SfM or MVS points, each with a fixed population or with the densification and pruning of Appendix~\ref{app:adaptive-population}.
\end{itemize}

Every experiment uses the same 15 DTU scans, all registered training views, 30,000 training iterations, with the same TSDF extraction and official DTU evaluation. Chamfer distance (CD) is the mean of accuracy and completeness, measured in millimeters and averaged over scans. When enabled, the normal-consistency and depth-distortion weights are 0.05 and 1000, respectively.

\paragraph{Surface losses.} Table~\ref{tab:surface_ablation}(a) reports 1.575\,mm CD without either loss, 1.246\,mm with depth distortion alone, 0.986\,mm with normal consistency alone and 0.978\,mm with both losses. Normal consistency provides the larger individual improvement, with a smaller complementary benefit from depth distortion.

\paragraph{Initialization.} For fixed V3 with both losses enabled, MVS initialization reduces mean CD from 1.146\,mm for structure-from-motion (SfM) initialization to 0.978\,mm, a 14.6\% reduction at unchanged primitive count and model memory.

\paragraph{Appearance and population.} In Table~\ref{tab:surface_ablation}(b), V3 has the lowest CD with either setting of initialization or population, and MVS initialization beats SfM for every variant. Densification improves CD in 5 out of 6 appearance mode and initialization pairs; the only exception is V3 with SfM (1.146\,mm fixed versus 1.193\,mm densified). Densified V3 from MVS has the lowest CD overall (0.922\,mm), but it needs 33.17\,MB model memory vs 8.83\,MB in the fixed variant, so the memory-matched comparison in Table~\ref{tab:dtu_memory_main} uses fixed V3. Figures~\ref{fig:surface_ablation} and~\ref{fig:surface_matrix_ablation} show the runs of Table~\ref{tab:surface_ablation}(a) and (b) on DTU scan24.

\begin{table}[t]
\caption{\textbf{DTU surface ablations (15 scans).} (a) Surface losses and initialization for fixed V3; the last row is the model of Table~\ref{tab:dtu_memory_main}. (b) Ablation over appearance model, initialization, and population, with normal and depth distortion losses enabled.}
\label{tab:surface_ablation}
\begingroup\small\flushrules
\setlength{\tabcolsep}{3pt}\renewcommand{\arraystretch}{1.3}
\fittabular[3pt]{5}{\begin{tabular}{@{}lllrr@{}}
\toprule
\multicolumn{5}{@{}l}{(a) Fixed-V3 surface-loss controls}\\
Initialization & Normal & Distortion & CD (mm)$\downarrow$ & Model (MB)$\downarrow$\\
\midrule
MVS & No & No & 1.575 & 8.83 \\
MVS & No & Yes & 1.246 & 8.83 \\
MVS & Yes & No & \cellcolor{tabsecond}0.986 & 8.83 \\
SfM & Yes & Yes & \cellcolor{tabthird}1.146 & 8.83 \\
MVS & Yes & Yes & \cellcolor{tabfirst}0.978 & 8.83 \\
\midrule
\multicolumn{5}{@{}l}{(b) Appearance, initialization and population}\\
Variant & Init. & Population & CD (mm)$\downarrow$ & Model (MB)$\downarrow$\\
\midrule
V1 & SfM & Fixed & 1.717 & \cellcolor{tabfirst}{2.72} \\
V1 & SfM & Densified & 1.406 & 10.20 \\
V1 & MVS & Fixed & 1.388 & \cellcolor{tabfirst}{2.72} \\
V1 & MVS & Densified & 1.204 & 11.17 \\
V2 & SfM & Fixed & 1.693 & \cellcolor{tabsecond}{3.94} \\
V2 & SfM & Densified & 1.528 & 13.57 \\
V2 & MVS & Fixed & 1.387 & \cellcolor{tabsecond}{3.94} \\
V2 & MVS & Densified & 1.253 & 14.62 \\
V3 & SfM & Fixed & \cellcolor{tabthird}{1.146} & \cellcolor{tabthird}{8.83} \\
V3 & SfM & Densified & 1.193 & 31.00 \\
V3 & MVS & Fixed & \cellcolor{tabsecond}{0.978} & \cellcolor{tabthird}{8.83} \\
V3 & MVS & Densified & \cellcolor{tabfirst}{0.922} & 33.17 \\
\bottomrule\end{tabular}}
\endgroup\end{table}

\begin{figure*}[t]
\centering
\includegraphics[width=\textwidth]{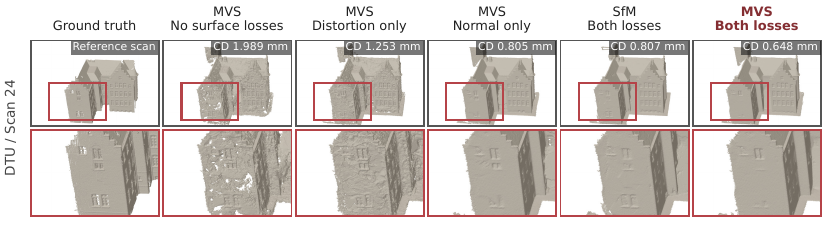}
\caption{\textbf{Surface loss and initialization ablation on DTU scan24.} Ground truth and the five fixed-V3 runs of Table~\ref{tab:surface_ablation}(a). Normal consistency reduces facade roughness and holes; MVS with both losses gives the lowest error on this scan.}
\label{fig:surface_ablation}
\end{figure*}

\begin{figure*}[t]
\centering
\includegraphics[width=\textwidth]{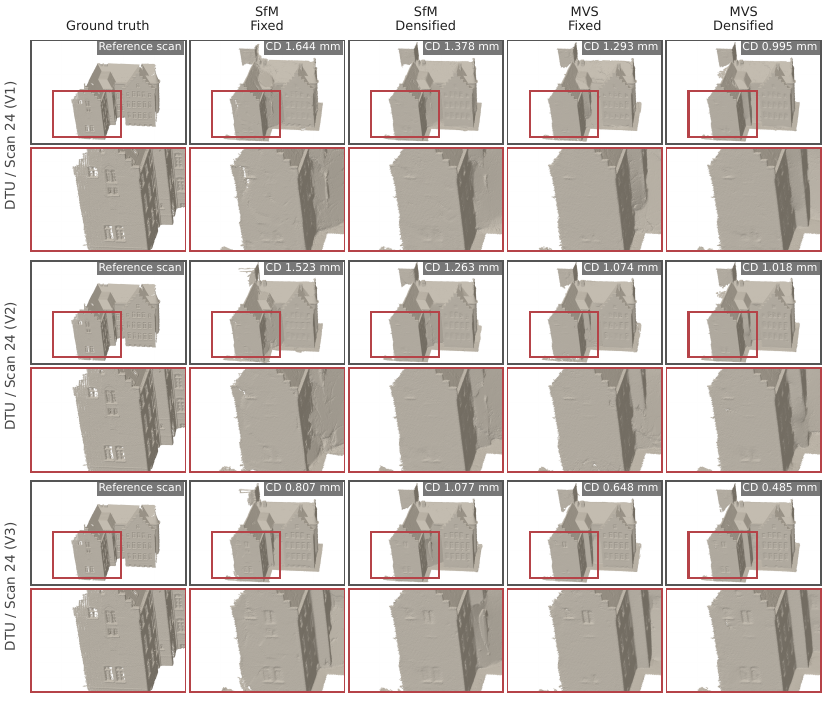}
\caption{\textbf{Appearance, initialization and primitive number ablation in surface reconstruction.} Rows show Hexels V1/V2/V3; columns compare SfM/MVS with fixed/densified promotives on DTU scan24, with both surface losses. V3/MVS retains facade openings in the close-up, and its densified variant has the lowest CD on this scan.}
\label{fig:surface_matrix_ablation}
\end{figure*}

\paragraph{Unconstrained surface reconstruction.}
\label{sec:dtu_unconstrained}
We additionally evaluate the methods without a matched memory budget. Native 2DGS achieves lower DTU CD than Hexels
(0.794\,mm), but Hexels remains competitive (0.978\,mm) while requiring substantially less model memory and training resources (1.93\,GiB VRAM for 2DGS vs. 0.61\,GiB for Hexels). On Tanks and Temples, Hexels has higher mean F1 (34.38\% versus 31.72\%). Figure~\ref{fig:surface_qualitative_extra} provides additional qualitative results. The zoomed crops show
holes in Triangle Splatting's skull and pedestal surfaces and spurious sheets above the `Caterpillar' cab in the memory-matched 2DGS results. Hexels reduces these artifacts in the displayed regions.

\begin{figure*}[t]
\centering
\includegraphics[width=\textwidth]{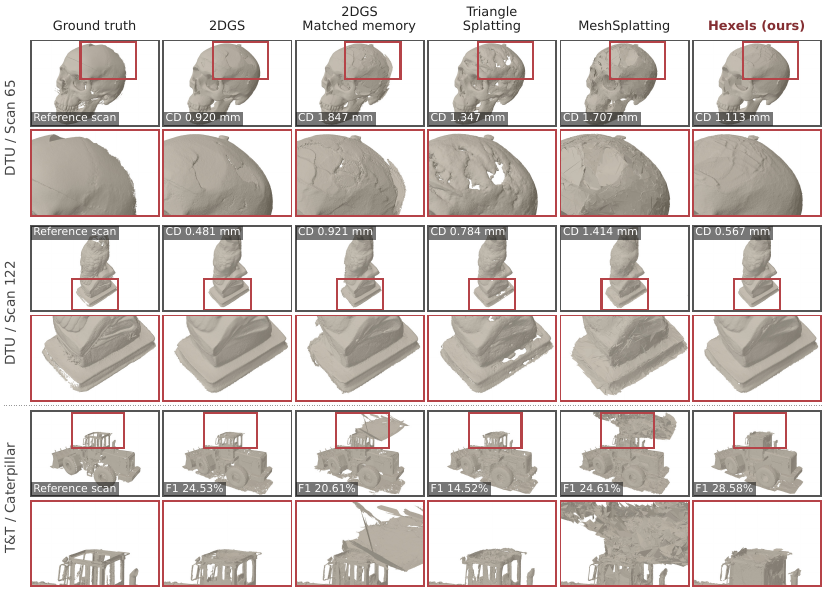}
\caption{\textbf{Additional surface reconstruction close-ups.} The selected views reveal holes in Triangle Splatting's skull and pedestal surfaces and spurious sheets above the Caterpillar cab in the memory-matched 2DGS reconstructions. Hexels reduces these artifacts.}
\label{fig:surface_qualitative_extra}
\end{figure*}

\section{Why Hexagons and Spatial Color?}
\label{app:local-approximation}

We want each Hexel to represent a large surface region without losing the
appearance variation inside it. This requires two choices: a shape that
can cover surfaces efficiently, and a color model that can describe more
than a constant color. The honeycomb theorem motivates the first choice:
regular hexagons minimize boundary length among equal-area partitions of
the plane~\cite{hales2001honeycomb}. The spatial color field addresses
the second.

We examine these choices on one planar patch at a fixed viewing direction.
First, we show that an affine color field represents the local color slope,
with the remaining error controlled by how that slope changes across the
patch. This determines a region where one color field stays within a chosen
error tolerance. We then show that, within the same region, the largest
affine hexagon covers twice the area of the largest affine triangle and
about 30\% more than the largest affine square.

\subsection{Linear Color}

We fix a viewing direction and use  orthonormal coordinates
$\bx\in\mathbb{R}^2$ on a localized plane. Let $\bp$ be the patch center and
$\bh=\bx-\bp$. In axes aligned with the primitive,
$\bh=(s_u u,s_v v)^\top$, where $(u,v)$ are its normalized local
coordinates.

Let $f$ be one target color (channel). An affine color field represents a
constant and a linear change in color. The remaining approximation error
depends on the second derivatives of $f$, which describe how its slope
changes across the patch. We are assuming a standard local approximation
argument underlying piecewise-polynomial representations
~\cite{cohen2011adaptive}.

\begin{proposition}[Local error of an affine color field]
\label{prop:local-color}
Let $f$ be twice continuously differentiable on an open convex
neighborhood $U$ containing $\bp$. Suppose a fixed symmetric
positive-definite matrix $\bQ$ satisfies
\begin{equation}
 |\bz^\top D^2f(\bx)\bz|
 \leq \bz^\top\bQ\bz
 \qquad
 \text{for all }\bx\in U,\ \bz\in\mathbb{R}^2.
 \label{eq:theory-hessian-bound}
\end{equation}
Then, for every $\bp+\bh\in U$, the affine field
$\ell_{\bp}(\bp+\bh)=f(\bp)+\nabla f(\bp)^\top\bh$
satisfies
\begin{equation}
 |f(\bp+\bh)-\ell_{\bp}(\bp+\bh)|
 \leq \tfrac12\bh^\top\bQ\bh.
 \label{eq:theory-local-error}
\end{equation}
In particular, if $\|D^2f(\bx)\|_{\mathrm{op}}\leq M$ throughout $U$,
the error is at most $Mr^2/2$ wherever $\|\bh\|\leq r$.
\end{proposition}

\begin{proof}
Define $g(t)=f(\bp+t\bh)$ for $t\in[0,1]$. Applying the fundamental
theorem of calculus twice gives
\begin{equation}
 \begin{aligned}
 f(\bp+\bh)-\ell_{\bp}(\bp+\bh)
 &=\int_0^1(1-t)g''(t)\,dt,\\
 g''(t)&=\bh^\top D^2f(\bp+t\bh)\bh.
 \end{aligned}
\end{equation}
The segment lies in $U$, so
Eq.~\eqref{eq:theory-hessian-bound} gives
\begin{equation}
 \begin{aligned}
 |f(\bp+\bh)-\ell_{\bp}(\bp+\bh)|
 &\leq
 \int_0^1(1-t)
 \left|\bh^\top D^2f(\bp+t\bh)\bh\right|\,dt\\
 &\leq
 \bh^\top\bQ\bh\int_0^1(1-t)\,dt
 =\tfrac12\bh^\top\bQ\bh.
 \end{aligned}
\end{equation}
The operator-norm statement follows by replacing the quadratic-form
bound with $M\|\bh\|^2$.
\end{proof}

The affine field therefore accounts for the local slope, leaving an
error controlled by changes in that slope. In comparison, using only
the constant value $f(\bp)$ gives an error bound of $Lr$ when
$\|\nabla f\|\leq L$. A nonconstant affine color ramp is represented
exactly by the affine field, but not by one constant color.
This compares individual color fields; blending constant-color
primitives can also produce spatial variation.

The Hexel spatial field contains this Taylor approximation. For one
color channel, in primitive-aligned physical coordinates, the required
coefficients are
\begin{equation}
 c_0=f(\bp),\qquad
 g_u=s_u\,\partial_1 f(\bp),\qquad
 g_v=s_v\,\partial_2 f(\bp).
\end{equation}
Substituting these into $c_0+u g_u+v g_v$ gives
$\ell_{\bp}(\bp+\bh)$. Normalized local coordinates therefore change
the coefficients, not the class of affine functions that can be
represented.

For RGB, apply the same argument to each channel using a common
matrix $\bQ$. Let $\mathbf f$ be the target RGB field and
$\boldsymbol{\ell}_{\bp}$ its channel-wise Taylor approximation.
For a learned field $\mathbf C$, define
$\eta(\bx)=\|\mathbf C(\bx)-\boldsymbol{\ell}_{\bp}(\bx)\|_\infty$.
Then
\begin{equation}
 \|\mathbf C(\bp+\bh)-\mathbf f(\bp+\bh)\|_\infty
 \leq
 \eta(\bp+\bh)+\tfrac12\bh^\top\bQ\bh.
 \label{eq:theory-fit-error}
\end{equation}
Here $\|\cdot\|_\infty$ is the largest absolute RGB-channel error.
The term $\eta$ accounts for the difference between the learned field
and the Taylor field. The result establishes the available approximation
capacity, rather than assuming that optimization reaches the Taylor
coefficients.

\subsection{Patch Shape}

For a tolerance $\varepsilon>0$, define
\begin{equation}
 E=\{\bp+\bh:\bh^\top\bQ\bh\leq2\varepsilon\}.
 \label{eq:theory-error-ellipse}
\end{equation}
Provided $E\subset U$, the Taylor color error is at most
$\varepsilon$ throughout this ellipse. Its longer directions allow
larger displacements while satisfying the same error bound.
The matrix $\bQ$ bounds changes in color slope; it is not a surface
curvature matrix and is not estimated by our training procedure.

The regular triangle, square, and hexagon each tile the plane.
We now compare their affine images inside the same error ellipse:
how much area can each patch cover while satisfying the same
Taylor color-error bound?

\begin{proposition}[Patch area within one error ellipse]
\label{prop:hex-area}
Let $H_m$ be a filled regular $m$-gon centered at the origin, with
$m\geq3$ and unit circumradius. Among all translated, nonsingular
affine images of $H_m$ contained in $E$, the maximum area is
\begin{equation}
 A_m^{\max}
 =
 \frac{m\sin(2\pi/m)\,\varepsilon}
      {\sqrt{\det\bQ}}.
 \label{eq:theory-polygon-area}
\end{equation}
A maximizer is
$\bp+\sqrt{2\varepsilon}\,\bQ^{-1/2}\bR H_m$,
where $\bR$ is any orthogonal matrix. Consequently,
\begin{equation}
 \frac{A_6^{\max}}{A_4^{\max}}
 =\frac{3\sqrt3}{4}\approx1.299,
 \qquad
 \frac{A_6^{\max}}{A_3^{\max}}=2.
\end{equation}
\end{proposition}

\begin{proof}
The coordinate change
$\by=\bQ^{1/2}(\bx-\bp)/\sqrt{2\varepsilon}$
maps $E$ to the unit disk.

Let $\bv_j$, $j=0,\ldots,m-1$, be the vertices of $H_m$. Their equal
angular spacing gives
\begin{equation}
 \frac1m\sum_j\bv_j=0,\qquad
 \frac1m\sum_j\bv_j\bv_j^\top=\tfrac12\bI_2.
\end{equation}
These identities follow from the sums of the sine and cosine terms
at the vertex angles and at twice those angles.

Consider an affine image $\bt+\bB H_m$ contained in the unit disk.
Every vertex has squared distance at most one, so
\begin{equation}
 \|\bt\|^2+\tfrac12\operatorname{tr}(\bB^\top\bB)
 =
 \frac1m\sum_j\|\bt+\bB\bv_j\|^2
 \leq1.
 \label{eq:theory-vertex-average}
\end{equation}
Let $\sigma_1,\sigma_2$ be the singular values of $\bB$. Then
\begin{equation}
 |\det\bB|
 =\sigma_1\sigma_2
 \leq\tfrac12(\sigma_1^2+\sigma_2^2)
 \leq1-\|\bt\|^2
 \leq1.
 \label{eq:theory-det-bound}
\end{equation}
Thus, no translated affine image has greater area than the centered
regular polygon. Equality is attained with $\bt=0$ and $\bB$
orthogonal, for which the polygon lies inside the disk.

The area of $H_m$ is $m\sin(2\pi/m)/2$. Returning to the original
coordinates multiplies area by
$2\varepsilon/\sqrt{\det\bQ}$, yielding
Eq.~\eqref{eq:theory-polygon-area}.
\end{proof}

\paragraph{Connection to the Hexel parameterization.}
The maximizing hexagon can be represented using the rotation and two
tangent scales already present in Hexels. Write
$\bQ=\bU\operatorname{diag}(\lambda_1,\lambda_2)\bU^\top$,
where $\lambda_1,\lambda_2>0$ and the eigenvectors in $\bU$ are
chosen to form a rotation. Align the tangent axes with these
eigenvectors. Since our canonical hexagon has inradius one and
circumradius $2/\sqrt3$, choose
\begin{equation}
 s_u=\sqrt{\frac{3\varepsilon}{2\lambda_1}},
 \qquad
 s_v=\sqrt{\frac{3\varepsilon}{2\lambda_2}}.
 \label{eq:theory-optimal-scales}
\end{equation}
This places the six vertices on the error ellipse and attains
$A_6^{\max}$.

Thus, among affine images of these three regular tiling cells, the
hexagon covers the largest area inside the same sufficient color-error
ellipse. Combined with the tiling motivation, this supports using a
hexagonal patch with an affine color field: one element can represent
a larger region than the corresponding triangular or square element
under this local approximation bound.

\subsection{3D Alpha Compositing}

The preceding results concern the color represented within individual
patches. In 3D rendering, the final pixel also depends on alpha
compositing and the remaining background transmittance.

\begin{proposition}[Error under alpha compositing]
\label{prop:blend-error}
At one pixel, let $0\leq\alpha_i\leq1$ be a finite ordered list of
retained contributions, with colors $C_i$. Define
\begin{equation}
 w_i=\alpha_i\prod_{j<i}(1-\alpha_j),\qquad
 T=\prod_i(1-\alpha_i),\qquad
 I=\sum_iw_i C_i+Tb,
\end{equation}
where $b$ is the background color. For any target color $f$ and
any norm,
\begin{equation}
 \|I-f\|
 \leq
 \sum_iw_i\|C_i-f\|+T\|b-f\|.
 \label{eq:theory-composite-error}
\end{equation}
\end{proposition}

\begin{proof}
Set $T_0=1$ and $T_i=\prod_{j\leq i}(1-\alpha_j)$.
Then $w_i=T_{i-1}-T_i\geq0$, giving
$\sum_iw_i+T=1$. Hence,
\[
 I-f=\sum_iw_i(C_i-f)+T(b-f).
\]
The result follows from the triangle inequality.
\end{proof}

Thus, if every contributing color differs from the target by at most
$e$, the image error is at most
$(1-T)e+T\|b-f\|$. This result is independent of the primitive shape.
It also shows why local color accuracy alone does not ensure an
accurate image: the contributions must approximate the same pixel
target, and the remaining background contribution must be included.

\paragraph{Soft support and visibility.}
The sigmoid support extends beyond the nominal hexagon, where the
local Taylor bound may no longer apply. Let $\mathcal G$ contain
contributions with valid bounds $\|C_i-f\|_\infty\leq e_i$ relative
to the target color at this pixel. If the remaining contributions
have error at most $D$ and total weight
$\zeta=\sum_{i\notin\mathcal G}w_i$, then
\begin{equation}
 \|I-f\|_\infty
 \leq
 \sum_{i\in\mathcal G}w_i e_i
 +\zeta D+T\|b-f\|_\infty.
 \label{eq:theory-tail-error}
\end{equation}
Equation~\eqref{eq:theory-fit-error} can supply $e_i$ when the local
target at the intersection is the desired pixel color. An accurate
approximation to a different, occluded surface does not provide that
bound. Since individual Hexel colors are unclamped, $D=1$ cannot
be assumed.

These statements apply to the contributions actually accumulated by
the renderer, with $T$ their remaining transmittance.
\section{Extended Related Work}
\label{app:related}

\paragraph{Novel view synthesis and surface reconstruction.}
NeRF~\cite{mildenhall2020nerf} was introduced for novel view synthesis and
represents a scene as a neural field of density and view-dependent radiance.
This representation enabled high-quality novel-view rendering. However, the
repeated network evaluations required for volumetric rendering made the
original NeRF expensive to train and slow to render. A large body of
follow-up work introduced explicit feature grids, hash encodings,
precomputed representations, and baked representations to reduce the
sampling, rendering, and computational demands of the original neural-field
models~\cite{muller2022instant,yu2022plenoxels,reiser2023merf}.

3DGS~\cite{kerbl3Dgaussians} introduced an alternative for substantially
more efficient rendering using explicit anisotropic Gaussian volume
primitives, enabling real-time rasterization. However, because of their
volumetric nature, 3D Gaussians are not naturally surface aligned and can
produce inaccurate geometry. To remedy this, works such as
SuGaR~\cite{guedon2023sugar} regularize Gaussians toward surfaces. An
important solution was introduced by 2DGS~\cite{huang2024_2dgs}, which
instead represents the scene using oriented planar Gaussian surfels and
perspective-correct ray--plane intersections, retaining fast rasterization
while substantially improving the geometry.

Despite the large improvements made by 3DGS and 2DGS in training and
rendering time relative to NeRF, these improvements are still not enough for
truly resource-constrained settings. Efficiency requires several things at
the same time: learned model storage and training VRAM should be as low as
possible, primitive count should remain small relative to the cost of each
primitive, and the representation should still retain high reconstruction
quality and fast rendering. 3DGS and 2DGS can both require large primitive
populations and substantial memory. Despite this, we find that 2DGS remains
the strongest prior native surface representation in terms of the overall
performance--efficiency balance when these factors are considered together,
even after a number of subsequent works have attempted to improve different
parts of the problem.

\paragraph{Primitive-based differentiable rendering.}
Different follow-up works have tried to make the primitive more expressive,
reduce the number of primitives, increase rendering speed, reduce model
storage, or address limitations of 2DGS and earlier Gaussian
representations, such as the infinite support of Gaussian kernels.
GES~\cite{hamdi2024ges}, for example, replaces the Gaussian with a more
flexible generalized exponential kernel, while Deformable Beta
Splatting~\cite{liu2025beta} introduces bounded and deformable kernels.
3D Convex Splatting~\cite{held2025convex} instead uses smooth 3D convex
primitives, allowing more flexible volumetric geometry and sharper
boundaries, but remains a volumetric rather than surface-aligned
representation.

Triangle Splatting~\cite{held2026triangle}, on the other hand, optimizes an
independent triangle soup so that high rendering speeds can be achieved while
retaining a surface-aligned primitive. However, high-quality scenes can still
require very large triangle populations, increasing both model storage and
the computation and memory required during optimization. The independent
triangle soup also does not provide surface connectivity.
MeshSplatting~\cite{held2025mesh} follows this work by introducing explicit
connectivity and an opaque mesh representation, trading the independence of
individual splats for a conventional connected surface. Other works increase
geometric complexity in different ways, including Quadratic Gaussian
Splatting~\cite{zhang2024quadratic}, which introduces curved patches, and
MILo~\cite{guedon2025milo}, which brings mesh construction directly into
Gaussian optimization. In our experiments, we focus on 2DGS, Triangle Splatting, and MeshSplatting, which we find provide the strongest overall quality--efficiency trade-offs among prior native surface-aligned representations.

Hexels are also primitive-based renderers, but we try to solve a more
fundamental problem: can we design a more efficient primitive that is able
to exploit the structure found in the world, both in geometry and
appearance? Like 2DGS, we retain independent planar elements, but change the
region represented by one element and the signal represented inside it. We
use an affine hexagonal patch, which gives a large, nearly isotropic local
region, together with a first-order color field that models smooth
appearance variation across a surface. We find that much of the variation
on local surfaces is smooth, making such a representation particularly
efficient.

A first-order color field, by itself, is not unique to Hexels: a vertex-colored
triangle (another tiling primitive), for example, can also represent an affine color field. The domain
of our primitive, however, makes an important difference. Under the same
local approximation-error ellipse, our analysis shows that the largest
admissible affine regular hexagon can cover twice the area of the
corresponding affine triangle and about $1.3\times$ that of the square
(Appendix~\ref{app:local-approximation}). Hexels therefore increase the
amount represented by a simple planar element without requiring complex
geometry, volumetric convexes, or mesh connectivity.

\paragraph{Hybrid primitives and representations.}
The most direct comparisons to Hexels are fundamentally native surface
primitives. However, there has also been a trend toward hybrid
representations that separate geometry from appearance in order to improve
particular aspects of efficiency over 2DGS-like representations.
Gaussian Billboards~\cite{weiss2024billboards}, for example, introduces
spatially varying appearance through per-splat texture interpolation.
BBSplat~\cite{svitov2025bbsplat} learns separate RGB and alpha texture maps
to reduce the number of geometric primitives required. This can substantially
reduce primitive count, but the textures themselves become part of the model
and training cost.
NeST~\cite{zhang2025nest} follows a related idea by separating sparse surface
geometry from a neural appearance representation. Nexels
~\cite{rong2025nexels} similarly combines sparse surfels with a shared neural
appearance field and per-primitive colors. These methods can substantially
reduce geometric primitive count, but this capacity is moved into texture
maps, feature fields, or neural networks, which introduce additional memory,
training state, and computation.
Concurrent work such as Bake It Till You Make
It~\cite{kelkar2026bake}, which is also centered on improving the efficiency
of surface splats, follows this hybrid direction. It uses sparse surface
geometry together with a hash-based appearance representation during
optimization, and later bakes high-frequency appearance into a texture atlas
for rendering. This can reduce the number of geometric surface primitives
and improve rendering speed, but appearance is again moved outside the
primitive and requires a separate representation and baking pipeline.

These ideas of separating geometry and appearance, neural textures, and
texture baking are interesting, but they can in principle be applied on top of
different native surface primitives. They also often come with a different set of trade-offs, be it slower rendering, more memory or non-ideal decoupled optimization. They therefore answer a somewhat
different question from ours. We include these methods in our comparisons,
but count their full learned representation, including textures, feature
grids, and neural decoders, rather than only the number of geometric
elements. These hybrid setups, however, are not the objective of our work. 

\paragraph{Population control and compression.}
Given the multifaceted nature of the efficiency problem, another body of work
has developed around interventions on existing primitives to improve
different aspects of their efficiency. Scaffold-GS~\cite{lu2024scaffold},
for example, reduces 3DGS redundancy through anchors that generate
view-adaptive Gaussian attributes rather than storing every Gaussian
independently. 3DGS-MCMC~\cite{kheradmand2024mcmc} improves primitive
allocation through stochastic updates and relocation. FastGS
~\cite{ren2025fastgs} instead focuses on training time, using
multi-view-consistent densification and pruning to reduce unnecessary
Gaussian growth.

A related line of work compresses already learned representations.
LightGaussian~\cite{fan2024lightgaussian} prunes low-importance Gaussians,
distills appearance into lower-order spherical harmonics, and applies vector
quantization to reduce model size. Compact 3D Gaussian
Representation~\cite{lee2024compact} combines learnable masking with
shared neural appearance, vector quantization, and entropy coding.
Compressed 3DGS~\cite{niedermayr2024compressed} similarly uses
quantization-aware clustering and codebooks to reduce storage. Related ideas
have also been developed for surface-based Gaussian representations.

These methods are useful and complementary, but they start with an already
chosen primitive and add additional methodology around it. They may greatly
improve one aspect of efficiency without necessarily reducing the memory and
computation required elsewhere, particularly during training. Our focus is
therefore the more fundamental question: we can add pruning, quantization,
codebooks, texture atlases, neural decoders, and other post-hoc compression,
but can we also start with a more efficient native surface primitive?
Hexels address this question before these additional techniques are applied;
many of them could still be applied on top of Hexels afterward.

\paragraph{Neural Image Representations.}
A similar line of research has developed independently in 2D, where neural
image representations are used to encode visual media. Earlier approaches
often relied on autoencoders or implicit neural functions. More recently,
GaussianImage~\cite{zhang2024gaussianimage} adapted ideas from Gaussian
splatting to represent images using simple 2D Gaussian primitives,
demonstrating fast fitting and rendering. Image-GS~\cite{imagegs2025}
further improves this representation using content-adaptive initialization
and progressive optimization of anisotropic 2D Gaussians.
SAD~\cite{iinbor2026sad} instead represents the image through a soft
anisotropic partition whose regions depend on neighboring sites rather than
independently parameterized primitives. We compare Hexels primarily with Image-GS because both use independently parameterized image primitives, enabling a direct comparison of quality and efficiency per primitive. SAD instead represents images through a coupled partition whose regions depend on neighboring sites, addressing a different representation regime that we leave for future comparison.

We show that Hexels can also be used directly for 2D image representation,
where they outperform Image-GS in reconstruction quality and training
convergence under our evaluated settings. More importantly, Hexels are
designed as a general surface primitive rather than separate 2D and 3D
representations. In 2D, the primitive is placed directly in the image plane;
in 3D, it is oriented on a surface patch. While the compositor changes where
required by visibility and depth ordering, the basic hexagonal support and
appearance model remain the same. This gives Hexels a general capability
across image representation, novel-view synthesis, and surface
reconstruction, rather than requiring a separate representation for each
setting.

\paragraph{Overall efficiency.}
A common pattern across methods that have tried to make surface-aligned
representations more efficient is that the cost is often moved to another
part of the problem. Faster rasterization may still require a very large
primitive population. Sparse geometry may move capacity into external
texture maps or neural fields, increasing representation complexity,
training VRAM, or final model storage. Compression can make the deployed
scene small while having much less effect on the size and memory required
during training. Richer geometry can reduce primitive count while increasing
the cost of each primitive.

Hexels are designed to alleviate these problems together, without first
resorting to post-hoc retrofitting and before pruning or compression is
applied. We use the fact that natural images, surfaces, and scenes contain
local structure. Large regions are often approximately planar and vary
smoothly in appearance. A Hexel covers a bounded, efficient local region and
models first-order appearance variation inside it, allowing simple regions
to be represented by larger and fewer elements. This allows us to target high
quality, low learned storage, low training VRAM, few primitives, and high
FPS in the same native surface representation, in both 2D and 3D. Pruning,
quantization, neural appearance models, and other post-hoc improvements
remain possible additions rather than requirements of the representation.
\section{Full NVS Results}
\label{app:full_tables}

We report additional model-memory experiments, all six Hexels configurations (V1, V2, V3 fixed and densified), initial compression results, and per-scene results on the 34-scene benchmark. $F$ stands for fixed primitives from SfM initialization whereas $G$ stands for their densified counterparts, unless stated otherwise. 

\subsection{Matched Model Memory}
\label{app:nvs_budgets}

At the V2 memory budget, Hexels exceeds Triangle Splatting by
1.40 dB outdoors and 2.19 dB indoors on Mip-NeRF360
(Table~\ref{tab:nvs_memory_v2_mip_outdoor}).
On DTU, it renders at $1.65\times$ Triangle's mean FPS
(Table~\ref{tab:nvs_memory_v2_dtu}).
At the V3 budget, the PSNR gains over Triangle are 0.54/0.55 dB
on outdoor/indoor Mip-NeRF360 (Table~\ref{tab:nvs_memory_v3_mip_outdoor}) and 2.73/0.83 dB on DTU/T\&T (Table~\ref{tab:nvs_memory_v3_dtu}).
Across these V2/V3 budgets, Hexels has higher PSNR than MeshSplatting
across all 3 datasets, with higher SSIM and lower LPIPS.
Table~\ref{tab:scene_capacities} lists the model-memory budgets for each scene.

\begin{table*}[t]
\caption{\textbf{Matched model memory on Mip-NeRF360.} Fixed-population Hexels V2-F defines mean budgets of 7.33/20.77 MB outdoors/indoors. Hexels V2 improves PSNR by 1.40/2.19 dB relative to Triangle Splatting, respectively.}
\label{tab:nvs_memory_v2_mip_outdoor}
\small
\flushrules
\setlength{\tabcolsep}{4pt}
\renewcommand{\arraystretch}{1.14}
\fittabular{13}{%
\begin{tabular}{@{}lrrrrrr@{\hspace{9pt}}rrrrrr@{}}
\toprule
\multicolumn{1}{@{}l}{\multirow{3}{*}{Method}} & \multicolumn{6}{c}{Mip-NeRF360 outdoor (5)} & \multicolumn{6}{c@{}}{Mip-NeRF360 indoor (4)} \\
\cmidrule(lr){2-7}\cmidrule(l){8-13}
 & Train & VRAM & FPS & PSNR & SSIM & LPIPS & Train & VRAM & FPS & PSNR & SSIM & LPIPS \\
 & min $\downarrow$ & GiB $\downarrow$ &  $\uparrow$ & dB $\uparrow$ &  $\uparrow$ & VGG $\downarrow$ & min $\downarrow$ & GiB $\downarrow$ &  $\uparrow$ & dB $\uparrow$ &  $\uparrow$ & VGG $\downarrow$ \\
\midrule
\multicolumn{13}{@{}c@{}}{\textcolor{black!65}{\textit{Volumetric primitives}}} \\
3DGS & \cellcolor{tabtealfirst!55!white}6.3 & \cellcolor{tabtealfirst!55!white}0.66 & \cellcolor{tabtealfirst!55!white}1112.3 & 16.72 & 0.313 & 0.722 & \cellcolor{tabtealfirst!55!white}11.4 & 2.44 & \cellcolor{tabtealfirst!55!white}376.6 & \cellcolor{tabtealfirst!55!white}27.13 & 0.823 & 0.412 \\
3D Convex & 24.3 & 0.88 & 178.0 & \cellcolor{tabtealfirst!55!white}19.04 & \cellcolor{tabtealfirst!55!white}0.416 & \cellcolor{tabtealfirst!55!white}0.638 & 54.6 & \cellcolor{tabtealfirst!55!white}1.86 & 49.3 & 25.04 & \cellcolor{tabtealfirst!55!white}0.826 & \cellcolor{tabtealfirst!55!white}0.404 \\
\addlinespace[2pt]
\multicolumn{13}{@{}c@{}}{\textcolor{black!65}{\textit{Hybrid primitives}}} \\
BBSplat & \cellcolor{tabbluefirst!93!blue}28.0 & \cellcolor{tabbluefirst!93!blue}1.41 & \cellcolor{tabbluefirst!93!blue}109.0 & 19.24 & 0.374 & 0.688 & \cellcolor{tabbluefirst!93!blue}52.2 & 8.00 & \cellcolor{tabbluefirst!93!blue}44.4 & 25.19 & 0.779 & 0.443 \\
Nexels & 61.3 & 4.71 & 59.4 & \cellcolor{tabbluefirst!93!blue}20.72 & \cellcolor{tabbluefirst!93!blue}0.460 & \cellcolor{tabbluefirst!93!blue}0.534 & 111.3 & \cellcolor{tabbluefirst!93!blue}7.97 & 33.6 & \cellcolor{tabbluefirst!93!blue}25.74 & \cellcolor{tabbluefirst!93!blue}0.807 & \cellcolor{tabbluefirst!93!blue}0.379 \\
\addlinespace[2pt]
\multicolumn{13}{@{}c@{}}{\textcolor{black!65}{\textit{Surface primitives}}} \\
2DGS & \cellcolor{tabsecond}16.5 & \cellcolor{tabthird}1.92 & \cellcolor{tabthird}245.1 & 16.11 & 0.328 & 0.736 & \cellcolor{tabsecond}25.1 & 5.93 & \cellcolor{tabthird}115.4 & \cellcolor{tabthird}22.13 & \cellcolor{tabthird}0.762 & \cellcolor{tabthird}0.489 \\
Triangle Spl. & \cellcolor{tabfirst}12.9 & \cellcolor{tabsecond}0.96 & \cellcolor{tabsecond}358.9 & \cellcolor{tabsecond}20.71 & \cellcolor{tabsecond}0.448 & \cellcolor{tabsecond}0.592 & \cellcolor{tabfirst}21.0 & \cellcolor{tabsecond}1.45 & \cellcolor{tabfirst}233.8 & \cellcolor{tabsecond}26.55 & \cellcolor{tabsecond}0.843 & \cellcolor{tabsecond}0.383 \\
MeshSplatting & 27.0 & 3.70 & 72.9 & \cellcolor{tabthird}18.81 & \cellcolor{tabthird}0.349 & \cellcolor{tabthird}0.674 & 40.9 & \cellcolor{tabthird}5.52 & 45.1 & 20.83 & 0.665 & 0.558 \\
\textbf{Hexels V2-F} & \cellcolor{tabthird}17.8 & \cellcolor{tabfirst}0.63 & \cellcolor{tabfirst}491.9 & \cellcolor{tabfirst}22.11 & \cellcolor{tabfirst}0.562 & \cellcolor{tabfirst}0.497 & \cellcolor{tabthird}32.8 & \cellcolor{tabfirst}0.87 & \cellcolor{tabsecond}228.9 & \cellcolor{tabfirst}28.74 & \cellcolor{tabfirst}0.897 & \cellcolor{tabfirst}0.290 \\
\bottomrule
\end{tabular}}
\end{table*}

\begin{table*}[t]
\caption{\textbf{Matched model memory on DTU and Tanks and Temples.} Fixed-population Hexels V2-F defines mean budgets of 3.94/21.12 MB. Hexels V2 improves PSNR by 2.92/1.44 dB relative to Triangle Splatting, respectively.}
\label{tab:nvs_memory_v2_dtu}
\small
\flushrules
\setlength{\tabcolsep}{4pt}
\renewcommand{\arraystretch}{1.14}
\fittabular{13}{%
\begin{tabular}{@{}lrrrrrr@{\hspace{9pt}}rrrrrr@{}}
\toprule
\multicolumn{1}{@{}l}{\multirow{3}{*}{Method}} & \multicolumn{6}{c}{DTU (15)} & \multicolumn{6}{c@{}}{Tanks and Temples (10)} \\
\cmidrule(lr){2-7}\cmidrule(l){8-13}
 & Train & VRAM & FPS & PSNR & SSIM & LPIPS & Train & VRAM & FPS & PSNR & SSIM & LPIPS \\
 & min $\downarrow$ & GiB $\downarrow$ &  $\uparrow$ & dB $\uparrow$ &  $\uparrow$ & VGG $\downarrow$ & min $\downarrow$ & GiB $\downarrow$ &  $\uparrow$ & dB $\uparrow$ &  $\uparrow$ & VGG $\downarrow$ \\
\midrule
\multicolumn{13}{@{}c@{}}{\textcolor{black!65}{\textit{Volumetric primitives}}} \\
3DGS & \cellcolor{tabtealfirst!55!white}4.7 & \cellcolor{tabtealfirst!55!white}0.45 & \cellcolor{tabtealfirst!55!white}937.2 & \cellcolor{tabtealfirst!55!white}28.34 & \cellcolor{tabtealfirst!55!white}0.867 & 0.365 & \cellcolor{tabtealfirst!55!white}6.1 & 0.94 & \cellcolor{tabtealfirst!55!white}723.5 & \cellcolor{tabtealfirst!55!white}21.53 & \cellcolor{tabtealfirst!55!white}0.718 & \cellcolor{tabtealfirst!55!white}0.430 \\
3D Convex & 12.8 & 0.61 & 280.9 & 27.49 & 0.863 & \cellcolor{tabtealfirst!55!white}0.356 & 22.4 & \cellcolor{tabtealfirst!55!white}0.91 & 172.3 & 20.72 & 0.706 & 0.445 \\
\addlinespace[2pt]
\multicolumn{13}{@{}c@{}}{\textcolor{black!65}{\textit{Hybrid primitives}}} \\
BBSplat & \cellcolor{tabbluefirst!93!blue}14.3 & \cellcolor{tabbluefirst!93!blue}0.64 & \cellcolor{tabbluefirst!93!blue}183.2 & 27.63 & 0.854 & 0.347 & \cellcolor{tabbluefirst!93!blue}21.7 & \cellcolor{tabbluefirst!93!blue}1.55 & \cellcolor{tabbluefirst!93!blue}123.7 & 19.46 & 0.643 & 0.521 \\
Nexels & 28.8 & 2.26 & 154.6 & \cellcolor{tabbluefirst!93!blue}28.83 & \cellcolor{tabbluefirst!93!blue}0.887 & \cellcolor{tabbluefirst!93!blue}0.277 & 69.9 & 2.98 & 93.6 & \cellcolor{tabbluefirst!93!blue}22.02 & \cellcolor{tabbluefirst!93!blue}0.729 & \cellcolor{tabbluefirst!93!blue}0.383 \\
\addlinespace[2pt]
\multicolumn{13}{@{}c@{}}{\textcolor{black!65}{\textit{Surface primitives}}} \\
2DGS & \cellcolor{tabthird}8.6 & \cellcolor{tabthird}0.70 & \cellcolor{tabthird}426.3 & \cellcolor{tabthird}22.88 & \cellcolor{tabthird}0.786 & \cellcolor{tabthird}0.450 & \cellcolor{tabsecond}11.1 & \cellcolor{tabthird}1.61 & \cellcolor{tabthird}249.1 & \cellcolor{tabthird}19.37 & \cellcolor{tabthird}0.675 & \cellcolor{tabthird}0.481 \\
Triangle Spl. & \cellcolor{tabfirst}7.7 & \cellcolor{tabsecond}0.61 & \cellcolor{tabsecond}513.9 & \cellcolor{tabsecond}26.65 & \cellcolor{tabsecond}0.840 & \cellcolor{tabsecond}0.364 & \cellcolor{tabfirst}10.1 & \cellcolor{tabsecond}0.80 & \cellcolor{tabfirst}459.9 & \cellcolor{tabsecond}22.04 & \cellcolor{tabsecond}0.764 & \cellcolor{tabsecond}0.368 \\
MeshSplatting & 13.3 & 1.85 & 128.7 & 21.44 & 0.685 & 0.499 & 17.2 & 2.17 & 117.5 & 16.66 & 0.581 & 0.553 \\
\textbf{Hexels V2-F} & \cellcolor{tabsecond}8.5 & \cellcolor{tabfirst}0.56 & \cellcolor{tabfirst}845.6 & \cellcolor{tabfirst}29.57 & \cellcolor{tabfirst}0.895 & \cellcolor{tabfirst}0.298 & \cellcolor{tabthird}14.4 & \cellcolor{tabfirst}0.60 & \cellcolor{tabsecond}386.9 & \cellcolor{tabfirst}23.49 & \cellcolor{tabfirst}0.819 & \cellcolor{tabfirst}0.287 \\
\bottomrule
\end{tabular}}
\end{table*}

\begin{table*}[t]
\caption{\textbf{Matched model memory on Mip-NeRF360.} Fixed-population Hexels V3-F defines mean budgets of 16.42/46.56 MB outdoors/indoors. Hexels V3 improves PSNR by 0.54/0.55 dB relative to Triangle Splatting, respectively.}
\label{tab:nvs_memory_v3_mip_outdoor}
\small
\flushrules
\setlength{\tabcolsep}{4pt}
\renewcommand{\arraystretch}{1.14}
\fittabular{13}{%
\begin{tabular}{@{}lrrrrrr@{\hspace{9pt}}rrrrrr@{}}
\toprule
\multicolumn{1}{@{}l}{\multirow{3}{*}{Method}} & \multicolumn{6}{c}{Mip-NeRF360 outdoor (5)} & \multicolumn{6}{c@{}}{Mip-NeRF360 indoor (4)} \\
\cmidrule(lr){2-7}\cmidrule(l){8-13}
 & Train & VRAM & FPS & PSNR & SSIM & LPIPS & Train & VRAM & FPS & PSNR & SSIM & LPIPS \\
 & min $\downarrow$ & GiB $\downarrow$ &  $\uparrow$ & dB $\uparrow$ &  $\uparrow$ & VGG $\downarrow$ & min $\downarrow$ & GiB $\downarrow$ &  $\uparrow$ & dB $\uparrow$ &  $\uparrow$ & VGG $\downarrow$ \\
\midrule
\multicolumn{13}{@{}c@{}}{\textcolor{black!65}{\textit{Volumetric primitives}}} \\
3DGS & \cellcolor{tabtealfirst!55!white}7.1 & \cellcolor{tabtealfirst!55!white}0.88 & \cellcolor{tabtealfirst!55!white}830.8 & 18.68 & 0.383 & 0.675 & \cellcolor{tabtealfirst!55!white}13.5 & \cellcolor{tabtealfirst!55!white}2.26 & \cellcolor{tabtealfirst!55!white}318.8 & \cellcolor{tabtealfirst!55!white}29.51 & \cellcolor{tabtealfirst!55!white}0.886 & \cellcolor{tabtealfirst!55!white}0.326 \\
3D Convex & 26.7 & 1.07 & 116.3 & \cellcolor{tabtealfirst!55!white}20.31 & \cellcolor{tabtealfirst!55!white}0.470 & \cellcolor{tabtealfirst!55!white}0.581 & 70.7 & 2.72 & 35.1 & 28.29 & 0.877 & 0.329 \\
\addlinespace[2pt]
\multicolumn{13}{@{}c@{}}{\textcolor{black!65}{\textit{Hybrid primitives}}} \\
BBSplat & \cellcolor{tabbluefirst!93!blue}29.8 & \cellcolor{tabbluefirst!93!blue}2.05 & \cellcolor{tabbluefirst!93!blue}92.4 & 19.92 & 0.405 & 0.648 & \cellcolor{tabbluefirst!93!blue}58.9 & 12.99 & \cellcolor{tabbluefirst!93!blue}35.0 & 27.03 & 0.834 & 0.367 \\
Nexels & 63.5 & 5.11 & 54.7 & \cellcolor{tabbluefirst!93!blue}21.88 & \cellcolor{tabbluefirst!93!blue}0.542 & \cellcolor{tabbluefirst!93!blue}0.454 & 113.1 & \cellcolor{tabbluefirst!93!blue}8.23 & 31.6 & \cellcolor{tabbluefirst!93!blue}27.84 & \cellcolor{tabbluefirst!93!blue}0.865 & \cellcolor{tabbluefirst!93!blue}0.307 \\
\addlinespace[2pt]
\multicolumn{13}{@{}c@{}}{\textcolor{black!65}{\textit{Surface primitives}}} \\
2DGS & \cellcolor{tabsecond}16.9 & \cellcolor{tabthird}2.23 & \cellcolor{tabthird}220.9 & 16.69 & 0.354 & 0.718 & \cellcolor{tabsecond}29.1 & 8.28 & \cellcolor{tabthird}83.5 & \cellcolor{tabthird}28.39 & \cellcolor{tabthird}0.864 & \cellcolor{tabthird}0.358 \\
Triangle Spl. & \cellcolor{tabfirst}13.8 & \cellcolor{tabsecond}1.03 & \cellcolor{tabsecond}337.2 & \cellcolor{tabsecond}21.73 & \cellcolor{tabsecond}0.520 & \cellcolor{tabsecond}0.518 & \cellcolor{tabfirst}24.8 & \cellcolor{tabsecond}1.74 & \cellcolor{tabsecond}195.8 & \cellcolor{tabsecond}28.78 & \cellcolor{tabsecond}0.886 & \cellcolor{tabsecond}0.312 \\
MeshSplatting & 26.3 & 3.76 & 71.7 & \cellcolor{tabthird}19.27 & \cellcolor{tabthird}0.378 & \cellcolor{tabthird}0.648 & 41.9 & \cellcolor{tabthird}5.54 & 42.5 & 22.16 & 0.706 & 0.528 \\
\textbf{Hexels V3-F} & \cellcolor{tabthird}17.6 & \cellcolor{tabfirst}0.66 & \cellcolor{tabfirst}484.6 & \cellcolor{tabfirst}22.27 & \cellcolor{tabfirst}0.566 & \cellcolor{tabfirst}0.496 & \cellcolor{tabthird}32.7 & \cellcolor{tabfirst}0.99 & \cellcolor{tabfirst}219.8 & \cellcolor{tabfirst}29.33 & \cellcolor{tabfirst}0.904 & \cellcolor{tabfirst}0.283 \\
\bottomrule
\end{tabular}}
\end{table*}

\begin{table*}[t]
\caption{\textbf{Matched model memory on DTU and Tanks and Temples.} Fixed-population Hexels V3-F defines mean budgets of 8.83/47.34 MB. Hexels V3 improves PSNR by 2.73/0.83 dB relative to Triangle Splatting, respectively.}
\label{tab:nvs_memory_v3_dtu}
\small
\flushrules
\setlength{\tabcolsep}{4pt}
\renewcommand{\arraystretch}{1.14}
\fittabular{13}{%
\begin{tabular}{@{}lrrrrrr@{\hspace{9pt}}rrrrrr@{}}
\toprule
\multicolumn{1}{@{}l}{\multirow{3}{*}{Method}} & \multicolumn{6}{c}{DTU (15)} & \multicolumn{6}{c@{}}{Tanks and Temples (10)} \\
\cmidrule(lr){2-7}\cmidrule(l){8-13}
 & Train & VRAM & FPS & PSNR & SSIM & LPIPS & Train & VRAM & FPS & PSNR & SSIM & LPIPS \\
 & min $\downarrow$ & GiB $\downarrow$ &  $\uparrow$ & dB $\uparrow$ &  $\uparrow$ & VGG $\downarrow$ & min $\downarrow$ & GiB $\downarrow$ &  $\uparrow$ & dB $\uparrow$ &  $\uparrow$ & VGG $\downarrow$ \\
\midrule
\multicolumn{13}{@{}c@{}}{\textcolor{black!65}{\textit{Volumetric primitives}}} \\
3DGS & \cellcolor{tabtealfirst!55!white}5.1 & \cellcolor{tabtealfirst!55!white}0.57 & \cellcolor{tabtealfirst!55!white}823.3 & \cellcolor{tabtealfirst!55!white}29.45 & \cellcolor{tabtealfirst!55!white}0.895 & 0.320 & \cellcolor{tabtealfirst!55!white}7.8 & 1.31 & \cellcolor{tabtealfirst!55!white}561.6 & \cellcolor{tabtealfirst!55!white}23.07 & \cellcolor{tabtealfirst!55!white}0.781 & \cellcolor{tabtealfirst!55!white}0.352 \\
3D Convex & 15.2 & 0.71 & 214.8 & 28.19 & 0.886 & \cellcolor{tabtealfirst!55!white}0.309 & 28.5 & \cellcolor{tabtealfirst!55!white}1.25 & 118.8 & 22.19 & 0.755 & 0.380 \\
\addlinespace[2pt]
\multicolumn{13}{@{}c@{}}{\textcolor{black!65}{\textit{Hybrid primitives}}} \\
BBSplat & \cellcolor{tabbluefirst!93!blue}14.7 & \cellcolor{tabbluefirst!93!blue}0.74 & \cellcolor{tabbluefirst!93!blue}182.8 & 28.82 & 0.891 & 0.285 & \cellcolor{tabbluefirst!93!blue}24.5 & \cellcolor{tabbluefirst!93!blue}2.55 & \cellcolor{tabbluefirst!93!blue}99.1 & 20.87 & 0.692 & 0.451 \\
Nexels & 29.1 & 2.19 & 151.7 & \cellcolor{tabbluefirst!93!blue}29.15 & \cellcolor{tabbluefirst!93!blue}0.899 & \cellcolor{tabbluefirst!93!blue}0.247 & 73.0 & 3.29 & 82.2 & \cellcolor{tabbluefirst!93!blue}22.90 & \cellcolor{tabbluefirst!93!blue}0.775 & \cellcolor{tabbluefirst!93!blue}0.321 \\
\addlinespace[2pt]
\multicolumn{13}{@{}c@{}}{\textcolor{black!65}{\textit{Surface primitives}}} \\
2DGS & \cellcolor{tabsecond}9.4 & \cellcolor{tabthird}0.97 & \cellcolor{tabthird}359.4 & \cellcolor{tabthird}26.61 & \cellcolor{tabthird}0.860 & \cellcolor{tabthird}0.367 & \cellcolor{tabsecond}12.7 & 2.39 & \cellcolor{tabthird}186.3 & \cellcolor{tabthird}21.63 & \cellcolor{tabthird}0.749 & \cellcolor{tabthird}0.385 \\
Triangle Spl. & \cellcolor{tabfirst}8.5 & \cellcolor{tabsecond}0.65 & \cellcolor{tabsecond}482.6 & \cellcolor{tabsecond}27.23 & \cellcolor{tabsecond}0.866 & \cellcolor{tabsecond}0.326 & \cellcolor{tabfirst}12.0 & \cellcolor{tabsecond}1.04 & \cellcolor{tabfirst}393.6 & \cellcolor{tabsecond}22.94 & \cellcolor{tabsecond}0.806 & \cellcolor{tabsecond}0.303 \\
MeshSplatting & \cellcolor{tabthird}13.3 & 1.86 & 126.3 & 22.59 & 0.703 & 0.486 & 17.7 & \cellcolor{tabthird}2.25 & 113.2 & 17.24 & 0.609 & 0.527 \\
\textbf{Hexels V3-F} & \cellcolor{tabfirst}8.5 & \cellcolor{tabfirst}0.58 & \cellcolor{tabfirst}792.3 & \cellcolor{tabfirst}29.95 & \cellcolor{tabfirst}0.899 & \cellcolor{tabfirst}0.293 & \cellcolor{tabthird}14.6 & \cellcolor{tabfirst}0.71 & \cellcolor{tabsecond}370.6 & \cellcolor{tabfirst}23.78 & \cellcolor{tabfirst}0.824 & \cellcolor{tabfirst}0.281 \\
\bottomrule
\end{tabular}}
\end{table*}

\begin{table*}[!htbp]
\centering
\caption{\textbf{Model-memory budgets for all 34 scenes.} V1/V2/V3 require $80N$/$116N$/$260N$ bytes, where $N$ is the SfM point count.}
\label{tab:scene_capacities}
\begingroup\fontsize{8}{9.5}\selectfont
\setlength{\tabcolsep}{3pt}
\renewcommand{\arraystretch}{1.3}
\fittabular{12}{%
\begin{tabular}{@{}lrrr@{\hspace{1.4em}}lrrr@{\hspace{1.4em}}lrrr@{}}
\toprule
\multicolumn{4}{c}{\textit{Mip-NeRF360}} & \multicolumn{4}{c}{\textit{DTU NVS}} & \multicolumn{4}{c}{\textit{Tanks and Temples}} \\
\cmidrule(lr){1-4}\cmidrule(lr){5-8}\cmidrule(lr){9-12}
Scene & V1 (MB) & V2 (MB) & V3 (MB) & Scene & V1 (MB) & V2 (MB) & V3 (MB) & Scene & V1 (MB) & V2 (MB) & V3 (MB) \\
\midrule
bicycle & 4.342 & 6.296 & 14.111 & scan105 & 3.545 & 5.141 & 11.522 & Barn & 16.220 & 23.519 & 52.714 \\
bonsai & 16.529 & 23.967 & 53.719 & scan106 & 3.707 & 5.375 & 12.047 & Caterpillar & 15.024 & 21.785 & 48.829 \\
counter & 12.461 & 18.069 & 40.499 & scan110 & 2.042 & 2.960 & 6.635 & Courthouse & 39.629 & 57.461 & 128.793 \\
flowers & 3.068 & 4.448 & 9.970 & scan114 & 3.442 & 4.992 & 11.188 & Francis & 8.545 & 12.391 & 27.773 \\
garden & 11.101 & 16.097 & 36.079 & scan118 & 2.270 & 3.292 & 7.379 & Horse & 5.384 & 7.807 & 17.500 \\
kitchen & 19.309 & 27.999 & 62.755 & scan122 & 1.741 & 2.524 & 5.658 & Ignatius & 12.424 & 18.015 & 40.380 \\
room & 9.010 & 13.065 & 29.283 & scan24 & 2.496 & 3.620 & 8.113 & Lighthouse & 12.828 & 18.600 & 41.690 \\
stump & 2.564 & 3.718 & 8.333 & scan37 & 4.065 & 5.894 & 13.210 & Meetingroom & 9.868 & 14.309 & 32.072 \\
treehill & 4.189 & 6.074 & 13.614 & scan40 & 4.692 & 6.803 & 15.247 & Train & 14.734 & 21.365 & 47.887 \\
 &  &  &  & scan55 & 3.742 & 5.426 & 12.161 & Truck & 10.993 & 15.940 & 35.728 \\
 &  &  &  & scan63 & 1.421 & 2.061 & 4.618 &  &  &  &  \\
 &  &  &  & scan65 & 1.809 & 2.624 & 5.880 &  &  &  &  \\
 &  &  &  & scan69 & 1.858 & 2.694 & 6.037 &  &  &  &  \\
 &  &  &  & scan83 & 1.281 & 1.858 & 4.164 &  &  &  &  \\
 &  &  &  & scan97 & 2.658 & 3.854 & 8.639 &  &  &  &  \\
\bottomrule
\end{tabular}}\endgroup
\end{table*}

\subsection{Fixed and Densified Hexels}
\label{app:hexels_variants}

Densifying V3 improves PSNR by 1.37 dB outdoors and 0.78 dB indoors
on Mip-NeRF360 (Table~\ref{tab:nvs_hexels_appendix_mip_outdoor}),
increasing deployment size from 17.0 to 337.6 MB and 48.2 to 133.2 MB,
respectively.
On DTU, densified V3 has 0.20 dB lower PSNR than fixed V3 but 0.026 lower LPIPS; on T\&T it gains 0.76 dB (Table~\ref{tab:nvs_hexels_appendix_dtu}).
Fixed models therefore offer a particularly compact option for DTU,
while growth improves quality on the larger scenes.

\begin{table*}[t]
\caption{\textbf{Fixed and densified Hexels on Mip-NeRF360.} F/G denote fixed/grown populations.}
\label{tab:nvs_hexels_appendix_mip_outdoor}
\small
\flushrules
\setlength{\tabcolsep}{3pt}
\renewcommand{\arraystretch}{1.14}
\fittabular{9}{%
\begin{tabular}{@{}lrrrrrrrr@{}}
\toprule
\multicolumn{1}{@{}l}{\multirow{3}{*}{Method}} & \multicolumn{8}{c@{}}{Mip-NeRF360 outdoor (5)} \\
\cmidrule(l){2-9}
 & Train & VRAM & FPS & PSNR & SSIM & LPIPS & Deploy. & Prims. \\
 & min $\downarrow$ & GiB $\downarrow$ &  $\uparrow$ & dB $\uparrow$ &  $\uparrow$ & VGG $\downarrow$ & MB $\downarrow$ & K \\
\midrule
\textbf{Hexels V1-F} & \cellcolor{tabfirst}17.6 & \cellcolor{tabfirst}0.60 & \cellcolor{tabfirst}495.5 & 21.98 & 0.557 & 0.501 & \cellcolor{tabfirst}5.64 & 63.2 \\
\textbf{Hexels V1-G} & \cellcolor{tabthird}28.8 & 2.44 & 140.5 & \cellcolor{tabsecond}23.43 & \cellcolor{tabsecond}0.667 & \cellcolor{tabsecond}0.349 & 113.81 & 1278.6 \\
\textbf{Hexels V2-F} & \cellcolor{tabsecond}17.8 & \cellcolor{tabsecond}0.63 & \cellcolor{tabsecond}491.9 & 22.11 & 0.562 & 0.497 & \cellcolor{tabsecond}7.91 & 63.2 \\
\textbf{Hexels V2-G} & 29.3 & 3.84 & 125.4 & \cellcolor{tabthird}23.39 & \cellcolor{tabthird}0.663 & \cellcolor{tabthird}0.358 & 143.96 & 1151.6 \\
\textbf{Hexels V3-F} & \cellcolor{tabfirst}17.6 & \cellcolor{tabthird}0.66 & \cellcolor{tabthird}484.6 & 22.27 & 0.566 & 0.496 & \cellcolor{tabthird}17.00 & 63.2 \\
\textbf{Hexels V3-G} & 33.6 & 17.27 & 94.5 & \cellcolor{tabfirst}23.64 & \cellcolor{tabfirst}0.674 & \cellcolor{tabfirst}0.346 & 337.60 & 1254.9 \\
\bottomrule
\end{tabular}}
\par\smallskip
\fittabular{9}{%
\begin{tabular}{@{}lrrrrrrrr@{}}
\toprule
\multicolumn{1}{@{}l}{\multirow{3}{*}{Method}} & \multicolumn{8}{c@{}}{Mip-NeRF360 indoor (4)} \\
\cmidrule(l){2-9}
 & Train & VRAM & FPS & PSNR & SSIM & LPIPS & Deploy. & Prims. \\
 & min $\downarrow$ & GiB $\downarrow$ &  $\uparrow$ & dB $\uparrow$ &  $\uparrow$ & VGG $\downarrow$ & MB $\downarrow$ & K \\
\midrule
\textbf{Hexels V1-F} & \cellcolor{tabfirst}32.4 & \cellcolor{tabfirst}0.83 & \cellcolor{tabfirst}230.7 & 28.36 & 0.891 & 0.297 & \cellcolor{tabfirst}15.95 & 179.1 \\
\textbf{Hexels V1-G} & 43.1 & 1.32 & 115.9 & \cellcolor{tabthird}29.31 & \cellcolor{tabthird}0.903 & \cellcolor{tabsecond}0.270 & \cellcolor{tabthird}48.02 & 539.4 \\
\textbf{Hexels V2-F} & \cellcolor{tabthird}32.8 & \cellcolor{tabsecond}0.87 & \cellcolor{tabsecond}228.9 & 28.74 & 0.897 & 0.290 & \cellcolor{tabsecond}22.40 & 179.1 \\
\textbf{Hexels V2-G} & 41.7 & 1.40 & 122.5 & 29.27 & \cellcolor{tabsecond}0.904 & \cellcolor{tabthird}0.271 & 59.63 & 476.9 \\
\textbf{Hexels V3-F} & \cellcolor{tabsecond}32.7 & \cellcolor{tabthird}0.99 & \cellcolor{tabthird}219.8 & \cellcolor{tabsecond}29.33 & \cellcolor{tabsecond}0.904 & 0.283 & 48.19 & 179.1 \\
\textbf{Hexels V3-G} & 42.4 & 9.81 & 108.6 & \cellcolor{tabfirst}30.11 & \cellcolor{tabfirst}0.913 & \cellcolor{tabfirst}0.259 & 133.21 & 495.2 \\
\bottomrule
\end{tabular}}
\end{table*}

\begin{table*}[t]
\caption{\textbf{Fixed and densified Hexels on DTU and Tanks and Temples.} F/G denote fixed/grown populations.}
\label{tab:nvs_hexels_appendix_dtu}
\small
\flushrules
\setlength{\tabcolsep}{3pt}
\renewcommand{\arraystretch}{1.14}
\fittabular{9}{%
\begin{tabular}{@{}lrrrrrrrr@{}}
\toprule
\multicolumn{1}{@{}l}{\multirow{3}{*}{Method}} & \multicolumn{8}{c@{}}{DTU (15)} \\
\cmidrule(l){2-9}
 & Train & VRAM & FPS & PSNR & SSIM & LPIPS & Deploy. & Prims. \\
 & min $\downarrow$ & GiB $\downarrow$ &  $\uparrow$ & dB $\uparrow$ &  $\uparrow$ & VGG $\downarrow$ & MB $\downarrow$ & K \\
\midrule
\textbf{Hexels V1-F} & \cellcolor{tabfirst}8.4 & \cellcolor{tabfirst}0.55 & \cellcolor{tabsecond}830.8 & 28.94 & 0.890 & 0.303 & \cellcolor{tabfirst}3.04 & 34.0 \\
\textbf{Hexels V1-G} & 10.8 & 0.61 & 441.4 & 29.31 & \cellcolor{tabthird}0.897 & \cellcolor{tabsecond}0.278 & \cellcolor{tabthird}8.69 & 97.5 \\
\textbf{Hexels V2-F} & \cellcolor{tabsecond}8.5 & \cellcolor{tabsecond}0.56 & \cellcolor{tabfirst}845.6 & \cellcolor{tabthird}29.57 & 0.895 & 0.298 & \cellcolor{tabsecond}4.26 & 34.0 \\
\textbf{Hexels V2-G} & \cellcolor{tabthird}10.3 & 0.64 & 513.8 & 29.14 & 0.891 & \cellcolor{tabthird}0.282 & 11.40 & 91.1 \\
\textbf{Hexels V3-F} & \cellcolor{tabsecond}8.5 & \cellcolor{tabthird}0.58 & \cellcolor{tabthird}792.3 & \cellcolor{tabfirst}29.95 & \cellcolor{tabsecond}0.899 & 0.293 & 9.15 & 34.0 \\
\textbf{Hexels V3-G} & 10.6 & 1.21 & 483.1 & \cellcolor{tabsecond}29.76 & \cellcolor{tabfirst}0.902 & \cellcolor{tabfirst}0.268 & 27.03 & 100.4 \\
\bottomrule
\end{tabular}}
\par\smallskip
\fittabular{9}{%
\begin{tabular}{@{}lrrrrrrrr@{}}
\toprule
\multicolumn{1}{@{}l}{\multirow{3}{*}{Method}} & \multicolumn{8}{c@{}}{Tanks and Temples (10)} \\
\cmidrule(l){2-9}
 & Train & VRAM & FPS & PSNR & SSIM & LPIPS & Deploy. & Prims. \\
 & min $\downarrow$ & GiB $\downarrow$ &  $\uparrow$ & dB $\uparrow$ &  $\uparrow$ & VGG $\downarrow$ & MB $\downarrow$ & K \\
\midrule
\textbf{Hexels V1-F} & \cellcolor{tabfirst}13.3 & \cellcolor{tabfirst}0.55 & \cellcolor{tabfirst}438.4 & 23.05 & 0.811 & 0.295 & \cellcolor{tabfirst}16.22 & 182.1 \\
\textbf{Hexels V1-G} & 20.8 & 1.38 & 197.1 & \cellcolor{tabthird}23.95 & \cellcolor{tabthird}0.829 & \cellcolor{tabthird}0.261 & 67.79 & 761.6 \\
\textbf{Hexels V2-F} & \cellcolor{tabsecond}14.4 & \cellcolor{tabsecond}0.60 & \cellcolor{tabsecond}386.9 & 23.49 & 0.819 & 0.287 & \cellcolor{tabsecond}22.77 & 182.1 \\
\textbf{Hexels V2-G} & 22.7 & 1.96 & 167.5 & \cellcolor{tabsecond}24.27 & \cellcolor{tabsecond}0.834 & \cellcolor{tabsecond}0.257 & 89.80 & 718.3 \\
\textbf{Hexels V3-F} & \cellcolor{tabthird}14.6 & \cellcolor{tabthird}0.71 & \cellcolor{tabthird}370.6 & 23.78 & 0.824 & 0.281 & \cellcolor{tabthird}48.99 & 182.1 \\
\textbf{Hexels V3-G} & 24.3 & 12.06 & 142.7 & \cellcolor{tabfirst}24.53 & \cellcolor{tabfirst}0.839 & \cellcolor{tabfirst}0.248 & 198.84 & 739.1 \\
\bottomrule
\end{tabular}}
\end{table*}

\subsection{Deployment Compression analysis}
\label{app:nvs_compression}
While outside our core efficiency objectives, we also provide an initial analysis of lossless compression and appearance parameter quantization 
(PQ1024) on all six variants across 34 scenes.
Our lossless compression analysis keeps every FP32 parameter
exact. It regroups the bytes of each parameter array by position (all first
bytes of the 4-byte values, then all second bytes, and so on), so that the
highly repetitive sign and exponent bytes are stored together. It then
compresses the result with DEFLATE~\cite{deutsch1996deflate}, the standard
LZ77 and Huffman-coding compressor used by ZIP and gzip software.
PQ1024 replaces RGB appearance vectors with 1,024-entry codebooks:
one per SH coefficient for V3, and corresponding codebooks for the
spatial and directional gradients of V1/V2.
Geometry and base color remain lossless.
Parameters are decoded to FP32 before rendering, without retraining.

Lossless compression and PQ1024 evaluated here are additional experiments; the main NVS experiments use uncompressed FP32 models.
Quantizing V2 reduces PSNR by \val{0.55--0.61}{dB}, whereas densifier V3-G PQ1024
changes mean PSNR/SSIM/LPIPS by $-0.01674$ dB/$-0.000259$/$+0.000849$
(Table~\ref{tab:compression}).
The largest dataset-mean PSNR loss for V3-G is 0.0303 dB on Mip-NeRF360.
Across 34 scenes, its mean deployment size is 43.87 MB, 60.09\% below lossless compression.
Table~\ref{tab:compression} reports these compressed results whereas all other NVS comparison tables use uncompressed FP32 representation.

Encoding time includes codebook fitting; decoding measures cached reading
and CPU-based reconstruction, followed by upload into an allocated model.

\paragraph{2D model size (deployment) compression.}
For the 2D representation in Table~\ref{tab:hexels-2d}, mixed-bit storage reduces mean PSNR by 0.03307\,dB for constant-color Hexels and 0.01091\,dB for linear-color Hexels relative to their FP32 parameters.

After training/image fitting, constant-color Hexels store
positions, (log-)scales, rotations and RGB parameters with 16, 14, 12 and 10 bits per scalar, respectively. Linear-color Hexels use 16, 14, 12, 12, 14 and 10 bits
for positions, scales, rotations, base RGB, spatial RGB gradients and sharpness. The bit-packed files include a 1024-byte header containing
quantization ranges and decoder settings. At 10K primitives, the constant
and linear files occupy 128524 and 253524 bytes, or 0.2451 and 0.4836 bits
per pixel. The 16-bit Image-GS files occupy 161024 bytes, or 0.3071 bits per
pixel. Both Hexel variants show similar performance for 16-bit storage: at most 0.05\,dB mean PSNR loss, 0.2\,dB worst-image loss, and 0.001 mean SSIM loss on the evaluated samples.

\begin{table*}[t]
\caption{\textbf{Deployment compression across 34 scenes.} Quality changes are relative to uncompressed models; FP32 rows are the uncompressed reference results from the main text.}
\label{tab:compression}
\small
\flushrules
\setlength{\tabcolsep}{3pt}
\renewcommand{\arraystretch}{1.14}
\fittabular{11}{%
\begin{tabular}{@{}llrrrrrrrrr@{}}
\toprule
Variant & Codec & Encode & Decode & Upload & FPS & $\Delta$PSNR & $\Delta$SSIM & $\Delta$LPIPS & Deploy. & Saved \\
 & & s & ms & ms & & dB & & VGG & MB & \% \\
\midrule
V1 fixed & FP32 & -- & 6.0 & 2.8 & 608.4 & 0.0000 & 0.00000 & 0.00000 & 8.82 & 0.0 \\
V1 fixed & Lossless & 0.5 & 34.5 & 4.3 & 620.8 & 0.0000 & 0.00000 & 0.00000 & 6.67 & 24.2 \\
V1 fixed & PQ1024 & 1.8 & 33.1 & 3.1 & 619.1 & -0.0338 & -0.00057 & 0.00175 & 4.97 & 43.3 \\
V1 grown & FP32 & -- & 13.5 & 8.6 & 289.5 & 0.0000 & 0.00000 & 0.00000 & 46.16 & 0.0 \\
V1 grown & Lossless & 2.8 & 163.2 & 8.5 & 290.2 & 0.0000 & 0.00000 & 0.00000 & 34.90 & 24.4 \\
V1 grown & PQ1024 & 4.0 & 152.3 & 8.1 & 289.3 & -0.0325 & -0.00058 & 0.00168 & 25.91 & 43.8 \\
V2 fixed & FP32 & -- & 6.4 & 4.3 & 590.1 & 0.0000 & 0.00000 & 0.00000 & 12.38 & 0.0 \\
V2 fixed & Lossless & 0.8 & 47.6 & 3.4 & 610.4 & 0.0000 & 0.00000 & 0.00000 & 9.66 & 21.8 \\
V2 fixed & PQ1024 & 2.4 & 42.0 & 3.1 & 599.9 & -0.6061 & -0.00767 & 0.02537 & 5.38 & 55.9 \\
V2 grown & FP32 & -- & 17.1 & 8.0 & 310.4 & 0.0000 & 0.00000 & 0.00000 & 59.63 & 0.0 \\
V2 grown & Lossless & 3.7 & 209.3 & 10.4 & 314.6 & 0.0000 & 0.00000 & 0.00000 & 46.50 & 21.9 \\
V2 grown & PQ1024 & 5.4 & 186.3 & 12.2 & 309.0 & -0.5450 & -0.00777 & 0.02449 & 25.69 & 56.7 \\
V3 fixed & FP32 & -- & 9.6 & 5.3 & 560.5 & 0.0000 & 0.00000 & 0.00000 & 26.62 & 0.0 \\
V3 fixed & Lossless & 1.7 & 103.1 & 4.6 & 565.1 & 0.0000 & 0.00000 & 0.00000 & 21.59 & 18.8 \\
V3 fixed & PQ1024 & 4.7 & 77.4 & 5.4 & 555.7 & -0.0200 & -0.00032 & 0.00085 & 8.73 & 66.5 \\
V3 grown & FP32 & -- & 31.0 & 17.0 & 282.6 & 0.0000 & 0.00000 & 0.00000 & 135.73 & 0.0 \\
V3 grown & Lossless & 8.4 & 464.2 & 21.0 & 284.5 & 0.0000 & 0.00000 & 0.00000 & 109.92 & 18.9 \\
V3 grown & PQ1024 & 11.7 & 363.1 & 22.8 & 283.5 & -0.0167 & -0.00026 & 0.00085 & 43.87 & 67.4 \\
\bottomrule
\end{tabular}}
\end{table*}

\section{Per-Scene NVS Results}
\label{app:per_scene}

Tables~\ref{tab:v3_scenes_mipnerf360}, \ref{tab:v3_scenes_dtu}
and~\ref{tab:v3_scenes_tanks_and_temples} report V3 densified on all
nine Mip-NeRF360 scenes, fifteen DTU scans and ten T\&T scenes.

\begin{table*}[t]
\caption{\textbf{Per-scene V3 densified results on Mip-NeRF360.} The final row reports the mean across scenes.}
\label{tab:v3_scenes_mipnerf360}
\small
\flushrules
\setlength{\tabcolsep}{3pt}
\renewcommand{\arraystretch}{1.14}
\fittabular{9}{%
\begin{tabular}{@{}lrrrrrrrr@{}}
\toprule
Scene & Train & VRAM peak & FPS & PSNR & SSIM & LPIPS & Deploy. & Primitives \\
 & min & GiB & & dB & & VGG & MB & K \\
\midrule
bicycle & 32.4 & 18.55 & 88.5 & 24.00 & 0.693 & 0.329 & 348.07 & 1293.9 \\
flowers & 35.1 & 19.08 & 95.3 & 20.36 & 0.539 & 0.429 & 353.13 & 1312.7 \\
garden & 40.6 & 20.49 & 66.3 & 26.93 & 0.840 & 0.169 & 486.56 & 1808.7 \\
stump & 26.6 & 12.67 & 123.7 & 24.73 & 0.690 & 0.358 & 211.91 & 787.7 \\
treehill & 33.2 & 15.56 & 98.4 & 22.19 & 0.609 & 0.443 & 288.31 & 1071.7 \\
room & 31.4 & 7.42 & 142.0 & 29.86 & 0.901 & 0.321 & 111.51 & 414.5 \\
counter & 44.4 & 7.62 & 105.3 & 28.61 & 0.901 & 0.273 & 110.34 & 410.1 \\
kitchen & 51.0 & 9.42 & 87.9 & 30.56 & 0.918 & 0.173 & 127.87 & 475.3 \\
bonsai & 42.8 & 14.78 & 99.4 & 31.43 & 0.932 & 0.270 & 183.13 & 680.7 \\
\midrule
Mean & 37.5 & 13.95 & 100.7 & 26.52 & 0.780 & 0.307 & 246.76 & 917.3 \\
\bottomrule
\end{tabular}}
\end{table*}

\begin{table*}[t]
\caption{\textbf{Per-scene V3 densified results on DTU.} The final row reports the mean across scenes.}
\label{tab:v3_scenes_dtu}
\small
\flushrules
\setlength{\tabcolsep}{3pt}
\renewcommand{\arraystretch}{1.14}
\fittabular{9}{%
\begin{tabular}{@{}lrrrrrrrr@{}}
\toprule
Scene & Train & VRAM peak & FPS & PSNR & SSIM & LPIPS & Deploy. & Primitives \\
 & min & GiB & & dB & & VGG & MB & K \\
\midrule
scan24 & 12.7 & 0.95 & 322.7 & 26.68 & 0.911 & 0.166 & 30.38 & 112.9 \\
scan37 & 13.4 & 2.93 & 275.5 & 25.48 & 0.896 & 0.193 & 59.45 & 220.9 \\
scan40 & 14.3 & 4.04 & 240.3 & 26.73 & 0.905 & 0.194 & 72.14 & 268.1 \\
scan55 & 12.2 & 3.10 & 280.6 & 29.07 & 0.882 & 0.266 & 61.98 & 230.4 \\
scan63 & 10.4 & 0.60 & 579.4 & 30.04 & 0.938 & 0.167 & 12.05 & 44.7 \\
scan65 & 9.5 & 0.67 & 504.4 & 30.00 & 0.886 & 0.279 & 21.40 & 79.5 \\
scan69 & 9.7 & 0.65 & 475.8 & 27.02 & 0.886 & 0.241 & 20.68 & 76.8 \\
scan83 & 9.4 & 0.62 & 571.3 & 29.53 & 0.889 & 0.334 & 11.12 & 41.3 \\
scan97 & 11.1 & 0.68 & 406.8 & 27.72 & 0.875 & 0.287 & 21.76 & 80.8 \\
scan105 & 10.7 & 0.66 & 446.6 & 30.17 & 0.885 & 0.273 & 18.69 & 69.4 \\
scan106 & 9.1 & 0.69 & 555.6 & 33.19 & 0.916 & 0.323 & 19.64 & 73.0 \\
scan110 & 8.3 & 0.62 & 693.8 & 31.99 & 0.913 & 0.342 & 10.94 & 40.6 \\
scan114 & 9.8 & 0.69 & 480.2 & 30.40 & 0.909 & 0.289 & 21.93 & 81.5 \\
scan118 & 8.9 & 0.62 & 712.9 & 33.98 & 0.920 & 0.326 & 12.29 & 45.6 \\
scan122 & 8.9 & 0.60 & 700.6 & 34.38 & 0.919 & 0.337 & 11.00 & 40.8 \\
\midrule
Mean & 10.6 & 1.21 & 483.1 & 29.76 & 0.902 & 0.268 & 27.03 & 100.4 \\
\bottomrule
\end{tabular}}
\end{table*}

\begin{table*}[t]
\caption{\textbf{Per-scene V3 densified results on Tanks and Temples.} The final row reports the mean across scenes.}
\label{tab:v3_scenes_tanks_and_temples}
\small
\flushrules
\setlength{\tabcolsep}{3pt}
\renewcommand{\arraystretch}{1.14}
\fittabular{9}{%
\begin{tabular}{@{}lrrrrrrrr@{}}
\toprule
Scene & Train & VRAM peak & FPS & PSNR & SSIM & LPIPS & Deploy. & Primitives \\
 & min & GiB & & dB & & VGG & MB & K \\
\midrule
Train & 23.5 & 10.11 & 144.4 & 21.44 & 0.794 & 0.258 & 152.85 & 568.2 \\
Truck & 28.8 & 21.09 & 98.1 & 24.22 & 0.847 & 0.211 & 329.30 & 1224.1 \\
Francis & 22.1 & 9.13 & 165.9 & 28.34 & 0.899 & 0.331 & 126.42 & 469.9 \\
Horse & 26.4 & 16.56 & 113.2 & 24.82 & 0.878 & 0.207 & 250.68 & 931.8 \\
Lighthouse & 23.7 & 7.28 & 164.4 & 23.17 & 0.849 & 0.263 & 109.67 & 407.7 \\
Barn & 21.5 & 8.39 & 182.4 & 28.36 & 0.867 & 0.214 & 125.86 & 467.8 \\
Caterpillar & 21.6 & 13.20 & 150.9 & 23.82 & 0.799 & 0.262 & 181.02 & 672.9 \\
Courthouse & 25.4 & 3.33 & 162.5 & 22.87 & 0.800 & 0.270 & 134.50 & 500.0 \\
Ignatius & 29.7 & 23.35 & 82.5 & 22.39 & 0.784 & 0.236 & 440.19 & 1636.3 \\
Meetingroom & 20.3 & 8.18 & 162.7 & 25.91 & 0.880 & 0.230 & 137.91 & 512.6 \\
\midrule
Mean & 24.3 & 12.06 & 142.7 & 24.53 & 0.839 & 0.248 & 198.84 & 739.1 \\
\bottomrule
\end{tabular}}
\end{table*}

\section{Additional Qualitative Comparisons}
\label{app:nvs_qualitative}

Figures~\ref{fig:nvs_supp_mipnerf360}--\ref{fig:nvs_supp_tanks_and_temples}
compare V3 (densified) with native 2DGS, Triangle Splatting and MeshSplatting
on additional Mip-NeRF360, DTU, and Tanks and Temples scenes.
The enlarged regions show differences in object boundaries and fine textures.

\clearpage
\begin{figure*}[t]
\centering
\includegraphics[width=\textwidth,height=0.90\textheight,keepaspectratio]{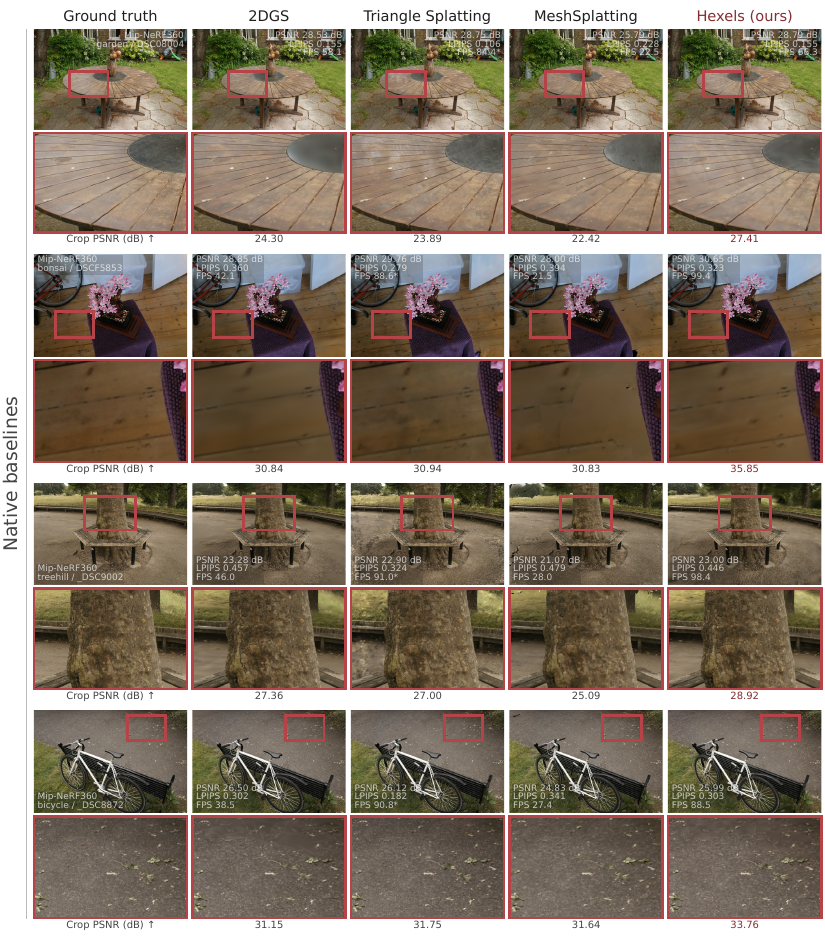}
\caption{\textbf{Additional Mip-NeRF360 comparisons.}}
\label{fig:nvs_supp_mipnerf360}
\end{figure*}

\begin{figure*}[p]
\centering
\includegraphics[width=\textwidth,height=0.90\textheight,keepaspectratio]{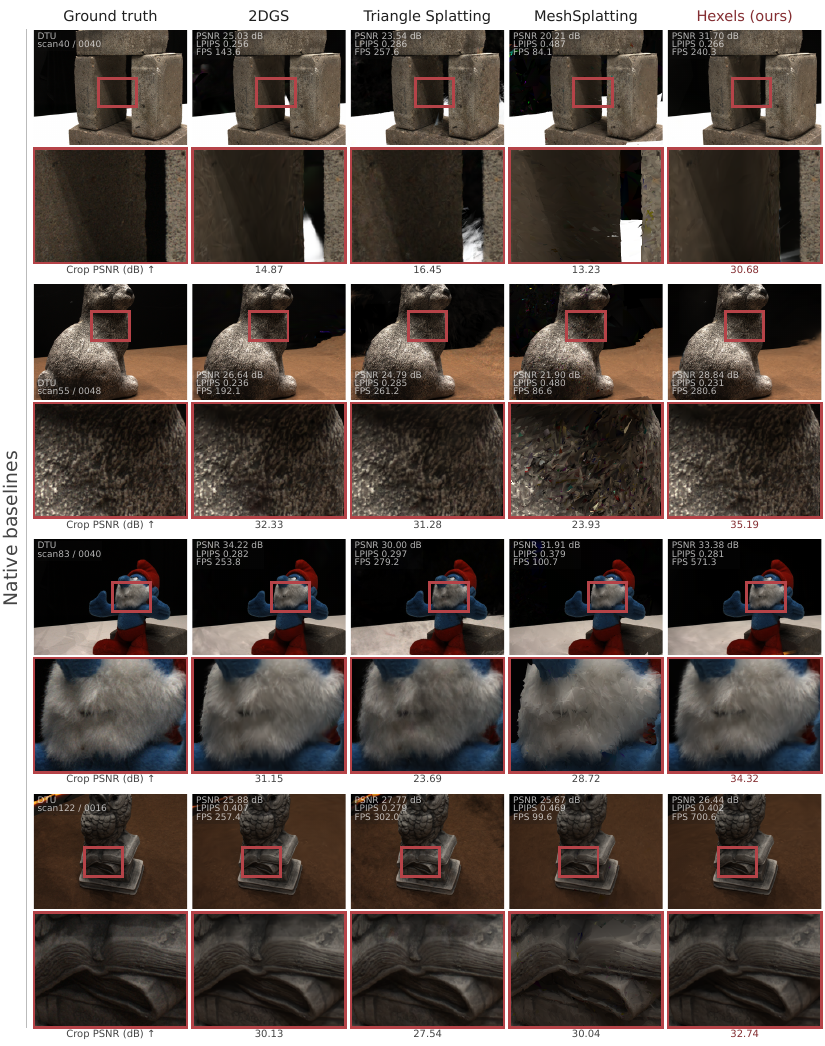}
\caption{\textbf{Additional DTU comparisons.}}
\label{fig:nvs_supp_dtu}
\end{figure*}

\begin{figure*}[p]
\centering
\includegraphics[width=\textwidth,height=0.90\textheight,keepaspectratio]{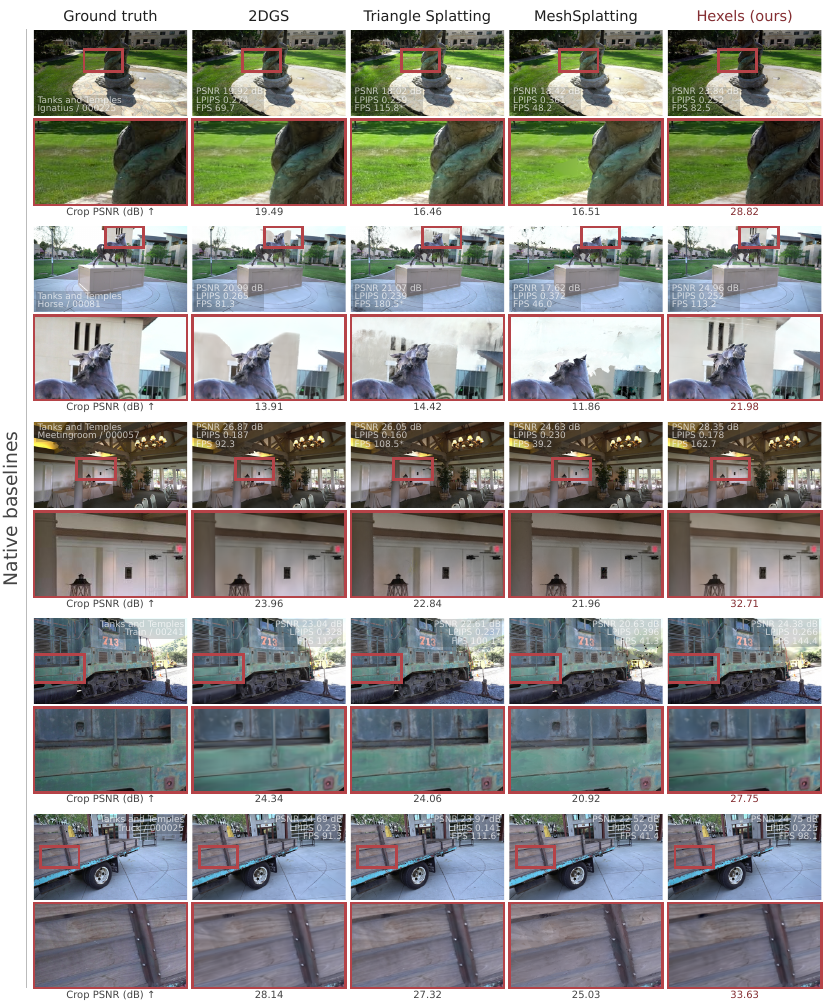}
\caption{\textbf{Additional Tanks and Temples comparisons.}}
\label{fig:nvs_supp_tanks_and_temples}
\end{figure*}

\end{document}